\documentclass[a4paper,11pt,parskip=full-,twoside,headings=normal]{scrbook}
\usepackage{scrpage2}
\usepackage{typearea}
\usepackage{latexsym}
\usepackage[utf8]{inputenc}
\usepackage[T1]{fontenc}
\usepackage{amsmath}
\usepackage{amssymb}
\usepackage{amsthm}
\ohead{\headmark}
\usepackage{url}
\usepackage{cite}
\usepackage{hyperref}

\usepackage{graphicx}
\usepackage{tabularx}
\usepackage{array}
\usepackage[inline]{asymptote}
\usepackage{eqparbox}
\usepackage{mdwmath}
\usepackage{mdwtab}
\usepackage[cmyk]{xcolor}
\usepackage{stmaryrd}
\usepackage{tikz}
\usepackage{enumerate}

\newcommand{\pc}{\mathbf{P}}

\newcommand{\ifof}{if and only if }
\newcommand{\rsd}{rough semantic domain }
\newcommand{\rsds}{rough semantic domains }
\newcommand{\ifsf}{\,\mathrm{if}\: \mathrm{and}\:\mathrm{only}\: \mathrm{if}\: }

\newcommand{\lf}{\leftthreetimes}

\newenvironment{mitemize}{\begin{itemize}}{\end{itemize}}
\newenvironment{impemize}{\begin{itemize}}{\end{itemize}}

\newcommand{\bbf}{\textbf}

\newtheorem{theorem}{Theorem}[chapter]
\newtheorem{definition}{Definition}[chapter]
\newtheorem{proposition}{Prop}[chapter]
\newtheorem{lemma}{Lemma}[chapter]
\newtheorem{remark}{Remark}[chapter]
\newcounter{fem}
\newtheorem{examp}[lemma]{Persistent Example}
\begin{document}

\title{Algebraic Semantics of Proto-Transitive Rough Sets}
\author{A. Mani\\
Department of Pure Mathematics\\
University of Calcutta\\
9/1B, Jatin Bagchi Road\\
Kolkata-700029, India\\
Web: \url{http://www.logicamani.in}\\
Email: \texttt {$a.mani.cms@gmail.com$}}
\date{}
\frontmatter

\begin{titlepage}
\begin{center}



\textsc{\Large Monograph}~\\[0.5cm]

\hrulefill 

{ \huge \bfseries Algebraic Semantics of Proto-Transitive Rough Sets~\\[0.3cm] }

\hrulefill 

\vspace{15pt}

\begin{center}
{\Large \bfseries  A. Mani\\
Department of Pure Mathematics\\
University of Calcutta\\
9/1B, Jatin Bagchi Road\\
Kolkata-700029, India\\
Email: {a.mani.cms@gmail.com}\\

Web: \url{http://www.logicamani.in}}
\end{center}

\vspace{78pt}

{\Large \bfseries First Edition' July,2014}

\end{center}
\end{titlepage}

\addchap*{Preface}

{\sffamily Rough sets over generalized transitive relations like proto-transitive ones have been initiated by the present author in the year 2012 \cite{AM270}.  Subsequently \cite{AM3690}, approximation of proto-transitive relations by other relations was investigated and the relation with rough approximations was developed towards constructing semantics that can handle \emph{fragments of structure}. It was also proved that difference of approximations induced by some approximate relations need not induce rough structures. In this research we develop  different semantics of proto transitive rough sets (\textsf{PRAX}) after characterizing the structure of rough objects and also develop a theory of dependence for general rough sets and use it to internalize the Nelson-algebra based approximate semantics developed earlier \cite{AM3690}. The theory of rough dependence initiated later \cite{AM3930} by the present author is extended in the process. This monograph is reasonably self-contained and includes proofs and extensions of representation of objects that were not part of earlier papers. }

\section*{Keywords}

{\sffamily Proto Transitive Relations, PRAX, PRAS, Generalized Transitivity, Rough Dependence, Rough Objects, Granulation, Algebraic Semantics, Approximate Relations, Approximate Semantics, Kleene Algebras, Axiomatic Theory of Granules, Geometry of Knowledge, Contamination Problem.

}

\newpage
\pagestyle{empty}
\section*{About the Author}

{\sffamily A. Mani is an active researcher in algebra, logic, rough sets, vagueness, philosophy and
foundations of Mathematics. She has published extensively on the subjects in a number of
international peer-reviewed journals for more than a decade. Her current affiliations
include the University of Calcutta in Kolkata and Division-R of STVROM. She is active in various academic groups
like ISRS, IRSS, ASL and FOM. She is also a teacher, free software activist, feminist, consultant in statistical and soft computing and service provider.}

\tableofcontents

\mainmatter

\chapter[Introduction]{Introduction}
\automark[section]{chapter}

Proto-transitivity is one of the infinite number of possible generalizations of transitivity. These types of generalized relations happen often in application contexts. Failure to recognize them causes mathematical models to be inadequate or underspecified and tends to unduly complicate algorithms and approximate methods. From among the many possible alternatives that fall under \emph{generalized transitivity}, we chose \emph{proto-transitivity} because of application contexts, its simple set theoretic definition, connections with factor relations and consequent generative value among such relations. It has a special role in modelling knowledge as well.

Proto-transitive approximation spaces \textsf{PRAX} have been introduced by the present author in \cite{AM270} and the structure of definite objects has been characterized in it to a degree. It is relatively a harder structure from a semantic perspective as the representation of rough objects is involved \cite{AM270}. Aspects of knowledge interpretation in \textsf{PRAX} contexts have been considered in \cite{AM270} and in \cite{AM3690} the relation of approximations resulting from approximation of relations to the approximations from the original relation are studied in the context of \textsf{PRAX}. These are used for defining an approximate semantics for \textsf{PRAX} and their limitations are explored by the present author in the same paper. All of these are expanded upon in this monograph. 

Rough objects as explained in \cite{AM99,AM240} are collections of objects in a classical domain (Meta-C) that appear to be indistinguishable among themselves in another rough semantic domain (Meta-R). But their representation in most \textsf{RST}s in purely order theoretic terms is not known. For \textsf{PRAX}, this is solved in \cite{AM270}. Rough objects in a \textsf{PRAX} need not correspond to intervals of the form $]a, b[$ with the definite object $b$ covering (in the ordered set of definite objects) the definite object $a$.

If $R$ is a relation on a set $S$, then $R$ can be approximated by a wide variety of partial or quasi-order relations in both classical and rough set perspective \cite{RJ2011}. Though the methods are essentially equivalent for binary relations, the latter method is more general. When the relation $R$ satisfies proto-transitivity, then many new properties emerge. This aspect is developed further in the present monograph and most of \cite{AM3690} is included.

When $R$ is a quasi-order relation, then a semantics for the set of ordered pairs of lower and upper approximations  $\{(A^{l}, A^{u}) ; \, A\subseteq S\}$ has recently been developed in \cite{SJ,JPR}. Though such a set of ordered pairs of lower and upper approximations are not rough objects in the \textsf{PRAX} context, we can use the approximations for an additional semantic approach to it. We prove that differences of consequent lower and upper approximations suggest partial structures for \emph{measuring} structured deviation. The developed method should also be useful for studying correspondences between the different semantics \cite{AM1800,AM3600}. Because of this we devote some space to the nature of transformation of granules by the relational approximation process.      

In this research monograph, we also investigate the nature of possible concepts of \emph{rough dependence} first. Though the concept of independence is well studied in probability theory, the concept of dependence is rarely explored in any useful way. It has been shown to be very powerful in classical probability theory \cite{BD2010} - the formalism is valid over probability spaces, but its axiomatic potential is left unexplored. Connections between rough sets and probability theory have been explored from rough measure and information entropy viewpoint in a number of papers \cite{PZ2002,SL2006,GPS04,YY2003,BCC2007}. The nature of rough independence is also explored in \cite{AM3930} by the present author and there is some overlap with the present work. Apart from problems relating to contamination, we show that the comparison by way of corresponding concepts of dependence fails in a very essential way. 

Further, using the introduced concepts of rough dependence we internalize the approximate semantics instead of depending on correspondences. This allows for richer variants of the earlier semantics of rough objects.

This monograph is reasonably self-contained and is organized as follows: In the rest of this chapter we introduce the basics of proto-transitivity, recall relevant information of Nelson algebras, granules and granulations. In the following chapter, we define relevant approximations in \textsf{PRAX} and study their basic properties and those of definite elements. In the third chapter, we propose an abstract and three other extended examples justifying our study. In the following chapter, we describe the algebraic structures that can be associated with the semantic properties of definite objects in a \textsf{PRAX}. The representation of rough objects is done from an interesting perspective in the fifth chapter. In the sixth chapter, we define new derived operators in a \textsf{PRAX} and consider their connection with non monotonic reasoning. These are of relevance in representation again. In the following chapter, atoms in the partially ordered set of rough object are described. This is followed by an algebraic semantics that relies on multiple types of aggregation and commonality operations. In the ninth chapter, a partial semantics similar to the increasing Nelson algebraic semantics is formulated. This semantics is completed in three different ways in the fourteenth chapter after internalization of dependency. In the tenth and eleventh chapters approximate relations and approximate semantics are considered - the material in these chapters includes expansions of the results in \cite{AM3690}. In the following two chapters, we define concepts of rough dependence, compare them with those of probabilistic dependence and demonstrate their stark differences - the material in these chapters are expansions of \cite{AM3930}. The knowledge interpretation of \textsf{PRAX} is revisited in the fifteenth chapter.

\section[Basics]{Basic Concepts, Terminology}

\begin{definition}
A binary relation $R$ on a set $S$ is said to be \emph{weakly-transitive, transitive or
proto-transitive} respectively on $S$ \ifof 
$S$ satisfies 
\begin{mitemize}
\item { If whenever $Rxy,\,Ryz$ and $x\,\neq\,y\,\neq\,z$ holds, then $Rxz$.
(i.e. $(R\circ R)\setminus \Delta_{S}\,\subseteq R$ (where $\circ$ is relation composition)
, or}
\item {whenever $Rxy \,\&\,Ryz$ holds then $Rxz$ (i.e. $(R\circ
R) \subseteq R$), or }
\item {Whenever $Rxy,\,Ryz , \,Ryx ,\, Rzy$ and $x\,\neq\,y\,\neq\,z$ holds, then $Rxz$ follows, respectively. Proto-transitivity of
$R$ is equivalent to $R\cap R^{-1} \,=\,\tau(R)$ being weakly transitive.} 
\end{mitemize}
\end{definition}

We will use the following simpler example to illustrate many of the concepts and situations in the monograph. For detailed motivations see  \textsf{Ch.}\ref{moex} on motivation and examples.

\begin{examp}\label{agre}

A simple real-life example of a proto-transitive, non transitive relation would be the relation $\mathbb{P}$, defined by 
\[\mathbb{P}xy \;\mathsf{if\;and\; only \;if \;} x \mathrm{\;thinks\; that\;} y \mathrm{\;thinks\; that\; color\; of \;object\;} O \mathrm{\;is \;a \;maroon}.\]

But we will use the following simple example from databases as a persistent one (especially in the chapters on approximation of relations) to illustrate a number of concepts. It has other attributes apart from the main one for illustrating more involved aspects.

Let $\mathcal{I}$ be survey data in table form with column names being for sex, gender, sexual orientations, other personal data and opinions on sexist contexts with each row corresponding to a person. We write 
\[Rab\;\mathsf{if\; and\; only\; if\;} \mathrm{person \;} a \mathsf{\;agrees\; with\;} b's \mathrm{\;opinions}.\]
The predicate \textsf{agrees with} can be constructed empirically or from the data by a suitable heuristic. Often $R$ is a proto-transitive, reflexive relation and this condition can be imposed to complete partial data as well (as a rationality condition). If $a$ agrees with the opinions of $b$, then we will say that $a$ is an \emph{ally} of $b$ - if $b$ is also an ally of $a$, then they are \emph{comrades}. Finding optimal subsets of allies can be an interesting problem in many contexts especially given the fact that responses may have some vagueness in them.  

\end{examp}

\begin{definition}
A binary relation $R$ on a set $S$ is said to be \emph{semi-transitive} on $S$ \ifof 
$S$ satisfies  
\begin{mitemize}
\item {Whenever $\tau(R)a b \& R b c $ holds then $R a c$ follows and }
\item {Whenever $\tau(R)a b \& R c a $ holds then $R c b$ follows.}
\end{mitemize}
\end{definition}

Henceforth we will use $Rxy$ for $(x, y)\in R$ uniformly. $Ref (S),$ $Sym(S),$ $Tol(S),$ 
$r\tau (S),$ $w\tau(S),\, p\tau(S),\, s\tau(S),\, EQ(S)$ will respectively denote the set
of reflexive, symmetric, tolerance, transitive, weakly transitive, pseudo transitive,
semi-transitive and equivalence relations on the set $S$ respectively.  

The following proposition has steep ontological commitments.
\begin{proposition}
For a relation $R$ on a set $S$, the following are satisfied:
\begin{mitemize}
 \item {$R$ is \emph{weakly transitive} \ifof $(R\cap R^{-1})\setminus \Delta_{S}
\subseteq R$.}
\item {$R$ is \emph{transitive} \ifof $(R\cap R^{-1}) \subseteq R$.}
\end{mitemize}
\end{proposition}

By a \emph{pseudo order}, we will mean an antisymmetric, reflexive relation. A
\emph{quasi-order} is a reflexive, transitive relation, while a partial order a
reflexive, antisymmetric and transitive relation.

Let $\alpha \subseteq \rho$ be two binary relations on $S$, then $\rho | \alpha$ will be
the relation on $S|\rho$  defined via $(x, y)\in \rho | \alpha$ \ifof $(\exists b \in x ,
c\in y ) (b, c) \in \rho$. The relation $Q|\tau(Q)$ for a relation $Q$ will be denoted by
$\sigma (Q)$. 

The following are known:

\begin{proposition}
If $Q$ is a quasi-order on $S$, then $Q|\tau(Q)$ is a partial order on $S|\tau(Q)$. 
\end{proposition}

\begin{proposition}
If $R\in Ref(S)$, then $R\in p\tau(S)$ \ifof $\tau(R)\in EQ(S)$. 
\end{proposition}

\begin{proposition}
In general, \[w\tau(S) \,\subseteq\, s\tau (S) \,\subseteq\, p\tau(S) .\]  
\end{proposition}

\begin{proposition}
If $R\in p\tau(S) \cap Ref(S)$, then the following are equivalent:
\begin{description}
\item [A1]{$([a], [b])\in R|\tau (R)$ \ifof $(a, b) \in R$.}
\item [A2]{$R$ is semi-transitive.}
\end{description}
\end{proposition}

In \cite{ICH}, it is proved that 
\begin{theorem}
If $R\in Ref(S)$, then the following are equivalent:
\begin{description}
\item [A3]{$R|\tau (R)$ is a pseudo order on $S|\tau(R)$ and A1 holds.}
\item [A2]{$R$ is semi-transitive.}
\end{description}
\end{theorem}

Note that \emph{Weak transitivity} of \cite{ICH} is \emph{proto-transitivity} here. $Ref (S)$, $r\tau (S),$ $w\tau(S)$, $p\tau(S)$, $EQ(S)$ will respectively denote the set of reflexive, transitive, weakly transitive, proto transitive, and equivalence relations on the set $S$ respectively. Clearly, $w\tau(S) \,\subseteq\, p\tau(S)$.
\begin{proposition}
 $\forall {R\in Ref(S)}(R\in p\tau(S) \leftrightarrow \tau(R)\in EQ(S))$. 
\end{proposition}

\begin{definition}
A \emph{Partial Algebra} $P$ is a tuple of the form \[\left\langle\underline{P},\,f_{1},\,f_{2},\,\ldots ,\, f_{n}, (r_{1},\,\ldots ,\,r_{n} )\right\rangle\] with $\underline{P}$ being a set, $f_{i}$'s being partial function symbols of arity $r_{i}$. The interpretation of $f_{i}$ on the set $\underline{P}$ should be denoted by $f_{i}^{\underline{P}}$, but the superscript will be dropped in this monograph as the application contexts are simple enough. If predicate symbols enter into the signature, then $P$ is termed a \emph{Partial Algebraic System}. (see \cite{BU,LJ} for the basic theory)  
\end{definition}

In a partial algebra, for term functions $p,\,q$, \[p\stackrel{\omega}{=}q\; \mathrm{iff} \;(\forall x\,\in\,dom(p)\,\cap\,dom(q)) p(x)=q(x).\] The \emph{weak strong equality} is defined via,  \[p\stackrel{\omega^{*}}{=}q\; \mathrm{iff} \;(\forall x\,\in\,dom(p)\,=\,dom(q)) p(x)=q(x).\] For two terms $s,\,t$, $s\,\stackrel{\omega}{=}\,t$ shall mean, if both sides are defined then the two terms are equal (the quantification is implicit). $s\,\stackrel{\omega ^*}{=}\,t$ shall mean if either side is defined, then the other is and the two sides are equal (the quantification is implicit).

\section{Nelson Algebras}

By a \emph{De Morgan lattice} $\Delta ML$ we will mean an algebra of the form $L\,=\,\left\langle \underline{L},\, \vee,\, \wedge , \, c , \, 0,\, 1  \right\rangle $ with $\vee,\,\wedge$ being distributive lattice operations and $c$ satisfying
\begin{mitemize}
\item {$x^{cc}\,=\, x\;;\;(x\vee y)^{c}\,=\,x^{c}\,\wedge\, y^{c}\, ;$}
\item {$(x\,\leq\, y\,\leftrightarrow\, y^{c}\, \leq\, x^{c})\;;\; (x\,\wedge\, y)^{c}\,=\,x^{c}\,\vee\, y^{c}\, ;$}
\end{mitemize}
It is possible to define a partial unary operation $\star$, via $x^{\star}\,=\, \bigwedge \{x\, :\, x\,\leq\,x^{c} \}$ on any $\Delta ML$. If it is total, then the $\Delta ML$ is said to be \emph{complete}. In a complete $\Delta ML$ $L$, we have  
\begin{mitemize}
\item {$x^{\star}\,\nleq\, x^{c}\;;\;x^{\star \star}\,=\,x \, ;$}
\item {$(x\,\leq\, y\,\longrightarrow\, y^{\star}\, \leq\, x^{\star})$.}
\item {$x^{c}\,=\, \bigvee \{y\, :\, x^{\star}\,\nleq\, y\}$.}
\end{mitemize}
A $\Delta ML$ is said to be a \emph{Kleene algebra} if it satisfies $x\,\wedge\,x^{c}\, \leq\, y\,\vee\, y^{c} $. If 
$L^{+}\,=\, \{x\,\vee\, x^{c}\, :\, x\in L\}$ and $L^{-}\,=\, \{x\,\wedge\, x^{c}\, :\, x\in L\}$, then in a Kleene algebra we have 
\begin{mitemize}
\item {$(L^{-})^{c}\,=\, L^{+}$ is a filter and $(L^{+})^{c}\,=\, L^{-}$ is an ideal. }
\item {$(\forall a, b\in L^{-})\,a\,\leq\, b^{c} $; $(\forall a, b\in L^{+})\,a^{c}\,\leq\, b$.}
\item {$x\in L^{-}$ if and only if $x \,\leq\, x^{c}$.}
\end{mitemize}

A \emph{Heyting algebra} $K$, is a relatively pseudo-complemented lattice, that is $(\forall a, b)\, a\,\Rightarrow b\,=\, \bigvee \{x\,;\, a\wedge x\,\leq \, b\}\,\in\, K$. 

A \emph{Quasi-Nelson} algebra $Q$ is a Kleene algebra that satisfies $(\forall a, b)\, a\Rightarrow (a^{c}\vee b )\, \in \,Q$. $a\Rightarrow (a^{c}\vee b )$ is abbreviated by $a\rightarrow b$ below. Such an algebra satisfies all of the sentences N1--N4:
\begin{align*}
 x\rightarrow x\,=\, 1\tag{N1}\\
 (x^{c}\vee y)\,\wedge\, (x\rightarrow y)\,=\,x^{c}\vee y\tag{N2}\\
 x\,\wedge\, (x\rightarrow y)\,=\, x\, \wedge\, (x^{c}\vee y)\tag{N3}\\
 x\rightarrow (y\wedge z)\,=\, (x\rightarrow y)\,\wedge\, (x\rightarrow z)\tag{N4}\\
 (x\wedge y)\rightarrow z\,=\, x\rightarrow (y\rightarrow z).\tag{N5}
\end{align*}

A Nelson algebra is a quasi-Nelson algebra satisfying \textsf{N5}. A Nelson algebra can also be defined directly as an algebra of the form $\left\langle A, \vee , \wedge, \rightarrow , c , 0, 1 \right\rangle $ with $\left\langle A, \vee , \wedge,  c , 0, 1 \right\rangle $ being a Kleene algebra with the binary operation $\rightarrow$ satisfying \textsf{N1--N5}.

\section{Granules and Granular Computing Paradigms}

The idea of granular computing is as old as human evolution. Even in the available information on earliest human habitations and dwellings, it is possible to identify a primitive granular computing process (\textsf{PGCP}) at work. This can for example be seen from the stone houses, dating to 3500 BCE, used in what is present-day Scotland. The main features of this and other primitive versions of the paradigm may be seen to be   

\begin{mitemize}
\item {Problem requirements are not rigid.}
\item {Concept of granules may be vague.}
\item {Little effort on formalization right up to approximately the middle of the previous century.}
\item {Scope of abstraction is very limited.}
\item {Concept of granules may be concrete or abstract (relative all materialist viewpoints).}
\end{mitemize}

The precision based granular computing paradigm, traceable to Moore and Shannon's paper \cite{Sha56}, will be referred to as the \emph{classical granular computing paradigm} \textsf{CGCP} is usually understood as the granular computing paradigm (The reader may  note that the idea is vaguely present in \cite{Sha48}). The distinct terminology would be useful to keep track of the differences with other paradigms. CGCP has since been adapted to fuzzy and rough set theories in different ways. 

Granules may be assumed to subsume the concept of information granules -- information at some level of precision. In granular approaches to both rough and fuzzy sets, we are usually concerned with such types of granules. Some of the fragments involved in applying CGCP may be:
\begin{mitemize}
\item {Paradigm Fragment-1: Granules can exist at different levels of precision.}
\item {Paradigm Fragment-2: Among the many precision levels, choose a precision level at which the problem at hand is solved. }
\item {Paradigm Fragment-3: Granulations (granules at specific levels or processes) form a hierarchy (later development).}
\item {Paradigm Fragment-4: It is possible to easily switch between precision levels.}
\item {Paradigm Fragment-5: The problem under investigation may be represented by the hierarchy of multiple levels of granulations.}
\end{mitemize}

The different stages of development of granular computing paradigms are as in the following:  

\begin{mitemize}
\item {Classical Primitive Paradigm till middle of previous century.}
\item {CGCP: Since Shannon's information theory}
\item {CGCP in fuzzy set theory. It is natural for most real-valued types of fuzzy sets, but even in such domains unsatisfactory results are normal. Type-2 fuzzy sets have an advantage over type-1 fuzzy sets in handling data relating to emotion words, for example, but still far from satisfactory.  For one thing linguistic hedges have little to do with numbers. A useful reference would be \cite{LZ9}.}
\item {For a long period (up to 2008 or so), the adaptation of CGCP for RST has been based solely on precision and related philosophical aspects. The adaptation is described for example in \cite{Ya01}. In the same paper the hierarchical structure of granulations is also stressed. This and many later papers on CGCP (like \cite{TYL}) in rough sets speak of structure of granulations. }
\item {Some Papers with explicit reference to multiple types of granules from a semantic viewpoint include \cite{AM69,AM99,SW3,SW,AM105}.}
\item {The axiomatic approach to granularity initiated in \cite{AM99} has been developed by the present author in the direction of contamination reduction in \cite{AM240}. From the order-theoretic/algebraic point of view, the deviation is in a very new direction relative the precision-based paradigm. The paradigm shift includes a new approach to measures. }
\end{mitemize}

There are other adaptations of CGCP to soft computing like \cite{KCM} that we will not consider.

Unless the underlying language is restricted, granulations can bear upon the theory with unlimited diversity. Thus for example in classical \textsf{RST}, we can take any of the following as granulations: collection of equivalence classes, complements of equivalence classes, other partitions on the universal set $S$, other partition in $S$, set of finite subsets of $\mathcal{S}$ and set of finite subsets of $\mathcal{S}$ of cardinality greater than 2. This is also among the many motivations for the axiomatic approach. 

A formal simplified version of the the axiomatic approach to granules is in \cite{AM1800}. The axiomatic
theory is capable of handling most contexts and is intended to permit relaxation of
set-theoretic axioms at a later stage. The axioms are considered in the framework of Rough
Y-Systems (\textsf{RYS}) that  maybe seen as a generalized form of \emph{abstract
approximation spaces} \cite{CC5} and approximation framework
\cite{CD3}. It includes relation-based RST, cover-based RST and more. These structures are
provided with enough structure so that a classical semantic domain (Meta-C) and at least
one rough semantic domain (called Meta-R) of roughly equivalent objects along with
admissible operations and predicates are associable. But the exact way of association is
not something absolute as there is no real end to recursive approximation processes of
objects. 

In the present monograph we will stick to successor, predecessor and related granules generated by elements and will avoid the precision based paradigm.

\chapter{Approximations and Definite Elements in PRAX}


\begin{definition}
By a \emph{Proto Approximation Space} $S$ (\textsf{PRAS for short}), we will mean a pair
of the form $\left\langle \underline{S},\, R \right\rangle  $ with $\underline{S}$ being a
set and $R$ being a proto-transitive relation on it. If $R$ is also reflexive, then it
will be called a \emph{Reflexive Proto Approximation Space} (\textsf{PRAX}) for short).
$\underline{S}$ may be infinite.
\end{definition}

If $S$ is a PRAX or a PRAS, then we will respectively denote \emph{successor neighborhoods},
\emph{inverted successor or predecessor neighborhoods} and \emph{symmetrized
successor neighborhoods} generated by an element $x \in S$ as follows:

\[[x]\,=\, \{y ;\, Ryx\}.\]
\[[x]_{i}\,=\, \{y;\, Rxy\} .\]
\[[x]_{o}\,=\, \{y ;\, Ryx\, \&\, Rxy\}.\] 
Taking these as granules, the associated granulations will be denoted by
$\mathcal{G} \,=\,\{[x]:\, x\in S \}$, $\mathcal{G}_{i}$ and $\mathcal{G}_{o}$ respectively. In all that
follows $S$ will be a \textsf{PRAX} unless indicated otherwise.

\begin{definition}
Definable approximations on $S$ include ($A\subseteq S$):
\begin{align*}
\tag{Upper Proto} A^{u}\,=\, \bigcup _{[x]\cap A \neq \emptyset} {[x]}.\\ 
\tag{Lower Proto} A^{l}\,=\, \bigcup_{[x]\subseteq A} {[x]}.\\
\tag{Symmetrized Upper Proto} A^{u_o}\,=\, \bigcup _{[x]_{o}\cap A \neq \emptyset}
{[x]_{o}}.\\
\tag{Symmetrized Lower Proto} A^{l_o}\,=\, \bigcup_{[x]_{o}\subseteq A} {[x]_{o}}.\\
\tag{Point-wise Upper} A^{u+}\,=\, \{ x \, :\, [x]\cap A \neq \emptyset \}.\\
\tag{Point-wise Lower} A^{l+}\,=\, \{x \,:\,[x]\subseteq A \}\,.
\end{align*}
\end{definition}

\begin{examp}
In the context of our example \ref{agre}, $[x]$ is the \emph{set of allies} $x$, while $[x]_{o}$ is the set of comrades of $x$. $A^{l}$ is the \emph{union of the set of all allies of at least one of of the members of} $A$ if they are all in $A$. $A^{u}$ is the union of the set of all allies of persons having at least one ally in $A$. $A^{l_{+}}$ is the set of all those persons in $A$ all of whose allies are within $A$. $A^{u_{+}}$ is the set of all those persons having allies in $A$.
\end{examp}

\begin{definition}
If $A\subseteq S$ is an arbitrary subset of a  \textsf{PRAX} or a
\textsf{PRAS} $S$, then 
\begin{gather}
A^{ux}\,=\, \bigcup _{[x]_{o}\cap A \neq \emptyset} {[x]}.\\
A^{lx}\,=\, \bigcup_{[x]_{o}\subseteq A} {[x]}. \\
A^{u*}= \bigcup \{[x] : [x]\cap A \neq \emptyset \& (\exists y)
([x],[y])\in \sigma (R), \, (x, y)\in R, \, x\neq y , [y]\subseteq A  \}.\\
A^{l*}= \bigcup\{[x] : [x]\subseteq A  \,\& (\exists y) (([x],[y])\in
\sigma(R),\,  x\neq y, \, [y]\subseteq A ) \}. 
\end{gather}
\end{definition}

The following inverted approximations are also of relevance as they provide Galois
connections in case of point-wise approximations (see \cite{JJ}) under particular
assumptions. Our main approximations of interest will be $l, u, l_{o}, u_{o}$. 

\begin{definition}
In the context of the above definition, the following will be referred to as inverted
approximations:
\begin{gather*}
A^{ui}\,=\, \bigcup _{[x]_{i}\cap A \neq \emptyset} {[x]_{i}}\,\\ 
A^{li}\,=\, \bigcup _{[x]_{i}\subseteq A}  {[x]_{i}}\,\\
A^{\vartriangle}\,=\, \{ x \, :\, [x]_{i}\cap A \neq \emptyset \}\\
A^{\triangledown}\,=\, \{x \,:\,[x]_{i}\subseteq A \}\,\\
\end{gather*}
\end{definition}

\begin{proposition}
In a PRAX $S$ and for a subset $A\subseteq S$, all of the following hold:
\begin{mitemize}
\item {$(\forall x) \,[x]_{o}\subseteq [x]$ }
\item {It is possible that $A^{l}\,\neq \, A^{l+}$ and in general, $A^{l}\,\parallel \,
A^{lo}$.}
\end{mitemize}
\end{proposition}
\begin{proof}
The proof of the first two parts are easy. For the third, we chase the argument up to a
trivial counter example (see the following chapter). 
\[\bigcup_{[x]\subseteq A} {[x]}\,\subseteq \bigcup_{[x]_{o}\subseteq A} {[x]}\, \supseteq
\,\bigcup_{[x]_{o}\subseteq A} {[x]_{o}} \]
\[\bigcup_{[x]_{o}\subseteq A} {[x]_{o}} \supseteq  \bigcup_{[x]\subseteq A} {[x]_{o}}\,
\subseteq \,\bigcup_{[x]\subseteq A} {[x]}. \]
\end{proof}

\begin{proposition}
For any subset $A$ of $S$, \[A^{u_{o}}\,\subseteq\, A^{u}.\] 
\end{proposition}

\begin{proof}
Since $[x]_{o}\cap A\neq \emptyset$, therefore
\[A^{u_{o}}\,=\, \bigcup _{[x]_{o}\cap A \neq \emptyset} {[x]_{o}} \,\subseteq \, \bigcup
_{[x] \cap A \neq \emptyset} {[x]_{o}}\, \subseteq\,A^{u_{o}}\,=\, \bigcup _{[x]\cap A \neq
\emptyset} {[x]}\,=\, A^{u}.\]

\end{proof}

\begin{definition}
If $X$ is an approximation operator, then by a $X$-\emph{definite element},
we will mean a subset $A$ satisfying $A^{X}\,=\, A$. The set of all $X$-definite elements
will be denoted by $\delta_{X}(S)$, while the set of $X$ and $Y$-definite elements ($Y$
being another approximation operator) will be denoted by $\delta_{XY}(S)$. In particular,
we will speak of \emph{lower proto-definite}, \emph{upper proto definite} and
\emph{proto-definite} elements (those that are both lower and upper proto-definite). 
\end{definition}
\begin{theorem}
In a \textsf{PRAX} $S$, the following hold:
\begin{mitemize}
\item {$\delta_{u}(S)\,\subseteq\,\delta_{u_{o}}(S)$, but $\delta_{l_{o}}(S)\,=\,
\delta_{u_{o}}(S)$ and $\delta_{u}(S)$ is a complete sublattice of $\wp (S)$ with respect
to inclusion. }
\item {$\delta_{l}(S)\,\parallel\,\delta_{l_{o}}(S)$ in general. ($\parallel$ means \emph{is not comparable}.)}
\item {It is possible that $\delta_{u}\,\nsubseteq \delta_{u_{o}}$.}
\end{mitemize}
\end{theorem}

\begin{proof}
\begin{mitemize}
\item {As $R$ is reflexive, if $A,\, B $ are upper proto definite, then $A\cup B$ and $A\cap B$
are both upper proto definite. So $\delta_{u}(S)$ is a complete sublattice of $\wp (S)$.} 
\item {If $A\,\in\,\delta_{u} $, then $(\forall x \in A ) [x]\subseteq A$ and 
$(\forall x \in A^{c}) [x]\cap A = \emptyset$.}
\item {So $(\forall x \in A^{c})\, [x]_{o}\cap A = \emptyset $.
But as $A\,\subseteq\,A^{u_{o}}$ is necessary, we must have $A\in \delta_{u_{o}}$.}
\end{mitemize}
\qed
\end{proof}

$A^{u+},\, A^{l+}$ have relatively been more commonly used in the literature and 
have also been the only kind of approximation studied in \cite{JJ} for example (the
inverse relation is also considered from the same perspective).

\begin{definition}
A subset $B\,\subseteq\,A^{l+} $ will be said to be \emph{skeleton} of $A$ if and only if \[\bigcup_{x\in B} [x] \,=\, A^{l} ,\] and the set skeletons of $A$ will be denoted by $\mathbf{sk}(A)$.   
\end{definition}

The skeleton of a set $A$ is important because it relates all three classes of approximations.
\begin{theorem}
In the context of the above definition, we have 
\begin{mitemize}
\item {$\mathbf{sk}(A)$ is partially ordered by inclusion with greatest element $A^{l+}$.}
\item {$\mathbf{sk}(A)$ has a set of minimal elements $\mathbf{sk}_{m}(S)$.}
\item {$\mathbf{sk}(A)\,=\, \mathbf{sk}(A^{l})$}
\item {$\mathbf{sk}(A)\,=\, \mathbf{sk}(B)\,\leftrightarrow\, A^{l}=B^{l}\,\&\, A^{l+}= B^{l+}$.}
\item {If $B \in \mathbf{sk}(A)$, then $A^{l}\,\subseteq\, B^{u}$.}
\item {If $\cap \mathbf{sk}(A)\,=\, B$, then $A^{l_o}\,\cap\, \bigcup_{x\in B} [x]\,=\, \emptyset$.}
\end{mitemize}
\end{theorem}

\begin{proof}
Much of the proof is implicit in other results proved earlier in this chapter.
\begin{mitemize}
\item {If $x\in A^{l}\setminus A^{l+}$, then $[x]\nsubseteq A^{l}$ and many subsets $B$ of $A^{l+}$ are in $\mathbf{sk}(A)$. If $B\subset K\subset A^{l+}$ and $B \in \mathbf{sk}(A)$, then $K\in \mathbf{sk}(A)$. Further we have a minimal elements in the inclusion order (even if $A$ is infinite) by the induced properties of inclusion in $\wp (S)$.}
\item {has been proved above.}
\item {More generally, if we have $A^{l}\,\subseteq\, B\,\subseteq\, A$, then $B^{l}\,=\,A^{l} $. So $\mathbf{sk}(A)\,=\, \mathbf{sk}(A^{l})$. }
\item {Follows from definition.}
\item {If $B \in \mathbf{sk}(A)$, then $A^{l}\,=\,B^{l}\, \subseteq\, B^{u}$. }
\end{mitemize}
\qed
\end{proof}

\begin{theorem}
All of the following hold in PRAX: 
\begin{mitemize}
\item {$(\forall A)\, A^{cl+} = A^{u+c},\,\,A^{cu+} = A^{l+c} $ - that is $l+$ and $u+$ are
mutually dual} 
\item {$u+$ ($l+$ resp.) is a monotone $\vee$- (complete $\wedge$- resp.) morphism.}
\item {$\partial (A) = \partial(A^{c})$, where $\partial$ stands for the boundary operator.}
\item {$\Im (u+)$ (the image of $u+$) is an interior system while $\Im (l+)$ is a closure system.}
\item {$\Im (u+)$ and $\Im (l+)$ are dually isomorphic lattices.}
\end{mitemize}
\end{theorem}

\begin{theorem}
In a \textsf{PRAX}, $(\forall A\in \wp (S))\, A^{l+} \subseteq A^{l},\:\:A^{u+} \subseteq A^{u}$ and all of the following hold.
\begin{align*}
\tag{Bi}(\forall A\in \wp(S))\, A^{ll}\,=\, A^{l}\, \& \, A^{u} \subseteq\, A^{uu} .\\
\tag{l-Cup}(\forall A, B \in \wp(S))\, A^{l}\cup B^{l}\,\subseteq (A\cup B)^{l}.\\
\tag{l-Cap}(\forall A, B \in \wp(S))\, (A\cap B)^{l}\,\subseteq\, A^{l}\cap B^{l}.\\
\tag{u-Cup}(\forall A, B \in \wp(S))\,(A\cup B)^{u}\,=\, A^{u}\cup B^{u}.\\
\tag{u-Cap}(\forall A, B \in \wp(S))\,(A\cap B)^{u} \,\subseteq\, A^{u}\cap B^{u} .\\  
\tag{Dual}(\forall A\in \wp (S))\, A^{lc}\,\subseteq\, A^{cu}.
\end{align*}
\end{theorem}
\begin{proof}
\begin{description}
\item[l-Cup]{For any $A, B\in \wp{S}$, $x\in (A\cup B)^{l}$
\begin{impemize}
\item {$(\exists y\in (A\cup B) )\, x\in [y]\, \subseteq \, A\cup B $.}
\item {$(\exists y\in A )\, x\in [y]\, \subseteq \, A\cup B $ or $(\exists y\in  B
)\, x\in [y]\, \subseteq \, A\cup B $.}
\item {$(\exists y\in A) \, x\in [y]\, \subseteq \, A $ or $(\exists y\in A) \, x\in [y]\,
\subseteq \, B $ or $(\exists y\in B) \, x\in [y]\, \subseteq \, A $ or $(\exists y\in B)
\, x\in [y]\, \subseteq \, B $ - this is implied by $x\in A^{l} \cup B^{l}$.}
\end{impemize}}
\item[l-Cap]{For any $A, B\in \wp{S}$, $x\in (A\cap B)^{l}$
\begin{impemize}
\item {$ x\in \, A\cap B $}
\item {$(\exists y\in A\cap B )\, x\in [y]\, \subseteq \, A\cap B $ and $x\,\in\, A,\;
x\,\in\, B$}
\item {$(\exists y\in A) \, x\in [y]\, \subseteq \, A $ and $(\exists y\in B) \, x\in
[y]\,\subseteq \,B $ - Clearly this statement implies  $x\in A^{l} \& x\in  B^{l}$,
but the converse is not true in general.}
\end{impemize}} 
\item[u-Cup]{$x\,\in\, (A\cup B)^{u}$
 \begin{impemize}
  \item {$x\,\in\, \bigcup_{[y]\cap (A\cup B)\neq \emptyset } [y] $}
  \item {$x\,\in\, \bigcup_{([y]\cap A)\cup ([y]\cap B)\neq \emptyset}  $}
  \item {$x\,\in\,\bigcup_{[y]\cap A\neq \emptyset} [y]  $ or $x\,\in\,\bigcup_{[y]\cap
B\neq \emptyset} [y]  $}
  \item {$x\in A^{u}\cup B^{u}$.}
 \end{impemize}}
\item[u-Cap] {By monotonicity, $(A\cap B)\,\subseteq\, A^{u}$ and $(A\cap B)\,\subseteq\,
B^{u}$, so $(A\cap B)^{u}\,\subseteq\, A^{u}\cap B^{u} $.}
\item[Dual]{If $z\in A^{lc}$, then $ z\in [x]^{c}$ for all $[x]\subseteq A$ and either,
$z\in A\setminus A^{l}$ or $z\in A^{c}$. If $z\in A^{c}$ then $z\in A^{cu}$.
If $z\in A\setminus A^{l}$ and $z\neq A^{cu \setminus A^{c}}$ then $[z]\cap A^{c}\,=\,
\emptyset$. But this contradicts $z\notin A^{cu} \setminus A^{c}$.
So $(\forall A\in \wp (S))\, A^{lc}\,\subseteq\, A^{cu} .$}
\end{description}
\qed
\end{proof}

\begin{theorem}
In a \textsf{PRAX} $S$, all of the following hold:
\begin{align}
{(\forall A, B\in \wp (S))\, (A\cap B)^{l+}\,=\, A^{l+}\cap B^{l+}}. \\
{(\forall A, B\in \wp (S))\, A^{l+}\cup B^{l+}\, \subseteq (A\cup B)^{l+}}. \\
{(\forall A \in \wp (S))\, (A^{l+})^{c} = (A^{c})^{u+} \,\& \,A^{l+}\subseteq A^{l_{o}}}.
\end{align}
\end{theorem}

\begin{proof}
\begin{enumerate}
\item {$x\in (A\cap B)^{l+}$ 
 \begin{impemize}
 \item {$[x]\subseteq A\cap B $} 
 \item {$[x]\subseteq A $ and $[x]\subseteq B $}
 \item {$x\in xA^{l+} $ and $x\in B^{l+}$.}
 \end{impemize}}
\item {$x \in A^{l+}\cup B^{l+} $
  \begin{impemize}
  \item {$[x]\subseteq A^{l+} $ or $[x]\subseteq B^{l+}$}
  \item {$[x]\subseteq A $ or $[x]\subseteq B$}
  \end{impemize}
$\Rightarrow [x]\subseteq A\cup B\, \Leftrightarrow\, x\in (A\cup B)^{l+} $.}
\item {$z\in A^{l+c}$
 \begin{impemize}
 \item {$z\,\notin\, A^{l+}$}
 \item {$[z]\,\nsubseteq A $}
 \item {$z\cap A^{c}\,\neq \,\emptyset $}
 \end{impemize}
}
\end{enumerate}
\qed
\end{proof}

\begin{theorem}
If $u+,\, l+ $ are treated as self maps on the power-set $\wp (S)$, $S$ being a PRAX or a
\textsf{PRAS} then all of the following hold:
\begin{mitemize}
\item {$(\forall x)\, x^{cl+}=x^{u+c},\,\,x^{cu+}=x^{l+c} $ - that is $l+$ and $u+$ are
mutually dual} 
\item {$l+,\, u+$ are monotone.}
\item {$l+$ is a complete $\wedge$-morphism, while $u+$ is a $\vee$-morphism.}
\item {$\partial (x) = \partial(x^{c})$, where partial stands for the boundary operator.}
\item {$\Im (u+)$ is an interior system while $\Im (l+)$ is a closure system.}
\item {$\Im (u+)$ and $\Im (l+)$ are dually isomorphic lattices.}
\end{mitemize}
\end{theorem}

\begin{theorem}
\[\mathrm{In\; a\;} \mathsf{PRAX}\; S, (\forall A \subseteq S)\, A^{l+} \subseteq A^{l},\:\:A^{u+} \subseteq A^{u}.\]
\end{theorem}

\begin{proof}
\begin{mitemize}
\item {If $x\in A^{l+}$, then $[x]\subseteq A$ and so $[x]\subseteq A^{l},\, x\in A^{l}$.}
\item {If $x\in A^{l}$, then $(\exists y \in A) [y]\subseteq A,\, Rxy$. But it is possible
that $[x] \nsubseteq A$, therefore it is possible that $x\notin A^{l+}$ and
$A^{l}\nsubseteq A^{l+}$.}
\item {If $x\in A^{u+}$, then $[x]\cap A \neq \emptyset$, so $x\in A^{u}$.}
\item {So $A^{u+}\subseteq A^{u}$.}
\item {Note that $x\in A^{u}$, \ifof $(\exists z\in S) \,x\in [z],\, [z]\cap A\neq
\emptyset $, but this does not imply $x\in A^{u+}$. }
\end{mitemize}
\qed
\end{proof}

\begin{theorem}
In a \textsf{PRAX} $S$, all of the following hold:
\begin{align}
{(\forall A\in \wp (S))\, A^{l+}\,\subseteq\, A^{l_{o}} }. \\
{(\forall A\in \wp (S))\, A^{u_{o}}\,\subseteq\, A^{u+} }. \\
{(\forall A\in \wp (S))\, A^{lc}\,\subseteq\, A^{cu}}. 
\end{align}
\end{theorem}

\begin{proof}
\begin{enumerate}
\item {
\begin{mitemize}
\item {If $x\in A^{l+}$, then $[x]\subseteq A$.}
\item {But as $[x]_{o}\subseteq [x]$, $A^{l+}\subseteq A^{l_{o}}$.}
\end{mitemize}}
\item {This follows easily from definitions.}
\item {
\begin{mitemize}
\item {If $z\in A^{lc}$, then $ z\in [x]^{c}$ for all $[x]\subseteq A$ and either, $z\in
A\setminus A^{l}$ or $z\in A^{c}$.}
\item {If $z\in A^{c}$ then $z\in A^{cu}$.}
\item {If $z\in A\setminus A^{l}$ and $z\neq A^{cu \setminus A^{c}}$ then $[z]\cap A^{c}\,=\,
\emptyset$.}
\item {But this contradicts $z\notin A^{cu} \setminus A^{c}$.}
\item {So $(\forall A\in \wp (S))\, A^{lc}\,\subseteq\, A^{cu} .$}
\end{mitemize}}
\end{enumerate}
\qed
\end{proof}

From the above, we have the following relation between approximations in general ($A^{u_+}
\longrightarrow A^{u}$ should be read as  $A^{u_+}$ \emph{is included in} $A^u$):

\begin{figure}[h]
\begin{center}
\begin{tikzpicture}[node distance=2cm, auto]
\node (Alp) {$A^{l_{+}}$};
\node (A0) [right of=Alp] {};
\node (A) [right of=A0] {$A$};
\node (Al) [above of=A0] {$A^{l}$};
\node (Alo) [below of=A0] {$A^{l_{o}}$};
\node (Auo) [right of=A] {$A^{u_{o}}$};
\node (Aup) [right of=Auo] {$A^{u_{+}} $};
\node (Au) [right of=Aup] {$A^{u}$};
\draw[->] (Alp) to node {}(Al);
\draw[->] (Alp) to node {}(Alo);
\draw[->] (Al) to node {}(A);
\draw[->] (Alo) to node {}(A);
\draw[->] (A) to node {}(Auo);
\draw[->] (Auo) to node {}(Aup);
\draw[->] (Aup) to node {}(Au);
\end{tikzpicture}
\end{center}
\caption{Relationship Between Approximations}
\end{figure}
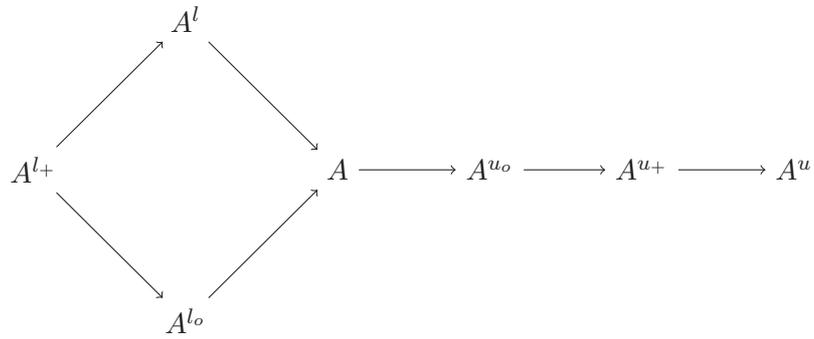

If a relation $R$ is purely reflexive and not proto-transitive on a set $S$, then the
relation $\tau(R)\,=\, R\cap R^{-1}$ will not be an equivalence and for a $A\subset S$,
it is possible that $A^{u_{o} l}\subseteq A $ or $A^{u_{o} l}\parallel A$ or $A\subseteq
A^{u_{o} l}$.

\chapter{Motivation and Examples}\label{moex}

Generalized transitive relations occur frequently in general information systems,
but are often not recognized as such and there is hope for improved semantics and
\textsf{KI} relative the situation for purely reflexive relation based \textsf{RST}. Not
all of the definable approximations have been investigated in even closely related
structures of general \textsf{RST}. Contamination-free semantics \cite{AM240} for the
contexts are also not known. Finally these relate to \textsf{RYS} and variants. A proper
characterization of roughly equal (requal) objects is also motivated by \cite{AM240}.

\section{Abstract Example}
Let $\S\,=\, \{a, b, c, e, f, g, h, l, n\}$ and let $R$  be a binary relation on it 
defined via
\begin{align*}
R\,=\,& \{(a,\,a),\,
(l,\,l),\,(n,\,n),\,(n,\,h),\,(h,\,n),\,(l,\,n),\,(g,\,c),\,(c,\,g)\\
 & (g,\,l),\,(b,\,
g), \,(g,\,b ), \, (h, \,g ), \,(a,\,b),\,(b,\,c),\,(h,\,a),\,(a,\,c)\}.
\end{align*}

Then $\left\langle S,\, R\right\rangle $ is a \textsf{PRAS}.
 
If $P$ is the reflexive closure of $R$ (that is
$P\,=\, R\cup \Delta_{S}$), then $\left\langle S,\, P\right\rangle $ is a \textsf{PRAX}.
The successor neighborhoods associated with different elements of $S$ are as follows
(\textbf{E} is a variable taking values in $S$):

\begin{table}[h]
\caption{Successor Neighborhoods}
\begin{small}
\begin{tabular}[h]{|c|c|c|c|c|c|c|c|c|c|}
\hline
\textbf{E} & $a\,$ & $ b\,$ & $c\,$ & $g\,$ & $e\,$ & $f \,$ & $h\,$ & $l\,$ & $n\,$\\
\hline
\textbf{$[E]$} & $\{a, h \}$ & $\{b, c, g \}$ & $\{b, c, g \}$ & $\{b, c,
g, h \}$ &$\{e \}$ & $\{ f\}$ & $\{h,n \}$ & $\{l,g \}$ & $\{n,l,g,h\}$\\
\hline
\textbf{$[E]_{o}$} & $\{ a\}$ & $\{b, c, g \}$ & $\{ b, c, g\}$ &$\{b,
c, g \}$ & $\{e\}$ & $\{ f\}$ & $\{h, n \}$ & $\{ l\}$ &$\{n,h \}$\\
\hline
\end{tabular} 
\end{small}
\end{table}

\begin{gather*}
\mathrm{If}\; A\,=\,\{a,\, h,\, f\},\\
\mathrm{then}\; A^{l}\,=\,\{a,\, h,\, f\},\\
A^{l_o}\,=\,\{a,\, f\}\; \mathrm{and}\;A^{l_o}\,\subset \, A^{l}. \\
\mathrm{If}\; F\,=\,\{l\},\\
\mathrm{then}\; F^{l}\,=\,\emptyset,\; F^{l_{o}}\,=\,F \\
\mathrm{and}\;F^{l}\,\subset \, F^{l_{o}}. 
\end{gather*}

Now let $Z\,=\, N \cup S \cup X$, where $N$ is the set of naturals, $X$ is the set of
elements of the infinite sequences $\{x_{i}\},\,\{y_j\} $. Let $Q$ be a relation on $Z$
such that 
\begin{gather}
Q\cap S^{2} = P,\\
Q\cap N^{2} \;\mathrm{is\; some\; equivalence},\\
(\forall i\in N) (i, x_{3i +1}),\, (x_{2i}, i), \, (x_{i}, x_{i+1}),\,(y_{i}, y_{i+1})\in Q.
\end{gather}
$Q$ is then a proto-transitive relation. For any $i\in N$, let $P_{i}\,=\,\{y_{k}:\,k\neq 2j \& k < i \}\,\cup\,\{x_{2j}:\, 2j <
i\}$ - this will be used in later chapters. The extension of the example to involve nets and densely ordered subsets is standard.

\section{Caste Hierarchies and Interaction}

The caste system and religion are among the deep-seated evils of Indian society that
often cut across socio-economic classes and level of education. For the formulation of
strategies aimed at large groups of people towards the elimination of such evils it would
be fruitful to study interaction of people belonging to different castes and religions on
different social fronts.  

Most of these castes would have multiple subcaste hierarchies in addition.
Social interactions are necessarily constrained by their type and untouchability
perception. If $x, \, y$ are two castes, then with respect to a possible social
interaction $\alpha$, people belonging to $x$ will either regard people belonging to $y$
as untouchable or otherwise. As the universality is so total, it is possible to write
$\mathbb{U}_{\alpha}xy$ to mean that $y$ is untouchable for $x$ for the interaction
$\alpha$. Usually this is a asymmetric relation and $y$ would be perceived as a
\emph{lower caste} by members of $x$ and many others.

Other predicates will of course be involved in deciding on the possibility of the social
interaction, but if $\mathbb{U}_{\alpha} xy$ then the interaction is forbidden relative
$x$. If $\alpha$ is "context of possible marriage", then the complementary relation
($\mathbb{C}_{\alpha}$ say) is a reflexive proto-transitive relation. For various other
modes of interaction similar relations may be found.

In devising remedial educational programmes targeted at mixed groups, it would be important
to understand approximate perceptions of the group and the semantics of PRAX
would be very relevant. 

\section{Compatibility Prediction Models}

When we want to predict compatibility among individuals or objects, then the following model can be used. Specific examples include situations involving data from dating sites like OK-Cupid. 

Let one woman be defined by a sequence of sets of features $a_1 ,\, \ldots,\, a_n$ at different temporal instants and another woman by  
$b_1 ,\, \ldots,\, b_n$. Let $\omega (a_{i},\, b_{i}) $ be the set of features that are desired by $a_{i}$, but missing in $b_{i}$. Let $\rho$ be an equivalence relation on a subset $K$ of $S$ -- the set of all features, that determines the classical rough approximations $l_{\rho}, u_{\rho}$ on $\wp (K)$.

Let $(a, b) \in R$  if and only if $(\omega (a_{n},\, b_{n})^{l_{\rho}}$ is \emph{small} (for example, that can mean being an atom of $\wp (K)$). The predicate $R$ is intended to convey \emph{may like to be related}. In dating sites, this is understood in terms of profile matches: if a woman's profile matches another woman's and conversely and similarly with another woman's, then the other two woman are assumed to be mutually compatible.

\begin{proposition}
 $R$ is a proto-transitive relation and $\left \langle \underline{S},\, R \right \rangle $ is a \textsf{PRAS}.
\end{proposition}

\begin{proof}
Obviously $R$ need not be reflexive or symmetric in general.

If $(a, b),\, (b, c),\, (b, a),\, (c, b) \,\in \, R$, then $(a, c),\, (c, a) \, \in\, R$ is a reasonable rule.
\end{proof}

So we have a concrete example of a \textsf{PRAS} that is suggestive of many more practical contexts. 

\section{Indeterminate Information System Perspective}

It is easy to derive \textsf{PRAX} from population census, medical, gender studies and other databases and these correspond to information systems. We make the connection clearer through this example.

If our problem is to classify a specific population $O$, for a purpose based on scientific data on sex, gender continuum, sexual orientation and other factors, then our data base would be an indeterminate information system of the form  \[\mathcal{I}\,=\, \left\langle O,\, At,\, \{V_{a} :\, a\in At\},\, \{\varphi _{a} :\, a\in At\}  \right\rangle ,\] where $At$ is a set of attributes, $V_{a}$ a set of possible values corresponding to the attribute $a$ and $\varphi_{a}: O\, \longmapsto\, \wp(V_{a}) $ the valuation function. Sex is determined by many attributes corresponding to hormones, brain structure, karyotypes, brain configuration, anatomy, clinical sex etc. We can associate free/bound values of over six hormones, the values of which vary widely over populations. Suppose we are interested in a subset of attributes for which the inclusion/ordering of values (corresponding to any one of the 
attributes in the subset) of an object in another is relevant. We may, for example, be interested in patterns in sexual compatibility/relationships corresponding to such inclusions. This relation is proto-transitive. Formally for a $B\,\subseteq \, At$, if we let $(x,\,y)\,\in\, \rho_{B}$ if and only if $(\exists a\in B) \varphi_{a} x \,\subseteq \, \varphi_{a} y$, then $\rho_{B}$ is often proto-transitive via another predicate on $B$.

\chapter{Algebras of Rough Definite Elements}

In this chapter we prove key results on the fine structure of definite elements.

\begin{theorem}
On the set of proto definite elements $\delta_{lu}(S)$ of a \textsf{PRAX} $S$, we can
define the following:
\begin{align}
{x\wedge y\,\stackrel{\Delta}{=}\, x\cap y .}\\
{x\vee y\,\stackrel{\Delta}{=}\, x\cup y .}\\
{0\,\stackrel{\Delta}{=}\, \emptyset .}\\
{1\,\stackrel{\Delta}{=}\,S .}\\ 
{x^{c}\,\stackrel{\Delta}{=}\,S \setminus x .}
\end{align}
\end{theorem}

\begin{proof}
We need to show that the operations are well defined. Suppose  $x,\,y$ are proto-definite
elements, then
\begin{enumerate}
\item {\[(x\cap y)^{u}\subseteq x^{u}\cap y^{u} \,=\, x\cap y.\] 
\[(x\cap y)^{l}\,=\, (x^{u}\cap y^{u})^{l} \,=\, (x \cap y)^{ul} = (x\cap y)^{u}
\,=\, x\cap y .\] Since $a^{ul} = a^{u}$ for any $a$.}
\item {\[(x\cup y)^{u}\,=\, x\cup y\,=\, x^{l}\cup y^{l}\subseteq (x\cup y)^{l}.\]}
\item {$0\,\stackrel{\Delta}{=}\, \emptyset$ is obviously well defined.}
\item {Obvious.}
\item {Suppose $A\in \delta_{lu}(S)$, then $(\forall z\in A^{c}) \,[z]\cap A \,=\,
\emptyset$ is essential, else $[z]$ would be in $A^{u}$. This means $[z]\subseteq A^{c}$
and so $A^{c}\,=\, A^{cl}$. If there exists a $a\in A$ such that $[a]\cap A^{c}\neq
\emptyset$, then $[a]\subseteq A^{u}\,=\, A$. So $A^{c}\in \delta_{lu}(S)$. }
\end{enumerate}
\qed
\end{proof}

\begin{theorem}
The algebra $\delta_{proto}(S)\,=\, \left\langle \delta_{lu}(S), \vee,\wedge, c , 0, 1
\right\rangle $ is a Boolean lattice.
\end{theorem}

\begin{proof}
Follows from the previous theorem. The lattice order can be defined via,
$x\leq y$ \ifof \, $\,x\cup y = y$ and $x\cap y = x$.

\end{proof}

\chapter{The Representation of Roughly Objects}

The representation of roughly equal elements in terms of definite elements are well known in case of classical rough set theory. In case of more general spaces including tolerance spaces \cite{AM240}, most authors have been concerned with describing the interaction of rough approximations of different types and not of the interaction of roughly equal
objects. Higher order approaches, developed by the present author as in \cite{AM99} for
bitten approximation spaces, permit constructs over sets of roughly equal objects. In the
light of the contamination problem \cite{AM99,AM240}, it would be an improvement to
describe without higher order constructs. In this chapter a new method of representing
roughly equal elements based on expanding concepts of definite elements is developed. 

\begin{definition}\label{lesspre}
On $\wp (S)$, we can define the following relations:
\begin{gather*}
A\,\preceq \,B\,\mathrm{if \; and \; only \; if} \, A^{l}\subseteq B^{l}\,\& \,A^{u}\subseteq B^{u}.\tag{Rough Inclusion}\\
A\,\approx\, B\,\mathrm{if \; and \; only \; if}\, A\,\preceq\,B \,\& \, B\,\preceq\, A. \tag{Rough Equality}
\end{gather*}
\end{definition}

\begin{proposition}
The relation $\preceq$ defined on $\wp (S)$ is a bounded partial order and $\approx$ is an equivalence. The quotient $\wp(S)|\approx$ will be said to be the set of \emph{roughly equivalent objects}.
\end{proposition}

\begin{definition}
A subset $A$ of $\wp (S)$ will be said to a set of \emph{roughly equal} elements \ifof 
\[(\forall x, y\in A)\, x^{l}\,=\, y^{l}\,\& \, x^{u}\,=\,y^{u}.\] 
It will be said to be \emph{full} if no other subset properly including $A$ has the
property. 

\end{definition}

Relative the situation for a general \textsf{RYS}, we have
\begin{theorem}[Meta-Theorem]
In a \textsf{PRAX} $S$, full set of roughly equal elements is necessarily a union of
intervals in $\wp(S)$.
\end{theorem}

\begin{definition}
A non-empty set of non singleton subsets $\alpha\,=\, \{x\,:\,x\subseteq \wp(S)\}$ will be
said to be a \emph{upper broom} \ifof all of the following hold:
\begin{gather*}
{(\forall x, y \in \alpha)\, x^{u}\,=\, y^{u} . }\\
{(\forall x, y \in \alpha)\, x\,\parallel \, y. }\\
{\mathrm{If\;} \alpha \subset \beta, \mathrm{\;then\;} \beta \mathrm{\;fails\; to\; satisfy\; at\; least\;
one \;of \;the\; above\; two\; conditions.}}
\end{gather*}
The set of upper brooms of $S$ will be denoted by $\pitchfork (S)$.
\end{definition}

\begin{definition}
A non-empty set of non singleton subsets $\alpha\,=\, \{x\,:\,x\subseteq \wp(S)\}$ will be
said to be a
\emph{lower broom} \ifof all of the following hold:
\begin{gather}
{(\forall x, y \in \alpha)\, x^{l}\,=\, y^{l}\neq x  .}\\
{(\forall x, y \in \alpha)\, x\,\parallel \, y .}\\
{\mathrm{\;If\;} \beta \subset \alpha \& Card(\beta)\geq 2, \mathrm{\;then\;} \beta \mathrm{\;fails\; to\; satisfy\; condition\;}(1) \;\mathrm{or}\;(2). }
\end{gather}
The set of lower brooms of $S$ will be denoted by $\psi (S)$.
\end{definition}

\begin{proposition}
If $x\in \delta_{lu}(S)$ then $\{x\}\notin \pitchfork (S)$ and $\{x\}\notin \psi (S)$.
\end{proposition}

In the next definition, the concept of union of intervals in a partially ordered set is
modified in a way for use with specific types of objects.  

\begin{definition}
By a \emph{bruinval}, we will mean a subset of $\wp (S)$ of one of the following forms:
\begin{mitemize}
\item {Bruinval-0: Intervals of the form $(x, y),\,[x, y),\, [x, x],\, (x, y]$ for
$x,y\in \wp(S)$.}
\item {Open Bruinvals: Sets of the form $[x, \alpha)\,=\, \{z\, : \, x\, \leq\, z\,
< \, b\,\&\, b\in \alpha\}$, $(x, \alpha]\,=\, \{z\, : \, x\, < \, z\,
\leq\, b\,\&\, b\in \alpha\}$ and $(x, \alpha)\,=\, \{z\, : \, x\, < \, z\,
< \, b\,, b\in \alpha\}$ for $\alpha \in \wp(\wp (S))$.}
\item {Closed Bruinvals: Sets of the form $[x, \alpha]\,=\, \{z\, : \, x\, \leq\, z\,
\leq\, b\,\&\, b\in \alpha\}$ for $\alpha \in \wp(\wp (S))$.}
\item {Closed Set Bruinvals: Sets of the form $[\alpha, \beta]\,=\,\{z\, : \, x\, \leq\,
z\,\leq\, y\,\&\, x\in \alpha \& y\in \beta\}$ for $\alpha , \beta \in \wp(\wp (S))$ }
\item {Open Set Bruinvals: Sets of the form $(\alpha, \beta)\,=\,\{z\, : \, x\, < \,
z\,< \, y\,, x\in \alpha \& y\in \beta\}$ for $\alpha , \beta \in \wp(\wp (S))$.}
\item {Semi-Closed Set Bruinvals: Sets of the form $[[\alpha, \beta]]$ defined as
follows:  $\alpha = \alpha_{1}\cup\alpha_{2} $, $\beta = \beta_{1}\cup\beta_{2}$ and
$[[\alpha,\beta]]\,=\,(\alpha_{1},\beta_{1})\cup
[\alpha_{2},\beta_{2}]\cup(\alpha_{1},\beta_{2}]\cup [\alpha_{2},\beta_{1})  $ for $\alpha
, \beta \in \wp(\wp (S))$.} 
\end{mitemize}
\end{definition}

In the example of the second chapter, the representation of the rough object
$(P_{i}^{l}, P_{i}^{u})$ requires set bruinvals.

\begin{proposition}
If $S$ is a \textsf{PRAX}, then a set of the form $[x, y]$ with $x, y\in \delta_{lu} (S)$
will be a set of roughly equal subsets of $S$ \ifof $x\,=\, y$. 
\end{proposition}

\begin{proposition}
A bruinval-0 of the form $(x, y)$  is a full set of roughly equal
elements if 
\begin{center}
\begin{mitemize}
\item {$x, y \in \delta_{lu} (S)$,}
\item {$x$ is covered by $y$ in the order on $\delta_{lu} (S)$.}
\end{mitemize} 
\end{center}
\end{proposition}

\begin{proposition}
If $x, y\in \delta_{lu} (S)$ then sets of the form $[x, y),\, (x,
y]$ cannot be a non-empty set of roughly equal elements, while those of the form $[x, y]$
can be \ifof $x = y$. 
\end{proposition}

\begin{proposition}
A bruinval-0 of the form $[x, y)$  is a full set of roughly equal
elements if 
\begin{center}
\begin{mitemize}
\item {$x^{l}, y^{u}\in \delta_{lu} (S)$, $x^{l} = y^{l}$ and $x^{u} = y^{u}$,}
\item {$x^{l}$ is covered by $y^{u}$ in $\delta_{lu} (S)$ and}
\item {$x\setminus (x^{l})$ and $y^{u}\setminus y$ are singletons}
\end{mitemize} 
\end{center}
\end{proposition}

\begin{remark}
In the above proposition the condition $x^{l}, y^{u}\in \delta_{lu} (S)$, is not
necessary.  
\end{remark}

\begin{theorem}
If a bruinval-0 of the form $[x, y]$ satisfies 
\begin{gather*}
{x^{l} = y^{l} = x\, \&\, x^{u} = y^{u} .}\\
{Card(y^{u}\setminus y) = 1.}
\end{gather*}
then  $[x, y]$ is a full set of roughly equal objects.
\end{theorem}

\begin{proof}
Under the conditions, if $[x, y]$ is not a full set of roughly equal objects, then there
must exist at least one set $h$ such that $h^{l} = x$ and $h^{u} = y^{u}$ and $h\notin
[x, y]$. But this contradicts the order constraint $x^{l}\, \leq \, h \, y^{u}$. Note
that $y^{u}\notin [x, y]$ under the conditions.
\qed 
\end{proof}

\begin{theorem}
If a bruinval-0 of the form $(x, y]$ satisfies 
\begin{gather*}
{x^{l} = y^{l} = x \;\&\; (\forall z\in (x, y])\,z^{u} = y^{u},}\\
{Card(y^{u}\setminus y)=1.}
\end{gather*}
then $(x, y]$ is a full set of roughly equal objects, that does not intersect the full set
$[x, x^{u}]$.
\end{theorem}

\begin{proof}
By monotonicity it follows that $(x, y]$ is a full set
of roughly equal objects. then there
must exist at least one set $h$ such that $h^{l} = x$ and $h^{u} = y^{u}$ and $h\notin
[x, y]$. But this contradicts the order constraint $x^{l}\, \leq \, h \, y^{u}$. Note
that $y^{u}\notin [x, y]$ under the conditions.
\qed
\end{proof}

\begin{theorem}
A bruinval-0 of the form $(x^{l}, x^{u})$ is not always a set of roughly equal
elements, but will be so when $x^{uu}= x^{u}$. In the latter situation it will be full if
$[x^{l},x^{u})$ is not full. 
\end{theorem}

The above theorems essentially show that the description of rough objects depends on too
many types of sets and the order as well. Most of the considerations extend to other
types of bruinvals as is shown below and remain amenable. 

\begin{theorem}
An open bruinval of the form $(x, \alpha)$ is a full set of roughly equal elements \ifof
\begin{gather*}
{\alpha \,\in\,\pitchfork (S).}\\
{(\forall y\in \alpha)\, x^{l}\,=\, y^{l}, x^{u}\,=\, y^{u}}\\
{(\forall z) (x^{l}\subseteq z \subset x \longrightarrow  z^{u}\subset x^{u}).}
\end{gather*}
\end{theorem}

\begin{proof}
It is clear that for any $y\in \alpha$, $(x,y)$ is a convex interval and all elements in
it have same upper and lower approximations. The third condition ensures that $[z,\alpha)$
is not a full set for any $z\in [ x^{l}, x)$.
\qed 
\end{proof}

\begin{definition}
An element $x\in \wp (S)$ will be said to be a \emph{weak upper critical element relative}
$z\subset x$ \ifof $(\forall y\in \wp (S))\, (z\,=\,y^{l} \,\&\, x \subset y\,
\longrightarrow x^{u}\subset y^{u} )$.

An element $x\in \wp (S)$ will be said to be an \emph{upper critical element relative}
$z\subset x$ \ifof
$(\forall v,\,y\in \wp (S))\, (z\,=\,y^{l}\,=\,v^{l} \,\&\,v \subset x \subset y\,
\longrightarrow \,v^{u} = x^{u}\subset
y^{u} )$. Note that the inclusion is strict. 

An element $a$ will be said to be \emph{bi-critical relative} $b$ \ifof $(\forall x,\,
y\in \wp (S)) (a\subset x \subseteq y \subset b\,\longrightarrow\, x^{u} = y^{u}\,\&\,
x^{l} = y^{l}\,\&\, x^{u} \subset b^{u}\,\&\, a^{l}\subset x^{l})$. 
\end{definition}
If $x$ is an upper critical point relative $z$, then $[z,x)$ or $(z,x)$ is a set of
roughly equivalent elements.

\begin{definition}
An element $x\in \wp (S)$ will be said to be an \emph{weak lower critical element
relative} $z\supset x$ \ifof
$(\forall y\in \wp (S))\, (z\,=\,y^{u}\,\&\,y \subset x\, \longrightarrow \,y^{l}\subset
x^{l} )$.

An element $x\in \wp (S)$ will be said to be an \emph{lower critical element
relative} $z\supset x$ \ifof
$(\forall y, v \in \wp (S))\, (z\,=\,y^{u}\,=\, v^{u}\&\,y \subset x \subset v\,
\longrightarrow \,y^{l}\subset
x^{l}= v^{l} )$.

An element $x\in \wp (S)$ will be said to be an \emph{lower critical element} \ifof
$(\forall y\in \wp (S))\, (y \subset x\, \longrightarrow \,y^{l}\subset x^{l} )$
An element that is both lower and upper critical will be said to be \emph{critical}.
The set of upper critical, lower critical and critical elements respectively will be
denoted by $UC(S)$, $LC(S)$ and $CR(S)$. 
\end{definition}

\begin{proposition}
In a \textsf{PRAX}, every upper definite subset is also upper critical, but the converse
need not hold. 
\end{proposition}

The most important thing about the different lower and upper critical points is that they
help in determining full sets of roughly equal elements by determining the boundaries of
intervals in bruinvals of different types. 

\section{Types of Associated Sets}

Because of reflexivity it might appear that lower approximations in \textsf{PRAX} and classical \textsf{RST} are too similar at least in the perspective of lower definite objects. It is necessary to classify subsets of a \textsf{PRAX} $S$, to see the differences relative the behavior of lower approximations in classical \textsf{RST}. We will make use of this in some of the semantics as well.

\begin{definition}
For each element $x\in \wp (S)$ we can associate the following sets:
\begin{gather}
F_{0}(x)\,=\, \{y \,:\,(\exists a\in x^{c}) \, Rya\, \& y\in x \} \tag{Forward Looking}\\
F_{1}(x)\,=\, \{y \,:\,(\exists a\in x^{c}) \, Rya\, \& R zy \,\& z\in x \} \tag{1-Forward Looking}\\
\pi _{0} (x)\,=\, \{y \,:\, y\in x \,\&\, (\exists a\in x^{c})\,Ray  \} \tag{Progressive}\\
St (x)\,=\, \{y\,:\, [y]\,\subseteq\, x\, \&\, \neg (y\in F_{0}(x)) \} \tag{Stable}\\
Sym (x)\,=\, \{y\,:\, y\in x \,\&\, (\forall z\in x)(Ryz\, \leftrightarrow\, Rzy) \} \tag{Relsym}
\end{gather}
\end{definition}

\emph{Forward looking} set associated with a set $x$ includes those elements not in $x$ whose successor neighborhoods intersect $x$. Elements of the set may be said to be relatively forward looking. \emph{Progressive} set of $x$ includes those elements of $x$ whose successor neighborhoods are not included in $x$. It is obvious that progressive elements are all elements of $x\setminus x^{l} $. Stable elements are those that are strongly within $x$ and are not directly reachable in any sense from outside. $Sym(x)$ includes those elements in $x$ which are symmetrically related to all other elements within $x$. 

Even though all these are important we cannot easily represent them in the rough domain. Their approximations have the following properties:

\begin{proposition}
In the above context, we have 
\begin{gather*}
(\pi_{0}(x))^{l}\,=\, \emptyset \, \& \, (\pi_{0}(x))^{u}\,\subseteq\,x^{u}\setminus x^{l}  \\
(F_{0}(x))^{u}\,\subseteq\, x^{u} \\
St(x)^{l}\,\subseteq\, x^{l} \,\& \, F_{0}(x)\,=\,\emptyset\, \longrightarrow\,St(x)\,=\, x^{l+} \\
{Sym(x)}^{u}\,\subseteq \,x^{u} \,\&\, (Sym(x))^{l}\,\subseteq\, x^{l}. 
\end{gather*}
\end{proposition}
\begin{proof}
Proof is fairly direct. 
\qed
\end{proof}

\chapter{More on Representation of Rough Objects}

We have already shown in the previous chapter that the representation of rough objects by definite objects is not possible in a \textsf{PRAX}. So it is important to look at possibilities based on other types of derived approximations. We do this and solve the problem right up to representation theorems for the derived operators in this chapter.

\begin{definition}
If $x\in \wp (S)$, then 
\begin{mitemize}
\item {Let $\Pi _{\heartsuit}^{o}(x)\,=\, \{y\,;\, x\,\subseteq\, y \,\&\, x^{l}\,=\,y^{l} \&\, y^{u}\,\subseteq\, x^{uu}\}$.}
\item {Form the set of maximal elements $\Pi _{\heartsuit}(x)$ of $\Pi _{\heartsuit}^{o}(x)$ with respect to the inclusion order.}
\item {Select a unique element $\chi (\Pi _{\heartsuit}(x)) $ through a fixed choice function $\chi$.}
\item {Form $(\chi (\Pi _{\heartsuit}(x)))^{u}$.}
\item {$x^{\heartsuit \chi}\,=\, (\chi (\Pi _{\heartsuit}(x)))^{u}$ will be said to be the \emph{almost upper approximation} of $x$ relative $\chi$. }
\item {$x^{\heartsuit \chi}$ will be abbreviated by $x^{\heartsuit}$ when we work with fixed $\chi$.}
\end{mitemize}
The choice function will be said to be \emph{regular} if and only if $(\forall x, y)\,(x\,\subseteq y \,\&\, x^{l}\,=\, y^{l}\,\longrightarrow \,\chi (\Pi _{\heartsuit}(x))\,=\, \chi (\Pi _{\heartsuit}(y))  )$. We will assume regularity unless specified otherwise in what follows.
\end{definition}

\begin{definition}
If $x\in \wp (S)$, then 
\begin{mitemize}
\item {Let $\Pi _{\diamondsuit}^{o}(x)\,=\, \{y\,;\, x\,\subseteq\, y \,\&\, x^{l}\,=\,y^{l}\}$.}
\item {Form the set of maximal elements $\Pi _{\diamondsuit}(x)$ of $\Pi _{\diamondsuit}^{o}(x)$ with respect to the inclusion order.}
\item {Select a unique element $\chi (\Pi _{\diamondsuit}(x)) $ through a fixed choice function $\chi$.}
\item {$x^{\diamondsuit \chi}\,=\, \chi (\Pi _{\diamondsuit}(x))$ will be said to be the \emph{lower limiter} of $x$ relative $\chi$. }
\item {$x^{\diamondsuit \chi}$ will be abbreviated by $x^{\diamondsuit}$ when we work with fixed $\chi$.}
\end{mitemize}
\end{definition}

\begin{definition}
If $x\in \wp (S)$, then 
\begin{mitemize}
\item {Let $\Pi _{\flat}^{o}(x)\,=\, \{y\,;\, y\,\subseteq\, x \,\&\, x^{u}\,=\,y^{u}\}$.}
\item {Form the set of maximal elements $\Pi _{\flat}(x)$ of $\Pi _{\flat}^{o}(x)$ with respect to the inclusion order.}
\item {Select a unique element $\xi (\Pi _{\flat}(x)) $ through a fixed choice function $\xi$.}
\item {$x^{\flat \xi}\,=\, \xi (\Pi _{\flat}(x))$ will be said to be the \emph{ upper limiter} of $x$ relative $\chi$. }
\item {$x^{\flat \xi}$ will be abbreviated by $x^{\flat}$ when we work with fixed $\xi$.}
\end{mitemize}
\end{definition}

\begin{proposition}\label{non1}
In the context of the above definition, the almost upper approximation satisfies all of the following:
\begin{gather}
(\forall {x})\, x\,\subseteq \, x^{\heartsuit} \tag{Inclusion}\\ 
(\forall {x})\, x^{\heartsuit}\,\subseteq\, x^{\heartsuit \heartsuit} \tag{Non-Idempotence}\\
(\forall x\, y)\,(x\,\subseteq \, y\,\subseteq\, x^{\heartsuit}\,\longrightarrow\, x^{\heartsuit}\,\subseteq \, y^{\heartsuit})  \tag{Cautious Monotony}\\
(\forall {x})\, x^{u}\,\subseteq \, x^{\heartsuit} \tag{Supra Pseudo Classicality}\\
S ^{\heartsuit}\,=\, S  \tag{Top.}
\end{gather}
\end{proposition}

\begin{proof}
\begin{mitemize}
\item {Inclusion: Follows from the construction. If we have one element granules or successor neighborhoods included in $x$, then these must be in the lower approximation. If a granule $y$ is not included in $x$, but intersects it in $f$, then it is possible to include $f$ in each of $\Pi_{\heartsuit}(x)$. So inclusion follows.}
\item {Non-Idempotence: The reverse inclusion does not happen as $x^{u}\,\subseteq\, x^{uu}$. }
\item {Cautious monotony: It is clear that monotony can fail in general because of the choice aspect, but if we have $x\,\subseteq \, y\,\subseteq\, x^{\heartsuit}$, then $x^{l}\subseteq y^{l}$ and $y^{\heartsuit}$ has to be equal to $x^{\heartsuit}$ or include more granules because of regularity of the choice function.}
\item {Supra Pseudo Classicality: We use the adjective \emph{pseudo} because $u$ is not a classical consequence operator. In the construction of $x^{\heartsuit}$, we select from super-sets of $x^{l}$ that can generate maximal upper approximations and take the upper approximation of the selected. So that includes $x^{u}$ in general.}
\end{mitemize}
\qed
\end{proof}

We have used the names of conditions in relation to the standard terminology used in non-monotonic reasoning. The upper approximation operator $u$ is similar to classical consequence operator, but lacks idempotence. So the fourth property has been termed as \emph{supra} \emph{pseudo} \emph{classicality} as opposed to \emph{supra classicality}. This means we are in a more general domain of reasoning relative the domains of \cite{DM94}.

\begin{theorem}
In the context of \ref{non1}, we have the following additional properties:
\begin{gather*}
(\forall x)\, x^{\heartsuit} \subseteq x^{u \heartsuit}    \tag{Sub Left Absorption}\\ 
(\forall x)\, x^{\heartsuit} \subseteq x^{\heartsuit u}    \tag{Sub Right Absorption}\\
\boxdot(\forall x, y)\,(x^{u}\,=\, y^{u}\,\nrightarrow\,x^{\heartsuit}\,=\, y^{\heartsuit}  )    \tag{No Left Logical Equivalence}\\
\boxdot (\forall x, y)\,(x^{\heartsuit}\,=\, y^{\heartsuit}\,\nrightarrow\,x^{l}\,=\, y^{l})    \tag{No Jump Equivalence}\\
\boxdot(\forall x, y, z)\,(x\subseteq y^{\heartsuit}\,\&\, z\subseteq x^{u}\,\nrightarrow\,z\subseteq y^{\heartsuit}  )    \tag{No Weakening}\\
\boxdot(\forall x, y)\,(x\,\subseteq\, y\,\subseteq\, x^{u}\,\nrightarrow\,x^{\heartsuit}\,=\, y^{\heartsuit}  )    \tag{No subclassical cumulativity}\\
(\forall x, y)\,x^{\heartsuit}\,\cap\, y^{\heartsuit}\,\subseteq\,(x^{u}\,\cap\, y^{u})^{\heartsuit}     \tag{Distributivity}\\
(\forall x, y, z)\,(x\cup z)^{\heartsuit}\,\cap\, (y\cup z)^{\heartsuit}\,\subseteq\,(z \cup (x^{u}\,\cap\, y^{u}))^{\heartsuit}     \tag{Weak Distributivity}\\
(\forall x, y, z)\,(x\,\cup\, y)^{\heartsuit}\,\cap\, (x\,\cup\, z)^{\heartsuit}\,\subseteq (x\,\cup\, (y\,\oplus\,z))^{\heartsuit}
\tag{Disjunction in Antecedent}\\
(\forall x, y)\,(x\cup y)^{\heartsuit}\, \cap\, (x\cup y^{c})^{\heartsuit}\,\subseteq\, x^{\heartsuit} \tag{Proof by Cases}\\
\mathrm{If}\; y \subseteq (x\cup z)^{\heartsuit},\; \mathrm{then}\; x\,\implies y\subseteq z^{\heartsuit}    \tag{Conditionalization.} 
\end{gather*}
\end{theorem}

\begin{proof}
\begin{description}
\item [Sub Left Absorption]{For any $x$, $x^{\heartsuit}$ is the upper approximation of a maximal subset $y$ containing $x$ such that $x^{l}\,=\, y^{l}$ and $x^{u\heartsuit}$ is the upper approximation of a maximal subset $z$ containing $x^u$ such that $x^{ul}\,=\,x^{u}\,=\, z^{l}$. Since, $x^{l}\,\subseteq\, x^{ul}$ and $x\,\subseteq\, x^{u}$, so $x^{\heartsuit} \subseteq x^{u \heartsuit}$ follows. }
\item [Sub Right Absorption]{Follows from the properties of $u$.}
\item [No Left Logical Equivalence]{Two subsets $x,\, y$ can have unequal lower approximations and equal upper approximations and so the implication does not hold in general. $\boxdot$ should be treated as an abbreviation for \emph{in general}.}
\item [No Jump Equivalence]{The reason is similar to that of the previous negative result.}
\item [No weakening]{In general if $x\subseteq y^{\heartsuit}\,\&\, z\subseteq x^{u}$, then it is possible that $x^{u} \subseteq y^{\heartsuit}$ or $y^{\heartsuit}\,\subseteq \, x^{u}$. So we cannot be sure about $z\subseteq y^{\heartsuit}$.}
\item [No Subclassical Cumulativity]{If $x\,\subseteq\, y\,\subseteq\, x^{u}$, then $x^{l}\,\subseteq \, y^{l} $ in general and so elements of $\Pi_{\heartsuit}(x)$ may be included in $\Pi_{\heartsuit}(y)$, the two may be unequal and we may not be able to use a uniform choice function on them. So we need not have $x^{\heartsuit}\,=\, y^{\heartsuit}$. }
\item [Distributivity]{If $ z \in x^{\heartsuit}\,\cap\, y^{\heartsuit}$, then $z \in (\chi(\Pi_{\heartsuit}(x)))^{u} $ and $z \in (\chi(\Pi_{\heartsuit}(y)))^{u} $.  So if $z\in x^{l}$ and $z\in y^{l}$, then $z\in (x^{u}\,\cap\, y^{u})^{\heartsuit} $. 
Since in general, $(a\cap b)^{u}\,\subseteq\, a^u \cap b^{u}$ and $(a^{u}\cap b^u )^l = (a^{u}\cap b^u )$, we have the required inclusion. }
\end{description}
\begin{gather*}
(\forall x, y, z)\,(x\cup z)^{\heartsuit}\,\cap\, (y\cup z)^{\heartsuit}\,\subseteq\,(z \cup (x^{u}\,\cap\, y^{u}))^{\heartsuit}     \tag{Weak Distributivity}\\
(\forall x, y, z)\,(x\,\cup\, y)^{\heartsuit}\,\cap\, (x\,\cup\, z)^{\heartsuit}\,\subseteq (x\,\cup\, (y\,\oplus\,z))^{\heartsuit}
\tag{Disjunction in Antecedent}\\
(\forall x, y)\,(x\cup y)^{\heartsuit}\, \cap\, (x\cup y^{c})^{\heartsuit}\,\subseteq\, x^{\heartsuit} \tag{Proof by Cases}\\
\mathrm{If}\; y \subseteq (x\cup z)^{\heartsuit},\; \mathrm{then}\; x\,\implies y\subseteq z^{\heartsuit}    \tag{Conditionalization.} 
\end{gather*} 
\qed
\end{proof}

\begin{proposition}
\begin{gather*}
(\forall x, y)(x^{\diamondsuit}\,=\, y^{\diamondsuit}\,\longrightarrow\, x^{l}\,=\, y^{l;})\\ 
(\forall x, y)(x^{\flat}\,=\, y^{\flat}\,\longrightarrow\, x^{u}\,=\, y^{u}.) 
\end{gather*} 
\end{proposition}

\begin{flushleft}
\textbf{Discussion}: 
\end{flushleft}
In non monotonic reasoning, if $C$ is any consequence operator $:\wp(S)\,\longmapsto\, \wp(S)$, then the following named properties of crucial importance in semantics (in whatever sense, \cite{DM94,DM2003}):
\begin{gather}
A\subseteq B\subseteq C(A) \longrightarrow C(B) \subseteq C(A) \tag{Cut}\\
A\subseteq B\subseteq C(A) \longrightarrow C(B) = C(A) \tag{Cumulativity}\\
x\subseteq y\subseteq x^{u} \longrightarrow \, x^{\heartsuit}\,=\, y^{\heartsuit} \tag{subclassical subcumulativity}
\end{gather}

\begin{proposition}
In the context of the above definition, the lower limiter satisfies all of the following:
\begin{gather}
(\forall {x})\, x\,\subseteq \, x^{\diamondsuit} \tag{Inclusion}\\ 
(\forall {x})\, x^{\diamondsuit\diamondsuit}\,=\, x^{ \diamondsuit} \tag{Idempotence}\\
(\forall x\, y)\,(x\,\subseteq \, y\,\subseteq\, x^{\diamondsuit}\,\longrightarrow\, x^{\diamondsuit}\,= \, y^{\diamondsuit})  \tag{Cumulativity}\\
(\forall {x})\, x^{u}\,\subseteq \, x^{\diamondsuit} \tag{Upper Inclusion}\\
S ^{\diamondsuit}\,=\, S  \tag{Top}\\
(\forall x\, y)\,(x\,\subseteq \, y\, \subseteq\, x^{\diamondsuit}\,\longrightarrow\, x^{\diamondsuit}\,= \, y^{\diamondsuit})  \tag{Cumulativity}
\end{gather}
\end{proposition}

\emph{The above proposition means that the upper limiter corresponds to ways of reasoning in a stable way in the sense that the aggregation of conclusions does not affect inferential power or cut-like amplification}. 

We prove a limited concrete representation theorem for operators like $\heartsuit$ in special cases and $\diamondsuit$. The representation theorem is valid for similar operators in non-monotonic reasoning. The representation theorem permits us to identify cover based formulations of \textsf{PRAX}.

\begin{definition}
A collection of sets $\mathcal{S}$ will be said to be a \emph{closure system} of a type as per the following conditions:
\begin{gather*}
\tag{Closure System} (\forall \mathcal{H}\subseteq \mathcal{S})\, \cap \mathcal{H}\in \mathcal{S}. \\ 
\tag{U-Closure System} (\forall \mathcal{H}\subseteq \mathcal{S})\, (\cap \mathcal{H})^{u}\in \mathcal{S}. \\ 
\tag{L-Closure System} (\forall \mathcal{H}\subseteq \mathcal{S})\, (\cap \mathcal{H})^{l}\in \mathcal{S}. \\  
\tag{LU-Closure System} (\forall \mathcal{H}\subseteq \mathcal{S})\, (\cap \mathcal{H})^{l},(\cap \mathcal{H})^{u}\, \in \mathcal{S}. \\
\tag{Bounded} (\exists 0 , \top \in \mathcal{S})(\forall X\in \mathcal{S})\, 0 \subseteq X \subseteq \top\, .
\end{gather*}
\end{definition}

\begin{proposition}
In a \textsf{PRAX} $S$, the set $\mathcal{U}(S)\,=\,\{x^{u}; \, x\in \wp(S) \}  $ is not a bounded U-closure system.
\end{proposition}

\begin{proposition}
\[(\forall {x}) \, x^{\heartsuit u }\,\subseteq \, x^{u \heartsuit } .\] 
\end{proposition}
\begin{proof}
Because $x^{l}\,\subseteq x^{u}$, an evaluation of possible granules involved in the construction of $x^{\heartsuit u }$ and $x^{u \heartsuit }$ proves the result.  
\qed
\end{proof}

\begin{theorem}
In a \textsf{PRAX} $S$, the set $\mathcal{\heartsuit}(S)\,=\,\{x^{\heartsuit}; \, x\in \wp(S) \}  $ is a bounded LU-closure system if the choice operation is regular.
\end{theorem}

\begin{proof}
\begin{mitemize}
\item {$x^{\heartsuit}$ is the upper approximation of a specific $y$ containing $x$ that is maximal subject to $x^{l}\,=\, y^{l}$.}
\item {$x^{\heartsuit u}$ is the upper approximation of the upper approximation of a specific $y$ containing $x$ that is maximal subject to $x^{l} = y^{l}$ and its upper approximation.}
\item {Clearly, \[(\chi(\Pi_{\heartsuit} (x))\,\cap\, \chi (\Pi_{\heartsuit} (y)))^{u}\,\subseteq\, (\chi(\Pi_{\heartsuit} (x)))^{u}\,\cap\, (\chi(\Pi_{\heartsuit} (y)))^{u}. \] }
\item {The expression on the right of the inclusion is obviously a union of granules in the \textsf{PRAX}. }
\item {From a constructive bottom-up perspective, let $p_{1},\, p_{2},\, \ldots \, p_{s}$ be a collection of subsets of $x\setminus x^{l}$ such that }
\end{mitemize}
\begin{gather*}
\cup p_{i}\subseteq x\setminus x^{l}\\
(\exists z)\, p_{i}^{u}\,=\, [z] \\
\cup_{i\neq j}(p_i \cap p_j ) \; \mathrm{is \; minimal \; on\; all \; such\; collections.}
\end{gather*}

\begin{mitemize}
\item {Now we add subsets $k(p_{i})$ of $x^{uu}\setminus x^{c}$ to $x$ to form the required maximal subset.}
\item {For the lower approximation part, we simple use the preservation of $l$ by $cap$.}
\end{mitemize}
\qed
\end{proof}

\begin{proposition}
For each $x\in \wp(S)$ let $x^{\curlyvee}\,=\, (x^{\heartsuit})^{u} $, then we have the following properties:
\begin{gather*}
(\forall x ) \,  x \,\subseteq x^{\curlyvee} .\\
(\forall x )\, x^{\curlyvee \curlyvee}\,=\, x^{\curlyvee}.
\end{gather*} 
\end{proposition}
\begin{proof}
$x^{\heartsuit u \heartsuit u}\,=\, x^{\heartsuit u}$. Because if we could add a part of a class that retains the equality of lower approximations, then that should be adjoinable in the construction of $x^{\heartsuit}$ as well.  
\qed
\end{proof}

The following limited representation theorem can be useful for connections with covers.

\begin{definition}
Let $X$ be a set and $C\,:\, \wp (X)\,\longmapsto \,\wp(X)$ a map satisfying all the following conditions:
\begin{gather*}
(\forall A\in \wp(S))\, A\subseteq C(A)   \tag{Inclusion}\\
(\forall A\in \wp(S))\, C (C(A))\,=\,C(A)   \tag{Idempotence}\\
(\forall A, B \in \wp(S))\,(A\subseteq B \subseteq C(A)\,\longrightarrow\, C(A)\,\subseteq\, C(B))   \tag{Cautious Monotony,}
\end{gather*}
then $C$ will said to be a \emph{cautious closure operator} (\textsf{CCO}) on $X$.
\end{definition}

\begin{definition}\label{relevant}
Let $H\,=\, \left\langle\underline{H},\, \preceq\right\rangle$, be a partially ordered set over a set $\underline{H}$. A subset $\mathcal{K}$ of the set of order ideals $\mathcal{F}(H)$ of $H$ will be said to be \emph{relevant} for a subset $B\,\subseteq\, H$ (in symbols $\rho (\mathcal{K},\, H)$) if and only if the following hold:
\begin{gather*}
(\exists G\in \mathcal{K})(\forall P \in \mathcal{K})\, {P\,\subseteq\, G }. \\
(\forall P \in \mathcal{K})\, P\,\subseteq B .\\
\mathrm{For \; any}\; \mathcal{L}\subseteq \mathcal{F}(H), \; \mathrm{if}\; \mathcal{K}\subseteq \mathcal{L}, \; \mathrm{then}\\ (\exists \top \in \mathcal{L})(\forall Y \in \mathcal{L})\, Y\subseteq \top \neq H \,\&\, \cap \mathcal{L}\,=\, \cap \mathcal{K}.
\end{gather*}
\end{definition}

\begin{definition}
In the context of \textsf{Def.}\ref{relevant}, a map $\jmath : \wp(L) \,\longmapsto \, \wp (L)$ defined as below will be said to be \emph{safe}

\[\jmath(Z)\, =\,\left\{
\begin{array}{ll}
\cap \mathcal{K}, \:\mathrm{if \: all\; relevant \; collections\; for\; Z \; have\; same \; intersection.} \\
\cap\{ \alpha : Z\subseteq \alpha \in \mathcal{F}(H) \},\; \mathrm{else}.
\end{array}
\right. \]
\end{definition}

\begin{proposition}
A safe map $\jmath$ is a cautious closure operator. 
\end{proposition}

\begin{proof}
The verification of idempotence and inclusion is direct.
\begin{mitemize}
 \item {For $A, B\in \wp(L)$, if we have $A\subseteq B \subseteq \jmath (A) $,}
 \item {then either $A\subseteq B \subseteq \jmath(B) \subseteq \jmath (A) $ or $A\subseteq B \subseteq \jmath (A)\subseteq \jmath (B) $ must be true. }
 \item {If the former inclusions hold, then it is necessary that $\jmath (A) \,=\, \jmath (B)$.}
 \item {If $\jmath (B)$ is defined as the the intersection of order ideals and $\jmath (A)$ as that of relevant subcollections, then it is necessary that $\jmath (A)\subseteq \jmath (B)$. So cautious monotony holds. It can also be checked that monotonicity fails in this kind of situation.}
\end{mitemize}
\qed
\end{proof}

\begin{theorem}
On every Boolean ordered unary algebra of the form \[\mathcal{H}\,=\,\left\langle \wp (H), \subseteq , C \right \rangle, \] there exists a partial order $\leq$ on $K$ such that $\left\langle \wp (K), \subseteq, \jmath \right \rangle $
is isomorphic to $\mathcal{H}$.
\end{theorem}

\chapter{Atoms in the POSET of Rough Objects}
 
\begin{definition}\label{less}
 For any two elements $x, y\in \wp(S)|\approx$, let 
\[x\,\leq\, y \,\ifsf \, (\forall a\in x)(\forall b\in y) a^{l}\subseteq
b^{l}\,\&\,a^{u}\subseteq b^{u}. \]
$\wp(S)|\approx$ will be denoted by $H$ in what follows.
\end{definition}

\begin{proposition}
The relation $\leq$ defined on $H$ is a bounded and directed partial order. The least
element will be denoted by $0$ ($0\,=\, \{\emptyset\}$) and the greatest by $1$
($1\,=\,\{S\}$).  
\end{proposition}

\begin{definition}
For any $a, b\in H$, let $UB(a, b)\,=\, \{x\, :\,a\leq x \,\&\,b\leq x \}$ and 
$LB(a, b)\,=\, \{x\, :\,x\leq a \,\&\,x\leq b \}$. By a \emph{s-ideal} (strong ideal) of
$H$, we will mean a subset $K$ that satisfies all of 
\begin{gather*}
{(\forall x\in H)(\forall a\in K)(x\leq a\,\longrightarrow\, x\in K),}\\
{(\forall a, b\in K)\, UB(a, b)\cap K\neq \emptyset .} 
\end{gather*}

An \emph{atom} of $H$ is any element that covers $0$. The set of atoms of $H$ will be
denoted by $At(H)$. 
\end{definition}

\begin{theorem}
Atoms of $H$ will be of one of the following types:
\begin{description}
\item [Type-0]{Elements of the form $(\emptyset, [x])$, that intersect no other set of
roughly equivalent sets. }
\item [Type-1]{Bruinvals of the form $(\emptyset, \alpha)$, that do not contain full sets
of roughly equivalent sets.}
\item [Type-2]{Bruinvals of the form $(\alpha, \beta)$, that do not contain full sets of
roughly equivalent sets and are such that $(\forall x) x^{l}\,=\, \emptyset$.}
\end{description}
\end{theorem}

\begin{proof}
It is obvious that a bruinval of the form $(\alpha, \beta)$ can be an atom only if
$\alpha$ is the $\emptyset$. If not, then each element $x$ of the bruinval $(\emptyset,
\alpha)$ will satisfy $x^{l}=\emptyset \,\subset \, x^{u}$, thereby contradicting the
assumption that $(\alpha, \beta)$ is an atom.

If $[x]$ intersects no other successor neighborhood, then \[(\forall y\in (\emptyset,
[x])) y^{l}\,=\, \emptyset\,\&\,x^{u}\,=\, [x]\] and it will be a minimal set of
roughly equal elements containing $0$. 

The other part can be verified based on the representation of possible sets of roughly  
equivalent elements.
\qed 
\end{proof}

\begin{theorem}
The partially ordered set $H$ is atomic.
\end{theorem}

\begin{proof}
We need to prove that any element $x$ greater than $0$ is either an atom or  
there exists an atom $a$ such that $a \leq x$, that is 
\[(\forall x)(\exists a\in At(H))(0 < x\,\longrightarrow\, a\leq x ) .\]

Suppose the bruinval $(\alpha,\beta)$ represents a non-atom, then it is necessary that 
\[(\forall x\in \alpha)\, x^{l} \neq \emptyset \,\&\, x^{u}\subseteq S .  \]

Suppose the neighborhoods included in $x^{u}$ are $\{[y]\,:\, y\in B\subseteq S\}$.
If all combinations of bruinvals of the form $(\emptyset,\gamma)$ formed from these
neighborhoods are not atoms, then it is necessary that the upper approximation of every
singleton subset of a set in $\gamma$ properly contains another non-trivial
upper approximation. This is impossible. 

So $H$ is atomic. 
\qed
\end{proof}

\chapter{Algebraic Semantics-1}

An algebraic semantics is a complete description of reasoning about rough objects
involved in the context of \textsf{PRAX} or \textsf{PRAS} or any particular
instances thereof. In the present author's view the objects of interest should be roughly equal elements in some sense and the semantics should avoid objects of other kinds (from other semantic domains) thereby contaminating the semantics. But in any perspective, semantics relative any semantic domain is of interest. When it comes to the question of defining sensible operations over rough objects, given the ontological constraints, there is scope for much variation.

If $A,\, B\,\in\, \wp (S)$ and $A\approx B$ then $A^{u}\approx B^{u}$ and $A^{l}\approx B^{l}$, but $\neg (A\approx A^{u}) $ in general. We have already seen that $\leq$ is a partial order relation on $\wp(S)|\approx$. In this chapter we will be working mostly on $\wp (S)|\approx$ and will use lower case Greek alphabets for elements in it. 

\begin{theorem}
The following operations can be defined on $\wp (S)|\approx$ ($A, \, B\in \wp (S)$ and $[A],\,[B]$ are corresponding classes):
\begin{gather}
L[A]\,\stackrel{\Delta}{=}\,[A^{l}] \\
[A]\odot [B]\,\stackrel{\Delta}{=}\, [\bigcup _{X\in [A],\, Y\in [B] } (X\cap
Y)]\\
[A]\oplus [B]\,\stackrel{\Delta}{=}\, [\bigcup _{X\in [A],\, Y\in [B] } (X\cup
Y)]\\
U[A]\,\stackrel{\Delta}{=}\,[A^{u}]\\
[A]\cdot [B]\,\stackrel{\Delta}{=}\,\lambda (LB([A], [B]))\\
[A]\circledast [B]\,\stackrel{\Delta}{=}\,\lambda (UB([A], [B]))\\
 [A] + [B]\,\stackrel{\Delta}{=}\,\{X \,:\, X^{l}\,=\, (A^{l}\cap B^{l})^{l}\,\&\,X^{u}\,=\, A^{u}\cup B^{u} \} \\
 [A] \times [B]\,\stackrel{\Delta}{=}\,\{X \,:\,X^{l}\,=\, A^{l}\cup B^{l}\,\&\, X^{u}\,=\, A^{l}\cup B^{l}\,\cup (A^{u}\cap B^{u}) \\
[A]\otimes [B] \,\stackrel{\Delta}{=}\,\{X\,:\, X^{l}\,=\,A^{l}\cup B^{l}\,\&\, X^{u}\,=\,A^{u}\cup B^{u}\}. 
\end{gather}
\end{theorem}

\begin{proof}
If $A\approx B$ then $A^{u}\approx B^{u}$ and $A^{l}\approx B^{l}$, but $\neg (A\approx A^{u}) $ in general.
\begin{enumerate}
\item {If $B\in [A]$, then $B^{l}\,=\, A^{l}$, $B^{u}\,=\, A^{u}$ and $L[A] = L[B] \, =\, [A^{l}]$.}
\item {$[A]\odot [B]\,\stackrel{\Delta}{=}\, [\bigcup _{X\in [A],\, Y\in [B] } (X\cap
Y)]$ is obviously well defined as sets of the form $[A]$ are elements of partitions}
\item {Similar to the above.}
\item {If $B\in [A]$, then $B^{u}\,=\, A^{u}$ and so $[B^{u}]\,=\, [A^{u}]$.}
\item {$[A]\cdot [B]\,\stackrel{\Delta}{=}\,\lambda (LB([A], [B]))$. }
\item {$[A]\circledast [B]\,\stackrel{\Delta}{=}\,\lambda (UB([A], [B]))$.}
\item {$ [A] + [B]\,\stackrel{\Delta}{=}\,\{X \,:\, X^{l}\,=\, A^{l}\cap B^{l}\,\&\,X^{u}\,=\, A^{u}\cup B^{u} \} $. As the definitions is in terms of $A^{l},\, B^{l}, A^{u},\, B^{u} $, so  there is no issue. }
\item {Similar to above.}
\item {Similar to above.}
\end{enumerate}
\end{proof}

$+$, $\times$ and $\otimes$ will be referred to as \emph{pragmatic} aggregation, commonality and commonality operations as they are less ontologically committed to the classical domain and more dependent on the main rough domain of interest. $+$ and the other pragmatic operations cannot be compared by the $\leq$ relation and so do not confirm to intuitive understanding of the concepts of aggregation and commonality. 

The following theorems summarize the essential properties of the defined operations:

\begin{theorem}
\begin{gather*}
\tag{L1}{LL(\alpha)\,=\, L(\alpha). }\\
\tag{L2}{(\alpha \, \leq \, \beta\, \longrightarrow\,L(\alpha)\,\leq \, L(\beta) ).}\\
\tag{L3}{(L(\alpha)\,=\, [\alpha]\,\longrightarrow\,\alpha\,=\, \{\alpha^{l}\} ).}\\
\tag{U1}{(U(\alpha)\,\cap \, UU(\alpha) \neq \emptyset \,\longrightarrow\, U(\alpha)\,=\, UU(\alpha)).}\\
\tag{U2}{(UU(\alpha)\,=\, \emptyset \, \nrightarrow \, U(\alpha)\,=\, \emptyset).}\\
\tag{U3}{(\alpha\,\leq\, \beta\,\longrightarrow\, U(\alpha)\,\leq \, U(\beta).}\\
\tag{U4}{(U(\alpha)\,=\, \alpha\,\longrightarrow\,\alpha\,=\, \alpha^{l}\,=\, \alpha^{u} ).}\\
\tag{U5}{UL(\alpha)\,\leq \, U(\alpha).}\\
\tag{U6}{LU(\alpha)\,=\, U(\alpha).}
\end{gather*}
\end{theorem}
\begin{proof}
Let $\alpha\in \wp (S)|\approx$, then we can associate a pair of lower and upper approximations denoted by $\alpha_l$ and $\alpha_u$ respectively. By $\alpha^u$ and $\alpha^l$ we mean the global operations respectively on the set $\alpha$ (seen as an element of $\wp(S)$). These take singleton values and so we do not really need the approximations $\alpha_l$ and $\alpha_{u}$ and shall use the former. 

\textsf{Proof of L1:}
\begin{gather*}
\alpha\in\wp(S)|\approx,\;\mathrm{so}\; \alpha = \{X\,;\, \alpha_l \,=\, X^l \, \&\, \alpha_u \,=\,X^{u},\&\, X\in\wp (S) \}.\\
\alpha^{l}\,=\, \{X^{l}; \, X\,\in\, \alpha \}\,=\, \{\alpha_l \} \\
\mathrm{So}\; [\alpha^{l}]\,=\, \{Y\,; \, Y^{l}\,=\, \alpha^{l}\,\&\, Y^{u}\,=\, \alpha^{lu} \}.\\
(L(\alpha))^{l}\,=\, \{Y^l \,; \,Y^{l}\,=\, \alpha^{l}\,\&\, Y^{u}\,=\, \alpha^{lu} \}\,=\,\{\alpha^{l}\}.\\
\mathrm{This\; yields}\; LL(\alpha)\,=\,L(\alpha). \tag{L1}
\end{gather*}
\textsf{Proof of U1:}
\begin{gather*}
\alpha^u = \{X^u \,;\,\alpha^{l} \,=\, X^{l} \, \&\, \alpha^u \,=\,X^{u} \}\,=\, \{\alpha^{u}\}.\\
U(\alpha)\,=\,[\alpha^{u}]\,=\, \{Y\,; \, Y^{l}\,=\, \alpha^{u}\,\&\, Y^{u}\,=\, \alpha^{uu} \}. \\
\mathrm{So}\; U(\alpha)^{u}\,=\, \{\alpha^{uu}\}.\\
UU(\alpha)\,=\, [U(\alpha)^{u}]\,=\, [\alpha^{uu}]\,=\,\{Y\,;\,Y^{l}\,=\,\alpha^{uu}\,\&\, Y^{u}\,=\,\alpha^{uuu}\}.\\
\mathrm{Since}\; \alpha\,\subseteq\,\alpha^{u}\,\subseteq\, \alpha^{uu}\subseteq\, \alpha^{uuu},\\
\mathrm{therefore}\; (U(\alpha)\,\cap \, UU(\alpha) \neq \emptyset \,\longrightarrow\, U(\alpha)\,=\, UU(\alpha). \tag{U1}
\end{gather*}
The other parts can be proved from the above considerations.
\qed
\end{proof}

\begin{theorem}
In the context of the above theorem, the following hold:
\begin{gather*}
\tag{CO1}{\alpha\odot \beta\,=\, \beta\odot \alpha)}\\
\tag{CO2}{\alpha\,\leq\, \alpha\odot \alpha }\\
\tag{CO3}{\alpha \leq \alpha\odot \top  }\\
\tag{CO4}{\alpha\odot \alpha \,=\,\alpha\odot (\alpha\odot \alpha) \,=\, \alpha \odot \top}\\
\tag{AO1}{\alpha\oplus \beta\,=\, \beta\oplus \alpha)}\\
\tag{AO2}{\alpha \leq \alpha\oplus \beta}\\
\tag{AO3}{\alpha \leq \alpha\oplus \bot  }\\
\tag{AO4}{(\alpha\oplus \alpha) \oplus \alpha \,=\, \alpha\oplus \alpha}\\
\tag{AC}{\mathrm{In\;general}, \alpha\oplus (\alpha\odot \beta)\neq \alpha .}
\end{gather*}
\end{theorem}
\begin{proof}
\begin{description}
\item [CO1]{The definition of $\odot$ does not depend on the order in which we take the arguments as set theoretic intersection and union are commutative. To be precise $\bigcup _{X\in [A],\, Y\in [B] } (X\cap Y)\,=\, \bigcup _{X\in [A],\, Y\in [B] } (Y\cap
X)$.}
\item [CO2]{$\bigcup _{X\in [A],\, Y\in [A] } (X\cap Y)\,=\, \bigcup_{X\in [A]} X $. But because $X^{l}\cup Y^{l}\subseteq (X\cup Y)^{l}$ in general, so we do not have equality.}
\item [CO3]{Follows from the last inequality.}
\item [CO4]{In $[\alpha\odot (\alpha\odot \alpha)]$, we cannot introduce any new elements that are not in $[\alpha\odot \alpha]$ as the inequality in [CO2] is due to the lower approximation and we have already included all possible subsets }
\item [AO1]{The definition of $\oplus$ does not depend on the order in which we take the arguments as set theoretic union is  commutative.}
\item [AO2]{Even when $\beta=\alpha$, we can have the inequality for reasons mentioned earlier.}
\end{description}
Proof of [AO3, AO4, AC] are analogous or direct.
\qed 
\end{proof}

The above result means that $\odot$ is an imperfect commonality relation. It is a proper commonality among a certain subset of elements of $H$.

\begin{theorem}
In the context of the above theorem, the following properties of $+,\, \times, \otimes$ are provable:
\begin{gather*}
\tag{+I}{\alpha + \alpha \,=\, \alpha},\\
\tag{+C}{\alpha + \beta \,=\, \beta + \alpha},\\
\tag{cI}{\alpha \times \alpha \,=\, \alpha},\\
\tag{cC}{\alpha \times \beta \,=\, \beta \times \alpha},\\
\tag{+Is}{\alpha \leq \beta \longrightarrow \alpha + \gamma \leq \beta + \gamma },\\
\tag{cIs}{\alpha \leq \beta \longrightarrow \alpha \times \gamma \leq \beta \times \gamma },\\
\tag{+In}{\alpha \leq \beta \longrightarrow \alpha \leq \alpha \times \beta \leq \beta},\\
\tag{R1}{\alpha + \beta \leq \alpha \oplus \beta},\\
\tag{Mix1}{\alpha \times \beta \leq (\alpha \times \beta) \oplus \alpha.}
\end{gather*}
\end{theorem}
\begin{proof}
Most of the proof is in Sec.\ref{appsem}, so we do not repeat them here.
\qed 
\end{proof}

\begin{definition}
By a \emph{Concrete Pre-PRAX Algebraic System} (\textsf{CPPRAXA}), we will mean a system of the form 
\[\mathfrak{H}\,=\,\left\langle H,\, \leq, L, U, \oplus , \odot , + , \times, \otimes, \bot , \top  \right\rangle,\] with all of the operations being as defined in this chapter.
\end{definition}

Apparently we need to involve the algebraic properties  of the rough objects of $l_{o},\, u_{o}$ to arrive at a representation theorem.
Further we can improve the operations defined to some extent by the related operations of the following chapter. Results concerning this will appear separately. Definable filters in general have reasonable properties.

\begin{definition}
Let $K$ be an arbitrary subset of a \textsf{CPPRAXA} $\mathfrak{H}$. Consider the following statements:
\begin{gather*}
\tag{F1} {(\forall x\in K)(\forall y\in \mathfrak{H})(x \leq y \Rightarrow
y\in K) .}\\
\tag{F2} {(\forall x, y\in K)\, x\oplus y, Lx \in K .}\\
\tag{F3} {(\forall a, b \in \mathfrak{H})(1\neq a\oplus b\in K\,\Rightarrow\, a\in K
\;\mathrm{or}\; b\in K) .}\\
\tag{F4} {(\forall a, b \in \mathfrak{H})(1\neq UB(a, b)\in K\,\Rightarrow\, a\in K
\;\mathrm{or}\; b\in K) .}\\
\tag{F5} {(\forall a, b\in K)\,LB(a, b)\cap K\neq \emptyset .}
\end{gather*}

\begin{mitemize}
\item {If $K$ satisfies \textbf{F1} then it will be said to be an \emph{order filter}. The
set of such filters on $\mathfrak{H}$ will be denoted by $\mathfrak{O}_{F}(\mathfrak{H})$.}
\item {If $K$ satisfies \textbf{F1, F2} then it will be said to be a \emph{filter}. The
set of such filters on $\mathfrak{H}$ will be denoted by $\mathcal{F}(\mathfrak{H})$.}
\item {If $K$ satisfies \textbf{F1, F2, F3} then it will be said to be a \emph{prime filter}.
The set of such filters on $\mathfrak{H}$ will be denoted by $\mathcal{F}_{P}(\mathfrak{H})$.}
\item {If $K$ satisfies \textbf{F1, F4} then it will be said to be a \emph{prime
order filter}. The set of such filters on $\mathfrak{H}$ will be denoted by $\mathfrak{O}_{PF}(\mathfrak{H})$.}
\item {If $K$ satisfies \textbf{F1, F5} then it will be said to be an \emph{strong order
filter}. The set of such filters on $\mathfrak{H}$ will be denoted by $\mathfrak{O}_{SF}(\mathfrak{H})$.}
\end{mitemize}
Dual concepts of ideals of different kinds can be defined.  
\end{definition}

\begin{proposition}
Filters of different kinds have the following properties:
\begin{mitemize}
\item {Every set of filters of a kind is ordered by inclusion.}
\item {Every filter of a kind is contained in a maximal filter of the same kind.}
\item {$\mathfrak{O}_{SF}(\mathfrak{H})$ is an algebraic lattice, with its compact elements being
the finitely generated strong order filters in it.}
\end{mitemize}
\end{proposition}

\begin{definition}
For $F, P \in \mathcal{F} (\mathfrak{H})$, we can define the following operations:
\[F\wedge P\,\stackrel{\Delta}{=} F\cap P\]
\[F\vee P\,\stackrel{\Delta}{=} \left\langle F\cup P\right\rangle,  \]
where $\left\langle F\cup P\right\rangle$ denotes the smallest filter containing $F\cup
P$. 
\end{definition}

\begin{theorem}
$\left\langle \mathcal{F}(\mathfrak{H}),\, \vee , \, \wedge , \bot, \top \right\rangle$ is an
atomistic bounded lattice.
\end{theorem}

\chapter{Algebraic Semantics-2}\label{appsem}

We have seen that ordered pairs of the form $(A^{l},\, A^{u})$ do correspond to rough objects by definition. If we choose to ignore the representation and finer aspects of possible reasonable aggregation and commonality operations, then we still obtain an interesting order structure based fragment of semantic processes that is very useful in the approximation based semantics that we consider in  subsequent chapters. 

\begin{definition}
In a \textsf{PRAX} $S$, let \[\mathcal{R}(S)\,=\, \{(A^{l},\,A^{u})\, ; \, A\in \wp(S)\}.\]
Then we can define all of the following operations on $\mathcal{R}(S)$:
\begin{gather*}
(A^{l},\,A^{u})\,\vee\, (B^{l},\,B^{u})\,\stackrel{\Delta}{=}\, (A^{l}\cup B^{l},\,A^{u}\cup B^{u} ). \tag{Aggregation}\\  
\mathrm{If}\; (A^{l}\cap B^{l},\,(A^{u}\cap B^{u}))\in \mathcal{R}(S)\;\mathrm{then}\;\\ 
(A^{l},\,A^{u})\,\wedge\, (B^{l},\,B^{u})\,\stackrel{\Delta}{=}\, (A^{l}\cap B^{l},\,(A^{u}\cap B^{u})). \tag{Commonality}\\ 
\mathrm{If}\; (A^{uc},\,A^{lc})\in \mathcal{R}(S)\;\mathrm{then}\;\\
\sim (A^{l},\,A^{u})\,\stackrel{\Delta}{=}\, (A^{uc},\,A^{lc}). \tag{Weak Complementation}\\
\bot \,\stackrel{\Delta}{=}\, (\emptyset ,\, \emptyset ). \;\; \top \,\stackrel{\Delta}{=}\, (S,\, S). \tag{Bottom, Top}\\
(A^{l},\,A^{u})\,\barwedge\, (B^{l},\,B^{u})\,\stackrel{\Delta}{=}\, ((A^{l}\cap B^{l})^{l},\,(A^{u}\cap B^{u})^{l}).\tag{Proper Commonality}
\end{gather*}
\end{definition}

\begin{definition}
In the context of the above definition, a partial algebra of the form $\mathfrak{R}(S)\,=\,\left\langle \mathcal{R}(S),\, \vee ,\, \wedge ,\, c ,\,\bot ,\,\top \right\rangle $ will be termed a \emph{proto-vague algebra} and $\mathfrak{R}_{f}(S)\,=\,\left\langle \mathcal{R}(S),\, \vee ,\, \wedge ,\,\barwedge.\, c ,\,\bot ,\,\top \right\rangle $ will be termed a \emph{full proto-vague algebra}.

More generally, if $L,\, U$ are arbitrary rough lower and upper approximation operators over the \textsf{PRAX}, and if we replace each occurrence of $l$ by $L$ and $u$ by $U$ in the above definition then we will term the resulting algebra of the above form a $LU$-\emph{proto-vague partial algebra}. Thus we will speak of $l_{o} u_{o}$-proto-vague algebras and such.  
\end{definition}

\begin{theorem}
A full proto-vague partial algebra $\mathfrak{R}_{f}(S)$ satisfies all of the following:
\begin{enumerate}
\item {$\vee, \barwedge $ are total operations.}
\item {$\vee$ is a semi-lattice operation satisfying idempotency, commutativity and associativity.}
\item {$\wedge$ is a weak semi-lattice operation satisfying idempotency, weak strong commutativity and weak associativity. With $\vee$ it forms a weak distributive lattice.}
\item {$\sim$ is a weak strong idempotent partial operation; $\sim\sim\sim \alpha \,\stackrel{\omega^{*}}{=}\,\sim \alpha.$ }
\item {$\sim (\alpha \vee \beta)\,\stackrel{\omega}{=}\,\sim\alpha \wedge \sim\beta $ (Weak De Morgan condition) holds.}
\item {$\barwedge$ is an idempotent, commutative and associative operation that forms a lattice with $\vee$. }
\item {$\alpha \barwedge \bot \,=\, \alpha \wedge \bot \,=\, \bot$. $\alpha \vee \bot \,=\, \alpha$; $\alpha \barwedge \top \,=\, \alpha \wedge \top \,=\, \alpha$. $\alpha \vee \top \,=\, \top$.}
\item {$\sim (\alpha \wedge \beta)\,=\, (\sim \alpha \vee \sim \beta) \longrightarrow \sim (\alpha \barwedge \beta) \,=\,(\sim \alpha \vee \sim \beta).$ }
\item {$\alpha \vee (\beta \barwedge \gamma)\subseteq (\alpha \vee \beta) \barwedge (\alpha \vee \gamma) $, but distributivity fails.}
\end{enumerate}
\end{theorem}
\begin{proof}
Let $\alpha\,=\, (X^{l}, X^{u}),\; \beta\,=\, (Y^{l}, Y^{u})$ and $\gamma\,=\, (Z^{l}, Z^{u})$ for some $X,\, Y, \, Z\in \wp (S)$, then
\begin{enumerate}
\item {$\alpha \vee \beta\,=\, (X^{l}\cup Y^{l},X^{u}\cup Y^{u})$ belongs to $\mathfrak{R}(S)$ because the components are unions of successor neighborhoods and $X^{l}\cup Y^{l}\,\subseteq \,X^{u}\cup Y^{u}$. The proof for $\wedge$ is similar.  }
\item {$\alpha \vee (\beta \vee \gamma)\,=\, (X^{l},X^{u}) \vee ((Y^{l}, Y^{u})\vee (Z^{l}, Z^{u}))\,=\, (X^{l},X^{u}) \vee (Y^{l}\cup Z^{l}, Y^{u}\cup Z^{u})\,=\,(X^{l}\cup Y^{l}\cup Z^{l}, X^{u}\cup Y^{u}\cup Z^{u})\,=\, (\alpha\vee\beta)  \vee \gamma.$}
\item {We prove weak absorptivity and weak distributivity alone. 

$(X^{l}\cap (X^{l}\cup Y^{l}))\,=\, X^{l}$ and $(X^{u}\cap (X^{u}\cup Y^{u}))\,=\, X^{l}$ hold in all situations. If $(X^{l}\cup (X^{l}\cap Y^{l}))$ is defined then it is equal to $X^{l}$ and if $(X^{u}\cup (X^{u}\cup Y^{u}))$ is defined, then it is equal to $X^{u}$. So \[\alpha \vee (\alpha \wedge \beta)\,\stackrel{\omega}{=}\, \alpha\,{=}\,\alpha \wedge (\alpha \vee \beta).\]
 
For distributivity ($\alpha \vee (\beta \wedge \gamma)\,\stackrel{\omega}{=}\,(\alpha\vee\beta)\wedge(\alpha\vee \gamma ) $ and $\alpha \wedge (\beta \vee \gamma)\,\stackrel{\omega}{=}\,(\alpha\wedge\beta)\vee(\alpha\wedge \gamma ) $) again it is a matter of definability working in coherence with set-theoretic distributivity.  
}
\item {If $\sim \alpha$ is defined then $\sim \alpha\,=\,(X^{uc},X^{lc})$  and \[\sim\sim \alpha \,=\, \sim (X^{uc},X^{lc})\,=\, (X^{lcc},X^{ucc})\,=\,(X^{l},X^{u}),\] by definition. If $\sim \sim \alpha$ is defined, then $\sim \alpha$ is necessarily defined. So
\[\sim\sim\sim \alpha \,\stackrel{\omega^{*}}{=}\,\sim \alpha. \]}
\item {If $\sim (\alpha \vee \beta)$ and $\sim\alpha \wedge \sim\beta$ are defined then $\sim (\alpha \vee \beta)\,=\, \sim ((X^{l}\cup Y^{l}),(X^{u}\cup Y^{u}))\,=\, ((X^{uc}\cap Y^{uc}),(X^{lc}\cap Y^{lc}))\,\stackrel{\omega^{*}}{=}\,(X^{uc},X^{lc})\wedge (Y^{uc}, Y^{lc})\, =\,\sim\alpha \wedge \sim\beta $. So $\sim (\alpha \vee \beta)\,\stackrel{\omega^{*}}{=}\,\sim\alpha \wedge \sim\beta $.}
\item {$\alpha\barwedge \beta \,=\, \beta \barwedge \alpha \,\&\, \alpha \barwedge \alpha \,=\, \alpha$ are obvious.

$\alpha\barwedge(\beta \barwedge \gamma )\,=\,((X^{l}\cap (Y^{l}\cap Z^{l})^{l})^{l},\,(X^{u}\cap (Y^{u}\cap Z^{u})^{u})^{u} ) $
The components are basically the unions of common granules among the three. No granule in the final evaluation is eliminated by choice of order of operations. So $\alpha\barwedge(\beta \barwedge \gamma )\,=\, (\alpha\barwedge \beta) \barwedge \gamma $.

$\alpha \barwedge (\alpha \vee \beta) \,=\, ((X^{l}\cap (X^{l}\cup Y^{l}))^{l},\,(X^{u}\cap (X^{u}\cup Y^{u}))^{l} )\,=\, \alpha $. 

Further, $\alpha \vee (\alpha \barwedge \beta) \,=\, ((X^{l}\cup (X^{l}\cap Y^{l})^{l}),\,(X^{u}\cup (X^{u}\cap Y^{u})^{l}) )\,=\,\alpha$. So $\vee , \barwedge$ are lattice operations.}
\item {
\begin{mitemize}
\item {Since $\bot = (\emptyset , \emptyset )$, $\alpha \barwedge \bot \,=\, \alpha \wedge \bot \,=\, \bot$ and $\alpha \vee \bot \,=\, \alpha$ follow directly.}        
\item {Since $\top = (S, S)$, $\alpha \barwedge \top \,=\, \alpha \wedge \top \,=\, \alpha$ and $\alpha \vee \top \,=\, \top$ follow directly.}
\end{mitemize}}
\item {Follows from the previous proofs.}
\item {
\begin{mitemize}
\item {$\alpha \vee (\beta \barwedge \gamma) \,=\, ((X^{l}\cup (Y^{l}\cap Z^{l})^{l}),\,(X^{u}\cup (Y^{u}\cap Z^{u})^{l}) )$. If $a\in S$ and $[a]\subseteq X^{l}\cup (Y^{l}\cap Z^{l})^{l}$, and $[a]\subseteq (Y^{l}\cap Z^{l})^{l}$, then $[a]\subseteq Y^{l}$ and $[a]\subseteq Z^{l}$. So $[a]\subseteq X^{l}\cup Y^{l}$ and $[a]\subseteq X^{l}\cup Z^{l}$.}
\item {If $[a]\subseteq X^{l}\cup (Y^{l}\cap Z^{l})^{l}$ and if $[a]\,=\, P \cup Q$, with $P\subseteq X^{l}$, $Q\subseteq (Y^{l}\cap Z^{l})^{l}$ then $[a]\subseteq X^{l}\cup Y^{l}$ and $[a]\subseteq X^{l}\cup Z^{l}$. This proves $\alpha \vee (\beta \barwedge \gamma)\subseteq (\alpha \vee \beta) \barwedge (\alpha \vee \gamma) $.}
\item {If $[a]\subseteq ((X^{l}\cup Y^{l})\cap (X^{l}\cup Y^{l}))^{l}$ then $[a]\subseteq X^{l}\cup Y^{l}$ and $[a]\subseteq X^{l}\cup Z^{l}$. This means $[a]\,=\, P \cup Q$, with $P\subseteq X^{l}$, $Q\subseteq Y^{l}$ and $Q\subseteq Z^{l}$ and $Q$ is contained in union of some other granules. So $Q\subseteq Y^{l}\cap Z^{l}$, but we cannot ensure $Q\subseteq (Y^{l}\cap Z^{l})^{l}$ (required counterexamples are easy to construct). It follows that $((X^{l}\cup Y^{l})\cap (X^{l}\cup Y^{l}))^{l} \nsubseteq X^{l}\cup (Y^{l}\cap Z^{l})^{l}$.}
\end{mitemize}}
\end{enumerate}
\end{proof}

The following theorem provides us a condition for ensuring that $\sim \alpha$ is defined.

\begin{theorem}
If $X^{uu}\,=\,X^{u} $, then $\sim(X^{l},X^{u})\,=\, (X^{uc},X^{lc})$ but the converse is not necessarily true. 
\end{theorem}

\begin{proof}
\begin{mitemize}
\item {$\sim(X^{l},X^{u})$ is defined if and only if $X^{uc}$ is a union of granules.} 
\item {If $X^{uu}\,=\,X^{u} $ then $X^{uc}$ is a union of granules generated by \emph{some} of the elements in $X^{uc}$ , but the converse need not hold.}
\item {So we have the result.}
\end{mitemize}
\end{proof}

Let $W$ be any quasi-order relation that approximates $R$, and let the granules $[x]_w , \,[x]_{wi} $ and $l_w ,\, u_w$ be lower and upper approximations defined by analogy with the definitions of $l,\, u$. If $R\subset W$, then $(\forall x\in S)\,[x]\subseteq [x]_w$ and we have the following scenario ($A,\, B \in \wp(S)$. We write $A\parallel B$ for $A\nsubseteq B \,\&\,B\nsubseteq A$):

\begin{mitemize}
\item {If $A\subset B$ and $A^{u} = B^{u}$, then it is possible that $A^{u_w}\subset B^{u_w}$.}
\item {If $A\subset B$ and $A^{l} = B^{l}$, then it is possible that $A^{l_w}\subset B^{l_w}$.}
\item {If $A\subset B$ and $A^{u_w}= B^{u_w}$, then it is possible that $A^{u} \subset B^{u}$.}
\item {If $A\subset B$ and $A^{l_w}= B^{l_w}$, then it is possible that $A^{l} \subset B^{l}$.}
\item {If $A\parallel B$ and $A^{l}=B^{l}$, then it is possible that $A^{l_w} \parallel B^{l_w}$.}
\item {If $A\parallel B$ and $A^{l_w}=B^{l_w}$, then it is possible that $A^{l} \parallel B^{l}$.}
\item {If $A\parallel B$ and $A^{u}=B^{u}$, then it is possible that $A^{u_w} \parallel B^{u_w}$.}
\item {If $A\parallel B$ and $A^{u_w}=B^{u_w}$, then it is possible that $A^{u} \parallel B^{u}$.}
\item {If $A\subset B$, $A^{l} = B^{l}$ and $A^{u} = B^{u}$ , then it is possible that $A^{u_w}\subset B^{u_w}\,\&\,A^{l_w}\subset B^{l_w}$.}
\end{mitemize}

The above properties mean that meaningful correspondences between vague partial algebras and Nelson algebras may be quite complex.
Focusing on granular evolution alone, we can define 
\begin{gather*}
(\forall x\in S) \, \varphi_{o}([x])\,=\, \bigcup_{z\in [x]} [z]_{w}.\\
(\forall A \in \wp (S))\, \varphi (A^{l})\,=\, \bigcup_{[x]\,\subseteq A^{l}} \varphi_{o}([x]).\\
(\forall A \in \wp (S))\, \varphi (A^{u})\,=\, \bigcup_{[x]\,\subseteq A^{u}} \varphi_{o}([x]).
\end{gather*}
$\varphi(A^{l}\cup B^{l})\,=\, \bigcup_{[x]\subseteq A^{l}\cup B^{l}}$.

If $[x]\subseteq A^{l}\cup B^{l}$ 
 
 $\varphi$ can be naturally extended by components to a map $\tau$ as per \[\tau (A^{l},A^{u})\,=\, (\varphi(A^{l}),\, \varphi(A^{u})).\] 
 
\begin{proposition}
If $R\subseteq R_{w}$ and $R_w$ is transitive, then
\begin{mitemize}
\item {If $z\in [x]$ and $x\in [z]$, then $\varphi ([z]) = \varphi([x])$.}
\item {If $z\in [x]$, then $\varphi ([z])\subseteq \varphi([x])$.}
\item {\[(\forall A \in \wp (S))\, \varphi(A^{l}) = \bigcup_{[x]\subseteq A^{l}}\varphi ([x]) =  \bigcup_{[x]\subseteq A^{l}} [x]_{w} \]}
\end{mitemize}
\end{proposition}
\begin{proof}
\begin{mitemize}
\item {$z\in [x]$ yields $Rzx$. So if $Raz$, then $Rax$ and it is clear that $\varphi ([z]) \subseteq \varphi([x])$. $Rbx \& Rzx \& Rxz$ implies $R_{w}bz$ . } 
\item {This is the first part of the above.}
\item {Follows from the above.}
\end{mitemize}
\end{proof}

\begin{definition}
We will use the following abbreviations for handling different types of subsets of $S$:
\begin{gather*}
\Gamma_{u}(S)\,=\, \{A^{u}; A\in \wp (S) \}. \tag{Uppers}\\
\Gamma_{uw}(S)\,=\, \{A^{u_{w}}; A\in \wp (S) \}. \tag{w-Uppers}\\
\Gamma (S)\,=\, \{B; \, (\exists A\in \wp (S))\, B=A^{l}\;\mathrm{or}\;B=A^{u}\}. \tag{lower definites}
\end{gather*}
Note that $\delta_{l}(S)$ is the same as $\Gamma (S)$ and similarly for $\delta_{lw}(S)$. 
\end{definition}

$\tau$ has the following properties:

\begin{proposition}
If $R\subseteq R_{w}$ and $R_w$ is transitive, then
\begin{gather*}
\tau(\bot)\,=\, \bot_{w}.\\
\tau(\top )\,=\, \top_{w}.\\
(\forall \alpha, \beta \in \mathfrak{R}(S))\, \tau(\alpha\vee \beta)\,{=}\, \tau(\alpha)\vee \tau(\beta) .\\
(\forall \alpha, \beta \in \mathfrak{R}(S))\, \tau(\alpha\wedge \beta)\,\stackrel{\omega}{=}\, \tau(\alpha)\wedge \tau(\beta) .
\end{gather*}
\end{proposition}

\begin{proof}

\end{proof}

\begin{definition}
For each $\alpha\in \mathfrak{R}_{w}(S)$, the set of ordered pairs $\tau^{\dashv}(\alpha)$ will be termed as a \emph{co-rough object} of $S$, where \[\tau^{\dashv}(\alpha)\,=\, \{\beta \,; \, \beta\in \mathfrak{R}(S)\, \& \,\tau(\beta)\,=\, \alpha\}.\]
The collection of all co-rough objects will be denoted by $\mathfrak{CR}(S)$.
\end{definition}

This permits us to define a variety of closely related semantics of \textsf{PRAX} when $R\subseteq R_{w}$ and $R_w$ is transitive. These include:

\begin{itemize}
\item {The map $\tau : \mathfrak{R}_{f}(S)\, \longmapsto\, \mathfrak{R}_{w}(S)$. $\mathfrak{R}_{w}(S)$ being a Nelson algebra over an algebraic lattice.}
\item {$\mathfrak{R}_{f}(S)\,\cup\,\mathfrak{CR}(S)$ along with induced operations yields another semantics of \textsf{PRAX}.}
\item {$\mathfrak{R}(S) \,\cup\,\mathfrak{R}_{w}(S)$ enriched with algebraic and dependency operations described in \ref{dep}.}
\end{itemize}

\chapter{Approximate Relations}

If $R$ is a binary relation on a set $X$, then we let $R^{o} \, \stackrel{\partial}{=}\, R\,\cup \, \Delta_{X}$. The weak transitive closure of $R$ will be denoted by $R^{\#}$. If $R^{(i)}$ is the $i$-times composition $\stackrel{\underbrace{R\circ R \ldots \circ R}}{\textrm{i-times}}$, then $R^{\#}\,=\, \bigcup R^{(i)}$. $R$ is \emph{acyclic} if and only if $(\forall x) \, \neg R^{\#} xx$. 
The relation $R^{\cdot}$ is defined by $R^{\cdot} ab$ if and only if $R a b \,\&\, \neg (R^{\#}a b \,\& \, R^{\#} b a)$.

\begin{definition}
If $R$ is a relation on a set $S$, then the relations $R^{\leftthreetimes},\, R^{cyc}$ and $R^{h} $ will be defined via
\begin{gather}
{R^{\lf} a b \mathrm{\; if \;and\; only\; if\;} [b]_{R^{o}}\subset [a]_{R^{o}} \,\&\, [a]_{i R^{o}}\subset [b]_{i R^{o}}}\\
{R^{cyc} a b  \mathrm{\; if \;and\; only\; if\;} R^{\#} a b\, \& \,R^{\#} b a}\\
{R^{h} a b  \mathrm{\; if \;and\; only\; if\;} R^{\lf} a b\, \& \,R^{\cdot} a b .}
\end{gather}
In case of PRAX, $R^{o}\, =\,R $, so the definition of $R^{\lf}$ would involve neighborhoods of the form $[a]$ and $[a]_{i}$ alone.  $R^{\lf}\subset R$ and $R^{\lf}$ is a partial order.
\end{definition}

\begin{examp}
In our example \ref{agre}, $R^{\#}ab$ happens when $a$ is an ally of an ally of $b$. $R^{\lf}ab$ happens \textsf{iff} every ally of $b$ is an ally of $a$ and if $a$ is ally of $c$, then $b$ is an ally of $c$ - this can happen, for example, when $b$ is a Marxist feminist and $a$ is a socialist feminist. $R^{cyc}ab$ happens when \emph{$a$ is an ally of an ally of $b$ and $b$ is an ally of an ally of} $a$.  $R^{\cdot}ab$ happens whenever $a$ is an ally of $b$, but $b$ is not an ally of anybody who is an ally of $a$.  
\end{examp}

\begin{theorem}
$R^{h}\,=\, \emptyset .$ 
\end{theorem}

\begin{proof}
\begin{gather*}
{R^{h} a b \Leftrightarrow\, R^{\lf} a b\, \&\, R^{\cdot} a b}\\
{\Leftrightarrow \,\tau(R) a b\, \&\, (R\setminus \tau(R)) a b}\\
{\mathrm{But}\; \neg (\exists a ) (R\setminus \tau (R)) a a}.
\end{gather*}
{So $R^{h}\,=\, \emptyset$.}
\qed
\end{proof}

\begin{proposition}
All of the following hold in a PRAX $S$:
\begin{align}
{R^{\cdot}ab \,\leftrightarrow\,(R\setminus \tau(R)) ab }\\
{(\forall a, b) \neg (R^{\cdot} ab \,\&\, R^{\cdot} ba)}\\
{(\forall a, b, c)(R^{\cdot} a b \,\&\, R^{\cdot} b c \,\longrightarrow\, \neg R^{\cdot} a c).}
\end{align}
\end{proposition}
\begin{proof}
\begin{mitemize}
\item {$ R^{\cdot}ab \,\leftrightarrow\, R ab \, \&\, \neg (R^{\#}ab \, R^{\#} ba) $.}
\item {But $\neg (R^{\#}ab \, R^{\#} ba)$ is possible only when both $Rab$ and $Rba$ hold.}
\item {So $ R^{\cdot}ab \,\leftrightarrow\, R ab \, \&\, \neg (\tau (R) ab )  \,\leftrightarrow\, (R\setminus \tau (R)) ab$.}
\end{mitemize}
\qed
\end{proof}

\begin{theorem}
\begin{align}
{R^{\# \cdot}\,=\, R^{\#}\setminus \tau (R)}\\
{R^{\cdot \#}\,=\, (R\setminus \tau (R))^{\#}}\\
{(R\setminus \tau (R))^{\#} \,\subseteq \, R^{\#}\setminus \tau (R) .}
\end{align}
\end{theorem}

\begin{proof}
\begin{enumerate}
\item {\begin{align*}
 R^{\# \cdot} a b \leftrightarrow & R^{\#} a b\,\&\, \neg (R^{\#\#} a b \,\& \, R^{\#\#} b a) \\
      \leftrightarrow & R^{\#} a b \,\&\, \neg (R^{\#} a b \,\& \, R^{\#} b a) \\
      \leftrightarrow & R^{\#} a b \,\&\, \neg \tau (R) a b \\
      \leftrightarrow & (R^{\#}\setminus \tau (R)) a b.
       \end{align*}}
\item {\begin{align*}
 R^{\cdot \#} a b \leftrightarrow & (R^{\cdot})^{\#} a b \\
      \leftrightarrow & (R\setminus \tau (R))^{\#} a b .
       \end{align*}}
\item {Can be checked by a contradiction or a direct argument.}
\end{enumerate}
\qed
\end{proof}

We now look at possible properties that approximations of prototransitive relations may/should possess. If $<$ is a strict partial order on $S$ and $R$ is a relation, then consider the conditions :
\begin{gather}
\tag{PO1} (\forall a, b)(a < b\, \longrightarrow \, R^{\#} a b).\\
\tag{PO2} (\forall a, b)(a < b \, \longrightarrow \, \neg R^{\#} b a) .\\
\tag{PO3} (\forall a, b)(R^{\lf} a b \,\&\, R^{\cdot} a b \,\longrightarrow\, a < b.\\
\tag{PO4} \mathrm{If\;} a\equiv_{R} b, \mathrm{\;then\;} a\equiv_{<} b.\\
\tag{PO5} (\forall a, b)(a < b\, \longrightarrow \, R a b).
\end{gather}

As per \cite{RJ2011}, $<$ is said to be a \emph{partial order approximation} \textsf{POA} (resp. \emph{weak partial order approximation} \textsf{WPOA}) of $R$ if and only if \bbf{PO1, PO2, PO3, PO4} (resp. \bbf{PO1, PO3, PO4}) hold. A \textsf{POA} $<$ is \emph{inner approximation} \textsf{IPOA} of $R$ if and only if \bbf{PO5} holds. \bbf{PO4} has a role beyond that of approximation and depends on both successor and predecessor neighborhoods. $R^{h},\, R^{\cdot \lf}$ are \textsf{IPOA}, while $R^{\cdot \#},\, R^{\# \cdot}$ are \textsf{POA}s. 

By a \emph{lean quasi order approximation} $<$ of $R$, we will mean a quasi order satisfying \bbf{PO1} and \bbf{PO2}. 
The corresponding sets of such approximations of $R$ will be denoted by $POA(R),\, WPOA(R),\, IPOA(R), IWPOA(R) $ and $LQO(R)$

\begin{theorem}
For any $A,\, B\in LQO(R)$, we can define the operations $\& ,\vee , \top$: 
\begin{gather*}
(\forall x, y) (A \& B) x y\; \mathrm{if\; and\; only\; if}\;(\forall x, y)  A x y \,\&\, B x y. \\
(A \vee B)   \,=\,(A \cup B)^{\#},\\
\top = R^{\#}. 
\end{gather*}
\end{theorem}

\begin{proof}
\begin{mitemize}
\item {If $A a b$  then $R^{+} a b $ and if $B a b$ then $R^{+} a b $.}
\item {But if $(A\& B) a b$, then both  $A a b$ and $B a b$.}
\item {So $R^{+} a b$.}
\end{mitemize}

Similarly it can be shown that $A\vee B \in LQO(R)$. It is always defined and contained within $R^{\#}$ as it is the transitive completion of $A\cup B$. $\top\,=\, R^{\#}$ as transitive closure is a closure operator.   
\qed
\end{proof}

\begin{theorem}
In a \textsf{PRAX}, $R^{\cdot \#} \& R^{\# \cdot} x y  \,\leftrightarrow \, (R\setminus \tau(R))^{\#} x y .$
\end{theorem}

\section{Granules of Derived Relations}

The behavior of approximations and rough objects corresponding to derived relations is investigated in this section.

\begin{definition}
The relation $R^{\#\cdot}$ will be termed the \emph{trans ortho-completion} of $R$. The following granules will be associated with each $x\in S$ : 
\begin{gather}
[x]_{ot}\,=\, \{y\, ;\, R^{\#\cdot} y x \, \} \\
[x]_{ot}^{i}\,=\, \{y\, ;\, R^{\#\cdot} x y \, \} \\
[x]_{ot}^{o}\,=\, \{y\, ;\, R^{\#\cdot}y x \,\&\, R^{\#\cdot} x y \}.                                                                                                                                                                                      \end{gather}
Let the corresponding approximations be $l_{ot},\, u_{ot}$ and so on.
\end{definition}

\begin{theorem}
In a PRAX $S$,  $(\forall x\in S)\, [x]_{ot}^{o}\,=\, \{x \}.$
\end{theorem}

\begin{proof}
$R^{\#\cdot} x y\,\&\,R^{\#\cdot} y x $ means that the pair $(x,\, y)$ is in the transitive completion of $R$ and not in $\tau(R)$. 
So $y\in [x]_{ot}^{o}$ if and only if 
\[(\exists a,\, b)\, R x a \,\&\, R a y \, \&\, (\neg R a x \vee \neg R y a )\,\& \,(R y b \,\&\, R b x)\, \&\, (\neg R b y \vee \neg R x b).\] 

If we assume that $x\,\neq \, y$, then each of the possibilities leads to a contradiction as is shown below. In the context of the above statement:

\begin{flushleft}
\textbf{Case-1} 
\end{flushleft}

\begin{mitemize}
\item {$R x a \,\&\,R a y \,\&\,\neg R a x \,\&\, R y a \,\&\, R y b \,\&\, R b x \,\&\, \neg R b y \,\&\, R x b $.}
\item {This yields $R^{\#} x a \,\&\, R^{\#} bb \,\&\, R^{\#} b a \,\&\, R^{\#} a b .$}
\item {So, $R^{\#} x b \,\&\, R^{\#} y a \,\&\, R^{\#} a x$ and we have contradicted our original assumption.}
\end{mitemize}

\begin{flushleft}
\textbf{Case-2} 
\end{flushleft}

\begin{mitemize}
\item {$R x a \,\&\,R a y \,\&\, R a x \,\&\,\neg R y a \,\&\, R y b \,\&\, R b x \,\&\,  R b y \,\&\,\neg R x b $.}
\item {This yields the contradiction $R^{\#} a b$. }
\end{mitemize}

\begin{flushleft}
\textbf{Case-3} 
\end{flushleft}

\begin{mitemize}
\item {$R x a \,\&\,R a y \,\&\,\neg R a x \,\&\, R y a \,\&\, R y b \,\&\, R b x \,\&\,  R b y \,\&\,\neg R x b.$}
\item {This yields $R^{\#} b a \,\&\, R^{\#} a b \,\&\,R^{\#} a a \,\&\, R^{\#} b b  $ and $R^{\#} y y \&  R^{\#} x y \&  R^{\#} y x \& R y a \& R^{\#} x a$. }
\item {But such a $R^{\#}$ is not possible.}
\end{mitemize}

Somewhat similarly the other cases can be seen to lead to contradictions.
\qed
\end{proof}

By the \emph{symmetric center} of a relation $R$, we will mean the set $K_{R}\,=\, \bigcup e_{i} (\tau(R) \setminus \Delta_{S})$ - basically the union of elements in either component of $\tau (R)$ minus the diagonal relation on $S$.

\begin{proposition} $(\forall x)\, [x]\triangle [x]_{ot}\,\neq\, \emptyset$ as
\begin{align*}
x\,\notin\, K_{R}\, \longrightarrow\, [x]\subset [x]_{ot}\\
x\,\in\, K_{R}\, \longrightarrow\, [x]\nsubseteq [x]_{ot} \,\&\, \{x\}\subset [x]\cap [x]_{ot}.  
\end{align*}
\end{proposition}
 
\begin{proof}
\begin{align*}
z\in [x]_{ot} \, & \leftrightarrow\, R^{\#\cdot} z x\\
               \, & \leftrightarrow\, R^{\#} z x \,\&\, \neg \tau(R) z x\\
   \, & \leftrightarrow\, (R z x \,\&\, \neg R x z ) \,\mathrm{or}\, (\neg R z x \,\& \, \neg R x z  \,\&\, (R^{\#}\setminus R) z x). 
\end{align*}
\qed
\end{proof}

$K_{R}$ can be used to partially categorize subsets of $S$ based on intersection.

\begin{proposition}
$(R\setminus \tau(R))^{\#}\cup \tau (R)$ is not necessarily a quasi order.
\end{proposition}
\begin{proof}
$(x, y)\in (R\setminus \tau(R))^{\#}\cup \tau (R) $ and $(x, y)\notin \tau(R)$ and $x\in K_{R}\,\&\, y \notin K_{R}$ and $\exists z\in K_{R}\, \&\, z\neq x \,\&\, Rzx$ do not disallow $Rz y$. So $(R\setminus \tau(R))^{\#}\cup \tau (R)$ is not necessarily a quasi-order. We leave the missing part to the reader.
\qed
\end{proof}

\begin{proposition}
 $((R\setminus \tau(R))^{\#}\cup \tau (R))^{\#}\, =\, R^{\#}.$
\end{proposition}

\begin{proof}
Clearly $R\subseteq ((R\setminus \tau(R))^{\#}\cup \tau (R))^{\#}$ and it can be directly checked that if $a\in ((R\setminus \tau(R))^{\#}\cup \tau (R))^{\#} \setminus R$ then $a\in R^{\#}\setminus R$ and conversely. 
\end{proof}

\chapter{Transitive Completion and Approximate Semantics}

The interaction of the rough approximations in a \textsf{PRAX} and the rough approximations in the transitive completion can be expected to follow some order. \emph{The definite or rough objects most closely related to the difference of lower approximations and those related to the difference of upper approximations can be expected to be related in a nice way}. We show that this \emph{nice way} is not really a \emph{rough way}. But the results proved remain relevant for the formulation of semantics that involves that of the transitive completion as in \cite{JPR,SJ}. A rough theoretical alternative is possible by simply starting from sets of the form $A^{*}\,=\,(A^{l}\setminus A^{l_{\#}})\cup (A^{u_{\#}} \setminus A^{u})$ and taking their lower ($l_{\#}$) and upper ($u_{\#}$) approximations - the resulting structure would be a partial algebra derived from a Nelson algebra over an algebraic lattice (\cite{AM3690}).     

\begin{proposition}
For an arbitrary proto-transitive reflexive relation $R$ on a set $S$, (we use $\#$ subscripts for neighborhoods, approximation operators and rough equalities of the weak transitive completion) all of the following hold: 
\begin{align*}
(\forall x \in S)\, [x]_{R}\,\subseteq [x]_{R^{\#}} \tag{Nbd} \\
(\forall A \subseteq S )\, A^{l}\,\subseteq A^{l_{\#}} \,\&\,A^{u}\,\subseteq A^{u_{\#}} \tag{App}\\
(\forall A \subseteq S )(\forall B \in [A]_{\approx})(\forall C \in [A]_{\approx_{\#}})\, B^{l}\,\subseteq C^{l_{\#}} \,\&\,B^{u}\,\subseteq C^{u_{\#}}  \tag{REq}                         
\end{align*}
The reverse inclusions are false in general in the second assertion in a specific way. 
Note that the last condition induces a more general partial order $\preceq$ over $\wp(\wp(S))$ via $A \preceq B$ if and only if $(\forall C\in A)(\forall E\in B)\,C^{l}\subseteq E^{l_{\#}}\, \&\,C^{u}\subseteq E^{u_{\#}}$.  
\end{proposition}
\begin{proof}
The first of these is direct. For simplicity, we will denote the successor neighborhoods of $x$ by $[x]$ and $[x]_{\#}$ respectively. We look at the possibility tracking in the first part of the second assertion.
\begin{mitemize}
\item {If $z\in A^{l_{\#}}$ then $z\in A^{l}$ as $[x]_{\#}\subseteq A$ implies $[x]\subseteq A$.}
\item {If $z\in A^{l}$ then $(\exists x) \, z\in [x]\subseteq A^{l}$.}
\item {For this $x$, $z\in [x]_{\#}$, but it is possible that $[x]_{\#}\subseteq A$ or $[x]_{\#}\nsubseteq A$.}
\item {If $[x]_{\#}\nsubseteq A$, and $(\exists b\notin A )\,R_{\#} a x \,\&\, R a b \, \&\, R b x $ then we have a contradiction as $Rb x$ means $b\in [x]$.}
\item {If $[x]_{\#}\nsubseteq A$, and $(\exists b\in A )\,R_{\#} a x \,\&\, R a b \, \&\, R b x $ all we need is a $c\notin A \,\& R c b$ that is compatible with $R_{\#} c x$ and $A^{l}\nsubseteq A^{l_{\#}}$.}
\end{mitemize}
\qed
\end{proof}

\begin{definition}
By the \emph{l-scedastic approximation} $\hat{l}$  and the \emph{u-scedastic approximation} $\hat{u}$ of a subset $A\,\subseteq S$ we will mean the following approximations:
\[A^{\hat{l}}\,=\, (A^{l}\setminus A^{l_{\#}})^{l}, \;\; A^{\hat{u}}\,=\, ( A^{u_{\#}} \setminus A^{u})^{u_{\#}}. \]
The above cross difference approximation is the best possible from closeness to properties of rough approximations. 
\end{definition}

\begin{theorem}
For an arbitrary subset $A\subseteq S$ of a \textsf{PRAX} $S$,the following statements and diagram of inclusion ($\rightarrow)$ hold: 
\begin{mitemize}
\item {$A^{l_{\#}l} = A^{l_{\#}} = A^{ll_{\#}} = A^{l_{\#} l_{\#}}$}
\item {If $A^{u}\subset A^{u_{\#}}$ then $A^{u u_{\#}}\subseteq A^{u_{\#} u_{\#}}$. } 
\end{mitemize}
\begin{center}
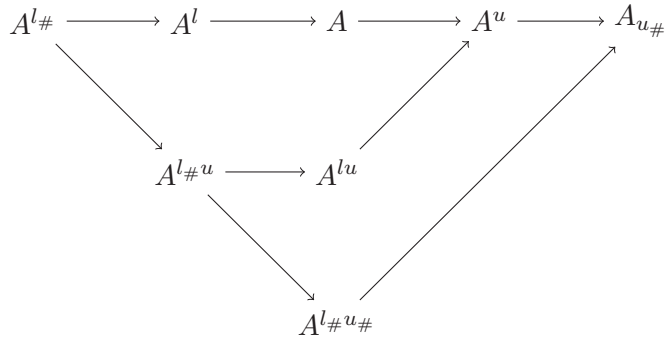
\begin{figure}[h]
\begin{tikzpicture}[node distance=2cm, auto]
\node (Alh) {$A^{l_{\#}}$};
\node (Al) [right of=Alh] {$A^{l}$};
\node (Alhu) [below of=Al] {$A^{l_{\#}u}$};
\node (Alu) [right of=Alhu] {$A^{lu}$};
\node (Alhuh) [below of=Alu] {$A^{l_{\#} u_{ \#}}$};
\node (A) [right of=Al] {$A$};
\node (Au) [right of=A] {$A^{u} $};
\node (Auh) [right of=Au] {$A_{u_{\#}}$};
\draw[->] (Alh) to node {}(Al);
\draw[->] (Al) to node {}(A);
\draw[->] (A) to node {}(Au);
\draw[->] (Au) to node {}(Auh);
\draw[->] (Alh) to node {}(Alhu);
\draw[->] (Alhu) to node {}(Alu);
\draw[->] (Alhu) to node {}(Alhuh);
\draw[->] (Alhuh) to node {}(Auh);
\draw[->] (Alu) to node {}(Au);
\end{tikzpicture}
\caption{Relation Between Approximate Approximations}
\end{figure}
\end{center}
\end{theorem}
\begin{proof}
It is clear that $A^{l}\subseteq A^{u}\subseteq A^{u_{\#}}$. So $A^{l}\,\nsubseteq \, A^{u_{\#}}\setminus A^{u}$.
\begin{align*}
x\in (A^{l} \setminus A^{l_{\#}})^{l} \, & \Rightarrow \, (\exists y)\, [y]_{\#}\nsubseteq A \,\&\, x \in [y] \subset A \,\&\, x\in [y]_{\#}\\
 \, & \Rightarrow \, x\in A^{u_{\#}}\,\&\, x \in A^{u} \\
 \, & \Rightarrow \, x\notin A^{u_{\#}}\setminus A^{u}.\\
\mathrm{But} \; [y]_{\#}\subset A^{u_{\#}} \, & \, (\exists z)\, z\in A^{u_{\#}}\,\&\, z\notin A^{u}\,\&\, z\in [y]_{\#}.\\
\mathrm{So} \;  [y]_{\#}\,\subset\, (A^{u_{\#}}\setminus A^{u})^{u_{\#}}\, & \, \mathrm{and} \, \mathrm{it} \,\mathrm{is}\, \mathrm{possible} \,\mathrm{that}\, [y]_{\#}\,\nsubseteq\, (A^{u_{\#}}\setminus A^{u})^{u}. 
\end{align*}
\qed
\end{proof}

\begin{theorem}
For an arbitrary subset $A\subseteq S$ of a \textsf{PRAX} $S$,  
\begin{align*}
(A^{l}\setminus A^{l_{\#}})^{l}\,\nsubseteq \, (A^{u_{\#}}\setminus A^{u})^{u_{\#}}\,\longrightarrow \, A^{u_{\#}}\,=\, A^{u}.\\
A^{u_{\#}}\,\neq \, A^{u} \,\longrightarrow \, A^{l}\setminus A^{l_{\#}})^{l}\,\subseteq \, (A^{u_{\#}}\setminus A^{u})^{u_{\#}}.
\end{align*}
\end{theorem}
\begin{proof}
\begin{itemize}
\item {Let $ S\,=\, \{a, b, c, e, f\}$ and}
\item {let $R$ be the transitive completion satisfying $ R a b ,\, R b c , \, R e f$.}
\item {If $B\,=\, \{a,\, b\}$, $B^{\hat{l}}\,=\, B$,}
\item {but $B^{u_{\#}}\,=\,\{a,\, b,\, c\}\,=\, B^{u}$.}
\item {So $B^{\hat{u}} \,=\, \emptyset$.}
\item {The second part follows from the proof of the above proposition under the restriction in the premise.}
\end{itemize}    
\qed
\end{proof}

\begin{theorem}
Key properties of the scedastic approximations follow:
\begin{enumerate}
\item {$(\forall B\in \wp (S))(B^{\hat{l}}\,=\, B\, \nrightarrow \, B^{\hat{u}}\,=\, B) $.}
\item {$(\forall B\in \wp (S))(B^{\hat{u}}\,=\, B\, \rightarrow \, B^{\hat{l}}\,=\, B) $.}
\item {$(\forall B\in \wp (S))\, B^{\hat{l} \hat{l}}\,=\, B^{\hat{l}}$.}
\item {$(\forall B\in \wp (S))\,B^{\hat{u} \hat{u}}\,\neq \, B^{\hat{u}}$. }
\item {It is possible that $(\exists B\in \wp (S)\,B^{\hat{u} \hat{u}}\,\subset \, B^{\hat{u}} )$.}
\end{enumerate}
\end{theorem}

\begin{proof}
\begin{enumerate}
\item {The counter example in the proof of the above theorem works for this statement.}
\item {$x\in B\,\leftrightarrow \, x\in (B^{u_{\#}}\setminus B^{u})^{u_{\#}}\,\leftrightarrow\, (\exists y\in B^{u_{\#}})(\exists z\in B^{u_{\#}}\setminus B^{u})\, x,\, z\in [y]_{\#} \,\&\, z\in B^{u_{\#}}\,\&\,  z \notin B^{u}  $. But this situation requires that elements of the form $z$ be related to $x$ and so we should have $B^{u_{\#}}\,=\, B^{u}$. }
\item {$B^{\hat{l} \hat{l}}\,=\, (B^{\hat{l} l}\setminus B^{\hat{l} l_{\#}})^{l}\,=\, ((B^{l}\setminus B^{l_{\#}})^{l}\setminus \emptyset )^{l}  \,=\, B^{\hat{l}}$. The missing step is of proving $(B^{l}\setminus B^{l_{\#}})^{l l_{\#}}\,=\, \emptyset $.}
\item [4-5]{ We prove the last two assertions together. We provide a counterexample and also show the essential pattern of deviation.
\begin{itemize}
\item []{Let $S\,=\, \{a, b, c, e, f\}$ and $R$ be a reflexive relation s.t. $R a b ,\, R b c , \, R e f$.}
\item []{If $A\,=\, \{a,\, e\}$, then $A^{u_{\#}} \,=\, \{a,\, b,\,c,\, e\}$ and $A^{u}\,=\, \{a,\, b,\, e\}$.}
\item []{Therefore $A^{\hat{u}}\,=\, \{c \} \,\&\, A^{\hat{u}\hat{u}}\,=\, \emptyset \;\&\;  A^{\hat{u} \hat{u}}\,\subset \, A^{\hat{u}}$.}
\item []{In general if $B$ is some subset, then $x\in B^{\hat{u}}\,=\, (A^{u_{\#}}\setminus A^{u} )^{u_{\#}}
\Rightarrow\, (\exists y \in A^{u_{\#}})(\exists z) \, y\in [z]_{\#}\, \&\, y\notin A^{u}\, \&\, y\notin A\, \&\,z\in A \,\&\, y\notin [z] \,\&\, y\in [x]_{\#}$.} 
\end{itemize}} 
\end{enumerate}
\qed
\end{proof}

An interesting problem can be given $A$ for which $A^{u_{\#}}\,\neq\, A^{u}$, when does there exist a $B$ such that
\[B^{l}\,=\,(A^{l}\setminus A^{l_{\#}})^{l}\,=\,A^{\hat{l}} \;\&\; B^{u}\,=\,(A^{u_{\#}}\setminus A^{u})^{u_{\#}}\,=\, A^{\hat{u}} ?\]

\chapter{Rough Dependence}\label{dep}

In this chapter, we introduce a concept of \emph{rough dependence} in general rough set theory. By the term \emph{rough dependence}, we intend to capture the relation between two objects (crisp or rough) that have some representable rough objects in common. There is no process for similarity with the concept \emph{mutual exclusivity} of probability theory and in rough sets we are actually handling evolution without regard to temporality. We would like to eventually analyze the extent to which ontology of not-necessarily-rough origin can be integrated in a seamless way. But in this monograph we will introduce basic concepts, compare them with probabilistic  concepts and look at the semantic value of introduced functions and predicates.

Overall the following  problems are basic and relevant for use in semantics:
\begin{mitemize}
\item {Which concepts of rough dependence provide for an adequate semantics of rough objects in the \textsf{PRAX} context? }
\item {More generally how does this relation vary over other \textsf{RST}s?}
\item {Characterize the connection between granularity and rough dependence?}
\end{mitemize}

By \emph{relation based RST} we mean rough theories originating from generalized approximation spaces of the form $U\,=\,\left\langle \underline{U},\,R\, \right\rangle$, with $\underline{U}$ being a set and $R$ being any binary relation on $\underline{U}$. If instead of a relation we start from a cover of the set, then we will refer to the rough theory as a \emph{cover based RST}.

\begin{definition}
The $\tau \nu$-\emph{infimal degree of dependence} $\beta_{i \tau \nu}$ of $A$ on $B$ will be defined as \[\beta_{i \tau \nu} (A,\, B)\,=\,\inf _{\nu (S) }\,\oplus \,\{C\,:\,C\in \tau(S) \, \&\,\pc C A \,\& \, \pc C B\}.\] Here the infimum means the largest $\nu(S)$ element contained in the aggregation.

The $\tau \nu$-\emph{supremal degree of dependence} $\beta_{s \tau \nu}$ of $A$ on $B$ will be defined as \[\beta_{s \tau \nu} (A,\, B)\,=\,\sup _{\nu (S) }\,\oplus \,\{C\,:\,C\in \tau(S) \,\&\, \pc C A \,\& \, \pc C B\}.\] Here the supremum means the least $\nu(S)$ element containing the sets.

The definition extends to \textsf{RYS} \cite{AM240} in a natural way.
\end{definition}

Note that all of the definitions do not use real-valued rough measures and the cardinality of sets in accord with one of the principles of avoiding contamination.  The ideas of dependence are more closely related to certain semantic operations in classical \textsf{RST}. But these were never seen to be of much interest. The connections with probability theories has been part of a number of papers including \cite{ZPB,ZP5,PZ2002,SL2006,YY2008}, however neither dependence nor independence have received sufficient attention. This is the case with other papers on entropy. It should be noted that the idea of independence in statistics is seen in relation to probabilistic approaches, but dependence has largely not been given much importance in applications. 

The positive region of a set $X$ is $X^{l}$, while the negative region is $X^{uc}$ -- this region is independent from $x$ in the sense of attributes being distinct, but not in the sense of derivability or inference by way of rules. When we talk of dependence or independence of a set relative another, then a basic question would be about possible balance between the two meta principles of independence in the rough theory and the relation to the granular concepts of independence. 

\begin{definition}
Two elements $x,\, y$ in a RBRST or CBRST $S$ will be said to be \emph{PN-independent} $ I_{PN}(xy)$ if and only if  \[x^{l}\,\subseteq \,y^{uc}\; \&\;y^{l}\,\subseteq \,x^{uc} . \]

Two elements $x,\, y$ in a RBRST or CBRST $S$ will be said to be \emph{PN-dependent} $\varsigma_{PN}(xy)$ if and only if  \[x^{l}\,\nsubseteq \,y^{uc}\; \&\;y^{l}\,\nsubseteq \,x^{uc} . \]
\end{definition}

\begin{theorem}
Over the \textsf{RYS} corresponding to classical \textsf{RST}, we have the following properties of dependence degrees when $\tau(S)\,=\, \mathcal{G}(S)$ - the granulation of $S$ and $\nu(S)\,=\, \delta_{l}(S)$ - the set of lower definite elements. We omit the subscripts $\tau \nu$ and braces in $\beta_{i \tau \nu}(x, y)$ in the following:

\begin{enumerate}
 \item {$\beta_{i}x y\,=\, x^{l}\, \cap\,y^ l \,=\,\beta_{s}x y $ (subscripts $i , \, s$ on $\beta$ can therefore be omitted).}
 \item {$\beta x x \,=\, x^l $. } 
 \item {$\beta x y \,=\, \beta y x$. }
 \item {$\beta (\beta x y) x \, =\, \beta x y$.}
 \item {$\pc (\beta x y)(\beta x (y \oplus z))$.}
 \item {$(\pc y^{l} z \longrightarrow \pc (\beta x y)(\beta x z))$.}
 \item {$\beta x y \,=\, \beta x^{l} y^{l} \,=\, \beta x y^{l} $.}
 \item {$\beta 0 x = 0 \,; \;\, \beta x 1 = x^{l}$.}
 \item {$(\pc x y \longrightarrow \beta x y = x^{l})$.}
\end{enumerate}
\end{theorem}

We prove this in the next chapter.

\begin{theorem}
For classical \textsf{RST}, a semantics over the classical semantic domain can be formulated with no reference to lower and upper approximation operators using the operations $\cap,\, c ,\, \beta $ on the power-set of $S$, $S$ being an approximation space.  
\end{theorem}

\begin{proof}
We have already shown that $l$ is representable in terms of $\beta$.
So the the result follows.
\end{proof}

\section{Dependence in PRAX}

When we set $\nu (S)\,= \delta_{l}(S)  $ and $\tau(S)\,=\, \mathcal{G}(S) $ - the successor neighborhood granulation, then the situation in \textsf{PRAX} contexts is similar, but we cannot define $u$ from $l$ and complementation. However when we set $\nu (S)\,= \delta_{u}(S)  $, then the situation is very different.

\begin{theorem}\label{vag}
Over the \textsf{RYS} corresponding to \textsf{PRAX} with $\pc \,=\, \subseteq $, $\oplus\,=\, \cup$ and $\odot \,=\, \cap$, we have the following properties of dependence degrees when $\tau(S)\,=\, \mathcal{S}$ - the granulation of $S$ and $\nu(S)\,=\, \delta_{l}(S)$ - the set of lower definite elements. In fact this holds in any reflexive \textsf{RBRST}. We omit the subscripts $\tau \nu$ and braces in $\beta_{i \tau \nu}(x, y)$ in the following:
\begin{enumerate}
 \item {$\beta_{i}x y\,=\, x^{l}\, \cap\,y^ l \,=\,\beta_{s}x y $ (subscripts $i , \, s$ on $\beta$ can therefore be omitted).}
 \item {$\beta x x \,=\, x^l $;  $\beta x y \,=\, \beta y x$. } 
 \item {$(x\odot y\,=\, 0 \,\longrightarrow \, \beta_{i} x y \,=\, 0)$, but the converse is false. }
 \item {$\beta (\beta x y) x \, =\, \beta x y$.}
 \item {$\pc (\beta x y)(\beta x (y \oplus z))$.}
 \item {$(\pc y^{l} z \longrightarrow \pc (\beta x y)(\beta x z))$.}
 \item {$\beta x y \,=\, \beta x^{l} y^{l} \,=\, \beta x y^{l} $.}
 \item {$\beta 0 x = 0 \,; \;\, \beta x 1 = x^{l}$.}
 \item {$(\pc x y \longrightarrow \beta x y = x^{l})$.}
\end{enumerate}
\end{theorem}

\begin{proof}
\begin{enumerate}
 \item {$\beta_{i}x y$ is the union of the collection of successor neighborhoods generated by elements $x$ and $y$ that are included in both of them. So $\beta_{i}x y \,=\, x^{l}\, \cap\,y^ l \,=\,\beta_{s}x y $.}
 \item {$\beta x x \,=\, x^l $;  $\beta x y \,=\, \beta y x$. is obvious} 
 \item {If $(x\odot y\,=\, 0$, then $x$ and $y$ have no elements in common and cannot have common successor neighborhoods. If $\beta_{i} x y \,=\, 0$, then $x,\,y$ have no common successor neighborhoods, but can still have common elements. So the statement follows.}
 \item {$\beta x y \subseteq x^{l}\,\subseteq x$ by the first statement. So $\beta (\beta x y) x \, =\, \beta x y$.}
 \item {$\pc (\beta x y)(\beta x (y \oplus z))$ follows by monotonicity.}
 \item {If $\pc y^{l} z$ is the same thing as $y^{l}\subseteq z$. $\beta x y\, =\, x^{l}\cap y^{l}$ and $\beta x z \,=\, x^{l}\cap z^{l}$ by the first statement. So we have $(\pc y^{l} z \longrightarrow \pc (\beta x y)(\beta x z))$.}
 \item {$\beta x y \,=\, \beta x^{l} y^{l} \,=\, \beta x y^{l} $ holds because $l$ is an idempotent operation in a \textsf{PRAX} \cite{AM270}.}
 \item {Rest of the statements are obvious.}
\end{enumerate}
\qed
\end{proof}

Even though the properties are similar for reflexive \textsf{RBRST} when $\nu (S)\,= \delta_{l}(S)  $ and $\tau(S)\,=\, \mathcal{G}(S) $, there are key differences that can be characterized in terms of special sets. 
\begin{itemize}
\item {$\beta x y = z $ if and only if $(\forall a\in z)(\exists b \in z) \, a\in [z]\,\subseteq x\cap y$.}
\item {So we can select a minimal $K_{z}\,\subseteq z$ satisfying $(\forall a\in z )(\exists b\in K_{z})\,a\in \,[b]\,\subseteq x $ and $(\forall e\in K_{z})\,[e]\subseteq x\cap y  $. Minimality being with respect to the inclusion order.}
\item {Let $\mathcal{P}_{z}$ be the collection of all such $K_{z}$ and let $\mathcal{B}_{z}$ be the subcollection of $\mathcal{P}_{z}$ satisfying the condition: if $K\,\in\,\mathcal{B}_{z} $ then $(\forall a\in K)(\forall b\in [a])(\exists J\in \mathcal{B}_{z})\, b\in J$. $\mathcal{P}_{z}$ will be called the local basis and $\mathcal{B}_{z}$, the local super basis of $z$. }
\end{itemize}
\begin{proposition}
For classical \textsf{RST} $(\forall z)\,\mathcal{B}_{z}\,=\, \mathcal{P}_{z}$  and conversely.
\end{proposition}

\begin{theorem}
In the context of \ref{vag}, if we set $\nu (S)\,= \delta_{u}(S)  $ and $\tau(S)$ is as before, then we have (by $\beta xy $, we mean $\beta_{i} xy$)
\begin{enumerate}
\item {$\pc (\beta x y) (\beta_{i \delta_{l}(S)} x y) $, }
\item {$\pc(\beta x x) ( x^l) $;  $\beta x y \,=\, \beta y x$. } 
\item {$(x\odot y\,=\, 0 \,\longrightarrow \, \beta_{i} x y \,=\, 0)$, but the converse is false. }
\item {$\beta (\beta x y) x \, =\, \beta x y$.}
\item {$\pc (\beta x y)(\beta x (y \oplus z))$.}
\item {$(\pc y^{l} z \longrightarrow \pc (\beta x y)(\beta x z))$.}
\item {$\beta x y \,=\, \beta x^{l} y^{l} \,;\, \pc (\beta x y^{l})(\beta x^{u} y^{u})$.}
\item {$\beta 0 x = 0 \,; \;\, \pc (\beta x 1) (x^{l})$.}
\item {$(\pc x y \longrightarrow \pc (\beta z x)(\beta z y))$}
\item {$(\beta x y)^l\, =\,\beta xy $.}
\end{enumerate}
\end{theorem}

\begin{proof}
\begin{enumerate}
\item {By definition $\beta_{i \tau \nu} (A,\, B)\,=\,\inf _{\nu (S) }\,\oplus \,\{C\,:\,C\in \tau(S) \, \&\,\pc C A \,\& \, \pc C B\}$, so $\beta x y$ is the greatest upper definite set contained in the union of common successor neighborhoods included in $x$ and $y$. So it is necessarily a subset of $x^{l}\cap y^{l}$. In a \textsf{PRAX}, $u$ is not idempotent and in general $x^u \,\subseteq\, x^{uu}$ (\cite{AM270}). So $\pc (\beta x y) (\beta_{i \delta_{l}(S)} x y) $. }
\item {The statements $\pc(\beta x x) ( x^l) $ and  $\beta x y \,=\, \beta y x$ follow from the above.} 
\item {The proof is similar to that of third statement of \ref{vag}. }
\item {In constructing $\beta (\beta x y) x$ from $\beta x y$, we are not searching for upper definite subsets strictly contained in the latter. So the property follows.}
\item {$\pc (\beta x y)(\beta x (y \oplus z))$ follows by monotonicity.}
\item {Obvious from previous statements.}
\item {Note that $\beta x^{u} y^{u}$ is a subset of $x^{u}\cap y^{u}$ and in general contains $\beta x y$.}
\item {Is a special case of the first statement. $0$ is the empty set and $1$ is the top. }
\item {Follows by monotonicity.}
\item {Upper definite subsets are necessarily lower definite, so $(\beta x y)^l\, =\,\beta xy $.}
\end{enumerate}
\qed
\end{proof}

The main properties of PN-dependence is as below:

\begin{theorem}
In the context of \ref{vag}, all of the following hold (we drop the subscript 'PN' in $\varsigma_{PN}$):
\begin{enumerate}
\item {$\varsigma xx  $. }
\item {$(\varsigma x y \,\leftrightarrow \varsigma y x)$.}
\item {In general, $\varsigma x y \,\&\, \varsigma z y $ does not imply $\varsigma x z$. But $\neg \varsigma x z$ is more likely if we assume a bit of frequentism.}
\item {In general, $\varsigma x y \,\nrightarrow \, \varsigma x^{u} y^{u} $ and  $\varsigma x^{u} y^{u} \,\nrightarrow \, \varsigma x y $.}
\item {$(x\cdot y =0 \,\longrightarrow\, \neg \varsigma x y) $.}
\item {$(\pc x y \,\longrightarrow\, \varsigma x y) $.}
\end{enumerate}
\end{theorem}

\begin{theorem}
In the context of \ref{vag}, if $\beta x y \neq 0$ then $\varsigma x y$, but the converse need not hold. In classical \textsf{RST}, the converse holds as well. 
\end{theorem}

\begin{proof}
If $\beta x y \neq 0$, then it follows that $x^{l}\,\cap\, y^{l}\neq \emptyset$ under the assumptions.
If we assume $x^l \subseteq y^{uc} \,\vee \, y^{l}\subseteq x^{uc}$, then in each of the three cases we have a contradiction.
So the first part of the result follows.

In the classical case, if $x^{l}\subseteq y^{uc}$ is not empty, then it should be a union of successor neighborhoods and similarly for  
$y^{l}\subseteq x^{uc}$. These two parts should necessarily be common to $x^{l}$ and $y^{l}$. So the converse holds for classical \textsf{RST}. The proof does not work for \textsf{PRAX} and we know why it does not succeed.
\qed
\end{proof}

\chapter{Comparison with Dependence in Probabilistic Theories}

Probability measures may not exist in the first place over any given collection of sets, so even \textsf{CBRST} is necessarily more general and the idea of mutual exclusivity is not the correct concept corresponding to rough dependence.  The basic idea of probabilistic dependence is oriented because occurrence of an event can be favorable or unfavorable for another event. In standard versions of rough set theory this has no corresponding concept. The concept of dependence in probability is rarely considered in the literature. The version in \cite{BD2010} uses a not-so intuitive valuation but is nevertheless useful. We abstract the subjective aspect of the valuation for comparison. 

Among the different understandings of probabilistic causation, frequentism (\cite{AH2009}) and the tendency to omit necessary conditions are particularly problematic in various soft computing situations. In avoiding real-valued rough measures, we are committing to avoid the excesses of frequentism in rough sets.   

If $(X,\,\mathcal{S} ,\, p)$ is a probability space with $X$ being a set, $\mathcal{S}$ being a $\sigma$-algebra over $X$ and $p$ being a probability function (we can use a collection of probability functions and handle more complex notions of dependence in 'probability structures', but these add little to the comparison), then the most natural dependence function $\delta :\, {\mathcal{S}}^{2}\,\longmapsto \, \Re $ is defined by  
\[\partial (x, \, y)\,=\, p(x\cap y)\,-\, p(x)\,\cdot \, p(y)\]

This function satisfies a number of properties that can be used to characterize dependence. In the subjective probability domain where $p$ takes value in a bounded partially ordered partial semi-ring or your favorite partially ordered algebra, we will need to replace $\delta$ with a pair of predicates. So orientation of dependence seems to be fundamental in general forms of probability theory as well. 

Two events $x,\, y \in X $ are \emph{mutually exclusive} if and only if $x\cap y\,\neq\, \emptyset$. This concept can be extended to countable sets of events in a natural way. Also it is worthwhile to modify the concept of mutual exclusivity as in following definition:

\begin{definition}
Two events $x,\, y$ will be said to be \emph{weakly mutually exclusive} (\textsf{WME}) if and only if \[x\,\cap\,y \,\neq\, z \, \& \,p(z)\,=\, 0 .\]
\end{definition}
Most results of probability theory involving mutual exclusivity continue to hold with the weaker assumption of \textsf{WME} and importantly is a better (though artificial) concept for comparison with the situation for rough sets. 

\begin{definition}
 In the above context, let 
 \begin{itemize}
 \item {$\pi x y$ if and only if $p(x) \cdot p(y)\, <\, p(x\cap y)  $}
 \item {$\sigma x y $ if and only if $p(x\cap y)\,<\, p(x) \cdot p(y) $}
 \end{itemize}
\end{definition}

\begin{proposition}
All of the following hold in a probability space:
\begin{itemize}
\item {$\pi x y^{c} \,\leftrightarrow \, \sigma y x $}
\item {$\pi x y \,\leftrightarrow \, \pi y x$}
\item {$(x\cap y \neq \emptyset \,\longrightarrow\, (\pi x a\,\&\, \pi y a \,\longrightarrow\, \pi (x\cup y) a ))$}
\item {$(x\cap y \neq \emptyset \,\longrightarrow\, (\sigma x a\,\&\, \sigma y a \,\longrightarrow\, \sigma (x\cup y) a ))$}
\item {$(\emptyset\,\neq\,x\subseteq y \, \longrightarrow\,\pi x y )$}
\item {$(x \cap y\,=\, \emptyset \,\longrightarrow\, \sigma x y)$}
\end{itemize}
\end{proposition}

Instead of using the the function $\partial(x, y)$, we can use the relations $\pi, \, \sigma$,as the former lacks a comparable contamination-free counterpart in rough set theory and also has peculiar properties like $\partial(x, x) \in [0,1/4]$. 

\begin{proposition}
In the probability space above $0 \leq \partial(x, x)\leq 0.25$, $-0.25 \leq \partial (x,x^{c}) \leq 0$ and 
$x,\,y$ are independent implies $\partial(x, y) = 0$, but not conversely. 
\end{proposition}

\begin{proof}
The proof of the inequalities follow by a simple application of real analysis. 
\qed
\end{proof}

So it follows that the interpretation of the function $\partial (x,y)$ as in \cite{BD2010} is actually incomplete. It combines certainty of the event with dependence. 

Even though we can speak of positive, negative and neutral regions corresponding to an arbitrary subset $A$ of a \textsf{RBRST} or \textsf{CBRST} $S$, natural ideas of dependence do not correspond to the scenario in probability space. In fact, 

\begin{theorem}
Predicates having properties identical with those of $\pi$ and $\sigma$ cannot be defined in the context of \ref{vag}. 
\end{theorem}

Proof of this and more general results will appear separately.  

\chapter{Dependency Semantics of PRAX}

We develop dependency based semantics in at least two ways. The \emph{internalization based semantics} is essentially about adjoining predicates to the Nelson algebra corresponding to $\mathfrak{R}_{w}(S)$. The \emph{cumulation based semantics} is essentially about cumulating both the semantics of $\mathfrak{R}(S)$, adjusting operations and adjoining predicates. We use broader dependency based predicates in this case, but the value of the method is in fusion of the methodologies.

The central blocks of development of the cumulation based dependency semantics are the following:

\begin{mitemize}
\item {Take $\mathfrak{R}(S) \,\cup\,\mathfrak{R}_{w}(S)$ as the universal set of the intended partial/total algebraic system.  }
\item {Use a one point completion of $\tau$ to distinguish between elements of $\mathfrak{R}_{w}(S)\,\setminus\, \mathfrak{R}(S)$ and those in $\mathfrak{R}(S)$.}
\item {Extend the idea of operational dependency to pairs of sets.}
\item {Extend operations of aggregation, commonality and dual suitably.}
\item {Interpretation and meaning of semantic dependence?}
\end{mitemize}

The first step is obvious, but involves elimination of other potential sets arising from the properties of the map $\tau$.

\section*{One Point Completion}

Because we have $R\subseteq R_{w}$ and $R_{w}$ is transitive, so
\begin{proposition}
\[\alpha\in \mathfrak{R}(S)\cap \mathfrak{R}_{w}(S)\;\mathrm{ if\; and\; only\; if\;} \tau (\alpha) \,=\, \alpha.\] 
\end{proposition}

We adjoin an element $0$ to $\mathfrak{R}(S) \,\cup\,\mathfrak{R}_{w}(S)$ to form $\mathfrak{R}^{*}(S)$ and extend $\tau$ (interpreted as a partial operation) to $\overline{\tau}$ as follows:
\[\overline{\tau} (\alpha)\,=\,\left\lbrace \begin{array}{ll}
\tau(\alpha) & \mathrm{if}\; \alpha\in \mathfrak{R}(S), \\
0 & \mathrm{if} \; \alpha\notin \mathfrak{R}(S).
\end{array}\right.
 \]
Note that this operation suffices to distinguish between elements common to $\mathfrak{R}(S)$ and $\mathfrak{R}_{w}(S)$, and those exclusively in $\mathfrak{R}(S)$ and not in $\mathfrak{R}_{w}(S)$.

\section*{Dependency on Pairs}

We have the option of considering all dependency relative the Nelson algebra or $\mathfrak{R}(S)$. First we consider everything relative the former -so that we may be able to avoid the references to the latter. 

\begin{definition}
By the \emph{paired infimal degree of dependence} $\beta^{+}_{i \tau_1 \tau_2 \nu_1 \nu_2}$ of $\alpha$ on $\beta$ will be defined as \[(\beta_{i \tau_1 \nu_1} (e_{1}\alpha,\, e_{1}\beta), \,\beta_{i \tau_2 \nu_2} (e_{2}\alpha,\, e_{2}\beta))    .\] Here the infimums involved are the largest $\nu_1(S)$ and $\nu_{2} (S)$ elements contained in the aggregation and the $e_{j}\alpha$ is the $j$-th component of $\alpha$. 
\end{definition}

We will however be interested in the following well defined specialization with $\tau_{1}(S)\,=\,\tau_{2} (S)\, =\,\mathcal{G}_{w}(S)$, $\nu_{1}\,=\,\delta_{lw}(S) $ and $\nu_{2}\,=\,\Gamma_{uw}(S)$ in all that follows. When we need to specialize the dependencies  between a element in $\mathfrak{R}(S)$ and its image in $\mathfrak{R}_{w}(S)$, we can define:

\begin{definition}
Under the above assumptions, by the \emph{relative semantic dependence} $\varrho (\alpha)$ of $\alpha\in \mathfrak{R}(S)$, we will mean
\[\varrho (\alpha) \,=\, \beta^{+}_{i}(\alpha, \tau(\alpha)).\]
\end{definition}

The idea of relative semantic dependence refers to elements in $\mathfrak{R}(S)$ and it can be reinterpreted as a relation on $\mathfrak{R}_{w}(S)$.

\section*{Internalization Based Semantics}

\begin{definition}
A relation $\Upsilon$ on $\mathfrak{R}_{w}(S)$ will be said to be a \emph{relsem-relation} if and only if 
\[\Upsilon\tau(\alpha)\gamma  \,\leftrightarrow \, (\exists \beta \in \tau^{\dashv}\tau(\alpha))\,\gamma\,=\,\varrho(\beta). \]
Note that, $\tau(\alpha)\,=\,\tau(\beta)$ by definition of $\tau^{\dashv}$. 
\end{definition}

Through the above definitions we have arrived at the following internalized approximate definition:

\begin{definition}
By an \emph{Approximate Proto Vague Semantics} of a \textsf{PRAX} ${S}$ we will mean an algebraic system of the form \[\mathfrak{P}(S)\,=\,\left\langle \mathfrak{R}_{w}(S),\, \Upsilon \vee, \wedge, c, \bot, \top \right\rangle ,\] with $\left\langle \mathfrak{R}_{w}(S),\, \vee_{w}, \wedge_{w}, c, \bot, \top \right\rangle $  being a Nelson algebra over an algebraic lattice and $\Upsilon$ being as above.
\end{definition}

\begin{theorem}
$\Upsilon$ has the following properties:
\begin{align*}
\alpha\,=\,\tau(\alpha)\,\longrightarrow\, \Upsilon \alpha \alpha .& \\ 
\Upsilon \alpha \gamma \, \longrightarrow\, \gamma \wedge_{w} \alpha\,=\, \gamma .& \\
\Upsilon \alpha \gamma \,\&\, \Upsilon \gamma \alpha \,\longrightarrow\, \alpha\,=\, \gamma .& \\
\Upsilon \bot \bot \,\&\, \Upsilon \top \top .& \\
\Upsilon \alpha \gamma \,\&\,\Upsilon \beta \gamma \longrightarrow\,\Upsilon (\alpha\vee_{w} \beta) \gamma . &
\end{align*}
\end{theorem}

\begin{proof}
\begin{mitemize}
\item {If $\alpha\,=\, \tau(\alpha)$, then $\alpha\,=\,\varrho (\alpha)\,=\, \beta^{+}_{i}(\alpha, \tau(\alpha))$. So $\Upsilon \alpha \alpha$.}
\item {If $\Upsilon \alpha \gamma$, then it follows from the definition of $\beta^{+}_{i}$, that the components of gamma are respectively included in those of $\alpha$. So $\gamma \wedge \alpha\,=\, \gamma $.}
\item {Follows from the previous. }
\item {Proof is easy.}
\item {From the premise we have $(\exists \mu \in \tau^{\dashv}\tau(\alpha))\,\gamma\,=\,\varrho(\mu)$ and $(\exists \nu \in \tau^{\dashv}\tau(\beta))\,\gamma\,=\,\varrho(\nu)$. This yields $(\exists \lambda \in \tau^{\dashv}\tau(\alpha\vee_{w}\beta))\,\gamma\,=\,\varrho(\lambda)$ as can be checked from the components.}
\end{mitemize} 
\qed
\end{proof}

$\Upsilon_{\tau(\alpha)}\,=\,\{\gamma \,; \,\Upsilon\tau(\alpha)\gamma \}$ is the approximate reflection of the set of $\tau$-equivalent elements in $\mathfrak{R}(S)$ identified by their dependence degree. In the approximate semantics we do not completely lose track of aggregation and commonality as the above theorem shows.

\begin{definition}
By the $\varrho /\sigma$-\emph{semantic dependences} $\varrho (\alpha)$, $\sigma(\alpha)$ of $\alpha\in \mathfrak{R}(S)$, we mean \[\varrho (\alpha) \,=\, \beta^{+}_{i}(\alpha, \tau(\alpha))\;\mathrm{and}\] 
\[\sigma (\alpha) = \beta^{+}_{i}(\alpha, ((\varphi (e_{1}\alpha) \setminus e_{1}\alpha)^{l},(\varphi (e_{2}\alpha) \setminus e_{1}\alpha)^{u}))\] respectively. Such relations are optional in the internalization process. 
\end{definition}

For a falls-down semantics, the natural candidates include the ones corresponding to largest equivalence or the largest semi-transitive contained in $R$. We reserve the latter for a separate paper. For the former, the general technique (using $\sigma(\alpha)$) extends to \textsf{PRAX} as follows:
\begin{definition}
\begin{mitemize}
\item {Define a map $\int$ from set of neighborhoods to $l$-definite elements  via \[\int ([x]_o)\,=\,\cup_{y\in[x]_o} [y] \] and extend it to images of $l_o, u_o$ via, \[\oint (A^{l_o})= \cup_{[y]_{o}\subseteq A^{l_o}} \int ([y]_o).\]}
\item {Extend this to a map $\ltimes :\mathfrak{R}_{o}(S)\mapsto \mathfrak{R}(S)$ via \[\ltimes(\alpha) = (\oint (e_{1}\alpha),\oint (e_{2}\alpha)).\]}
\item {Define a predicate $\Pi$ on $\mathfrak{R}_{o}(S)$ as per \[\Pi \alpha \nu\,\leftrightarrow\,(\exists \gamma\in \ltimes^{\dashv}\ltimes(\alpha))\, \beta_{i}^{+}(\alpha, \gamma) = \nu .\] Let $\Pi_{\alpha} =\{\nu \,; \Pi \alpha \nu\}$.}
\item {By a \emph{Direct Falls Down} semantics of \textsf{PRAX}, we will mean an algebraic system of the form 
\[\mathfrak{I}(S)\,=\,\left\langle \mathfrak{R}_{o}(S),\, \Pi , \vee, \wedge, c, \bot, \top \right\rangle,\] with $\left\langle \mathfrak{R}_{o}(S),\, \vee_{o}, \wedge_{o}, \rightarrow , c, \bot, \top \right\rangle $  being a semi-simple Nelson algebra \cite{PPM2}.}
\item {The falls down semantics determines a cover $\mathfrak{I}^{*}(S) = \{\Pi_{\alpha}\,;\, \alpha\in \mathfrak{R}_{o}(S) \}$}
\end{mitemize} 
\end{definition}

\begin{theorem}
In the above context, all of the following hold:
\begin{mitemize}
\item {$\Pi \alpha \alpha$.}
\item {$(\Pi \alpha \mu\, \& \,\Pi \mu \alpha \longrightarrow \alpha = \mu)$. }
\item {$(\Pi \alpha \gamma \, \longrightarrow \, \gamma \subseteq \alpha)$. The converse is false.}
\item {$\alpha\neq \bot \,\&\, \Pi \alpha \gamma \,\&\, \Pi \alpha \mu \,\longrightarrow\, \beta_{i}^{+}(\gamma, \mu) \neq \bot$.}
\item {$\mu \in \Pi_{\alpha} \,\&\, \mu\subseteq \nu\,\subseteq \alpha\,\longrightarrow\, \nu\in \Pi_{\alpha}$.}
\end{mitemize}
\end{theorem}

\section*{Cumulation Based Semantics}

The idea of cumulation is correctly a way of enhancing our original semantics based on proto-vagueness algebras with the Nelson algebraic semantics and the operational dependence. We define it for a central problem relating to the underlying semantic domains and  

\begin{definition}
By a \emph{cumulative proto-vague algebra} we will mean a partial algebra of the form 
\[\mathfrak{C}(S)\,=\, \left\langle \mathfrak{R}^{*}(S),\overline{tau},\, \oplus, \odot, \otimes, \dagger , \bot , \top  \right \rangle.  \]   
 
\end{definition}

\begin{flushleft}
\textbf{Problem:} 
\end{flushleft}

When can the cumulation based semantics be deduced from (that is the extra operations can be defined from the original ones) within a full proto-vagueness algebra?

\chapter{Geometry of Granular Knowledge Interpretation}

In this chapter we provide a brief overview of knowledge interpretation in the \textsf{PRAX} contexts in the light of the results on representation of rough objects. For details of the knowledge interpretation of rough sets, the reader is referred to \cite{AM909,ZPB,CHP}. By an extension of those considerations any proto-transitive relation corresponds to knowledge. Here we will be concerned with representation of knowledge and that in turn depends on our choice of semantic domain - the most natural is the one corresponding to the rough objects. But we know that representation is involved. 

\emph{Any knowledge, however involved, may be seen as a collection of
concepts with admissible operations of reasoning defined on them}. Knowledges associated
\textsf{PRAX} have various peculiarities corresponding to the
semantic evolution of rough objects in it. To start with, the semantic domains of representation
properly contain the semantic domains of interpretation. Not surprisingly, it is
because the rough objects corresponding to $l,\, u$ cannot be represented perfectly in
terms of objects from $\delta_{lu}(S)$ alone. \emph{In the nongranular perspective too,
this representation aspect should matter} - ''\emph{should}'', because it is matter of
choice during generalization from the classical case in the non granular approach.   

The natural \rsds  of $l,\, u$ is Meta-R, while that of $l_{o},\, u_{o}$ is
$\mathfrak{O}$ (say, corresponding rough objects of $\tau (R)$). These can be seen as
separate domains or as parts of a minimal containing domain that permits enough
expression. As we have seen knowledge is correctly representable in terms of \emph{atomic concepts of knowledge} at semantic
domains placed between Meta-C and Meta-R and not at the latter. So the
characterization of possible semantic domains and their mutual ordering - leading to their
geometry is of interest.

The following will be assumed to be \emph{part} of the interpretation:
\begin{mitemize}
\item {Two types of rough objects corresponding to Meta-R and $\mathfrak{O}$  and their
natural correspondence correspond to concepts or weakenings thereof. A concept relative
one semantic domain need not be one of the other.}
\item {A granule of the \rsd $\mathcal{O}$ is necessarily a concept of $\mathcal{O}$, but
a granule of Meta-R may not be a concept of $\mathcal{O}$ or Meta-R.}
\item {Critical points are not necessarily concepts of either semantic domain.}
\item {Critical points and the representation of rough objects require the \rsds to be
extended.}
\end{mitemize}

The above obviously assumes that a \textsf{PRAX} $S$ has at least two kinds of knowledge
associated (in relation to the Pawlak-sense interpretation). To make the interpretations
precise, we will indicate them by $\mathcal{I}_{1}(S)$ and $\mathcal{I}_{o}(S)$
respectively (corresponding to the approximations to $l,\, u$ and $l_{o},\, u_{o}$
respectively). The pair $(\mathcal{I}_{1}(S),\,\mathcal{I}_{o}(S))$ will also be
referred to as the \emph{generalized \textsf{KI}}. 
\begin{definition}
Given two \textsf{PRAX} $S = \left\langle \underline{S},\,R \right\rangle$, $
V = \left\langle \underline{S},\,Q \right\rangle$, $S$ will be said to be \emph{o-coarser}
than $V$ \ifof $\mathcal{I}_{o}(S)$ is coarser than $\mathcal{I}_{o}(V)$ in Pawlak-sense
( that is $\tau(R)\subseteq \tau (Q)$). Conversely, $V$ will be said to be a
\emph{o-refinement} of $S$.

$S$ will be said to be \emph{p-coarser} than $V$ \ifof 
$\mathcal{I}_{1}(S)$ is coarser than $\mathcal{I}_{1}(V)$ in the sense $R\subseteq Q$.
Conversely, $V$ will be said to be a \emph{p-refinement} of $S$.  
\end{definition}

An extended concept of positive regions is defined next.

\begin{definition}
If $S_{1}\,=\, \left\langle \underline{S},Q \right\rangle $ and $S_{2}\,=\, \left\langle
\underline{S},P \right\rangle $ are two \textsf{PRAX} such that $Q\subset R$, then
by the
\emph{granular positive region} of $Q$ with respect to $R$ is given by $gPOS_{R}(Q)\,=\,
\{ [x]_{Q}^{l_{R}}\, :\, x\in \underline{S} \}$, where $[x]_{Q}^{l_{R}}$ is the lower
approximation (relative $R$) of the $Q$-related elements of $x$. Using this we can define
the \emph{granular extent of dependence} of knowledge encoded by $R$ on the knowledge
encoded by $Q$ by natural injections $:gPOS_{R}(Q) \longmapsto \mathcal{G}_{R} $.  
 \end{definition}

Lower critical points can be naturally interpreted as preconcepts that are
definitely included in the discourse, while upper critical points are preconcepts that
include most of the discourse. The problem with this interpretation is that it's
representation requires a semantic domain at which critical points of different kinds can
be found. A key requirement for such a domain would be the feasibility of rough counting
procedures like \textsf{IPC} \cite{AM240}. We will refer to a semantic domain that has
critical points of different types as basic objects as a \textsf{Meta-RC}. 

The following possible axioms of granular knowledge that also figure in \cite{AM909} (due to the the present author), get into difficulties with the present approach and even when we restrict attention to $\mathcal{I}_{1}(S)$: 
\begin{enumerate}
\item {Individual granules are atomic units of knowledge.}
\item {Maximal collections of granules subject to a concept of mutual
independence are admissible concepts of knowledge.}
\item {Parts common to subcollections of maximal collections of granules are also
knowledge.}
\end{enumerate}

The first axiom holds in weakened form as the granulation $\mathcal{G}$ for
$\mathcal{I}_{1}(S)$ is only lower definite and affects the other. The possibility of
other nice granulations being possible for the \textsf{PRAX} case appears to be possible
at the cost of other nice properties. So we can conclude that in proper 
\textsf{KR} happens at semantic domains like \textsf{Meta-RC} where critical
points of different types are perceived. Further at Meta-R, rough objects may correspond
to knowledge or conjectures - if we require the concept of proof to be an ontological
concept or beliefs. The scenario can be made more complex with associations
of $\mathfrak{O}$ knowledges.

From a non-granular perspective, in Meta-R  rough objects must
correspond to knowledge with some of them lacking a proper evolution - there is no problem
here. Even if we permit $\mathfrak{O}$ objects, then in the perspective we would be able
to speak of two kinds of closely associated knowledges. 

The connections with non-monotonic reasoning and the approximate Nelson algebra semantics developed in this monograph suggest further enhancements to the above. These will be explored separately.

\chapter*{Further Directions and Remarks}

In this research we have developed the basic theory of rough sets over proto transitive relations, characterized the nature of rough objects and possible approximations, and have developed two different algebraic semantics for the same. Some of the work constitute a continuation of earlier work by the present author. Important connections between approximations and operators of generalized operators of non-monotonic reasoning have also been established in the monograph. This opens the door for many new kinds of semantic connections that we hope to consider in future. Various examples at the level of real-life applications are also outlined in the monograph. Concepts have been illustrated through a persistent example. Knowledge interpretation in \textsf{PRAX} contexts have also ben outlined.  

In continuation of earlier work by the present author \cite{AM3690} on semantic consequences of the relation between protransitivity and its approximations is developed in detail. The theory of rough dependence from the knowledge perspective is also specialized to \textsf{PRAX} and extended for the purposes of the semantics in this monograph. Connections with probabilistic dependence is shown to be lacking any reasonable basis and we once again unsettle unbridled frequentism in rough set theory. The relation of the developed theory with entropy is strongly motivated by the knowledge interpretation \cite{AM909,AM270} and will be part of future work.

The first algebraic semantics was seen to be inadequate in not being particularly elegant and requiring additional predicates for a reasonable abstract representation theorem. This was one reason for restricting derivations involving rough objects of $\tau(R)$.  
The internalization of a semantics of \textsf{PRAX} objects in Nelson algebras through ideas of rough dependence is shown to lead to a beautiful semantics. Further formulations of associated logics will be part of a future paper. The technique can be extended to define approximate semantics in various other rough set-theoretical contexts.

\bibliographystyle{splncs.bst}
\bibliography{biblioam11092014.bib}

\begin{thebibliography}{10}

\bibitem{AM270}
Mani, A.:
\newblock {Dialectics of Knowledge Representation in a Granular Rough Set
  Theory}.
\newblock In: {http://arxiv.org/abs/1212.6519 . Refereed Conference Paper:
  ICLA'2013, Inst. Math. Sci. Chennai}. (2013)  1--12

\bibitem{AM3690}
Mani, A.:
\newblock {Approximation Dialectics of Proto-Transitive Rough Sets}.
\newblock In: {Proceedings of ICFUA'2013}. {LNCS}, Springer Verlag (In Press)
  (Dec 2013)  10 pp

\bibitem{AM3930}
Mani, A.:
\newblock {Ontology, Rough Y-Systems and Dependence}.
\newblock International J of Computer Science and Appl. (Technomath Foundation)
  \textbf{In Press} (2014)  1--23 Special Issue of IJCSA on Computational
  Intelligence.

\bibitem{AM99}
Mani, A.:
\newblock {Choice Inclusive General Rough Semantics}.
\newblock Information Sciences \textbf{181}(6) (2011)  1097--1115

\bibitem{AM240}
Mani, A.:
\newblock {Dialectics of Counting and the Mathematics of Vagueness}.
\newblock Transactions on Rough Sets \textbf{XV}(LNCS 7255) (2012)  122--180

\bibitem{RJ2011}
Janicki, R.:
\newblock {Approximation of Arbitrary Binary Relations by Partial Orders:
  Classical and Rough Set Models}.
\newblock Transactions on Rough Sets \textbf{XIII}(LNCS 6499) (2011)  17--38

\bibitem{SJ}
Jarvinen, J., Radeleczki, S.:
\newblock {Representation of Nelson Algebras by Rough Sets Determined by
  Quasi-orders}.
\newblock Algebra Universalis \textbf{66} (2011)  163--179

\bibitem{JPR}
Jarvinen, J., Pagliani, P., Radeleczki, S.:
\newblock {Information completeness in Nelson algebras of rough sets induced by
  quasiorders}.
\newblock Studia Logica (2012)  1--20

\bibitem{AM1800}
Mani, A.:
\newblock {Axiomatic Approach to Granular Correspondences}.
\newblock In Li, T.,  et~al., eds.: {Proceedings of RSKT'2012, LNAI 7414}.
  Volume LNAI 7414., Springer-Verlag (2012)  482--487

\bibitem{AM3600}
Mani, A.:
\newblock {Contamination-Free Measures and Algebraic Operations}.
\newblock In Pal, N.,  et~al., eds.: {Fuzzy Systems (FUZZ), 2013 IEEE
  International Conference on Fuzzy Systems, Hyderabad, India}. Volume F-1438.
  (2013)  1--8

\bibitem{BD2010}
Dimitrov, B.:
\newblock {Some Obreshkov Measures of Dependence and their Use}.
\newblock Comptes Rendus Acad Bulg. Sci. \textbf{63}(1) (2010)  5--18

\bibitem{PZ2002}
Pawlak, Z.:
\newblock {Decision Tables and Decision Spaces}.
\newblock In: {Proc. 6th International Conf. on Soft Computing and Distributed
  Processing (SCDP'2002)}. (June 24--25 2002)

\bibitem{SL2006}
Slezak, D.:
\newblock {Rough Sets and Bayes Factor}.
\newblock In Peters, J.F., Skowron, A., eds.: {Transactions on Rough Sets III}.
  {LNCS 3400}.
\newblock Springer Verlag (2006)  202--229

\bibitem{GPS04}
Greco, S., Pawlak, Z., Slowinski, R.:
\newblock {Can Bayesian Measures be Useful for Rough Set Decision Making?}
\newblock Engg. Appl. of AI \textbf{17} (2004)  345--361

\bibitem{YY2003}
Yao, Y.:
\newblock {Probabilistic Approach to Roaugh Sets}.
\newblock Expert Systems \textbf{20}(5) (2003)  287--297

\bibitem{BCC2007}
Bianucci, D., Cattaneo, G., Ciucci, D.:
\newblock {Entropies and Co-Entropies of Coverings with Application to
  Incomplete Information Systems}.
\newblock Fundamenta Informaticae \textbf{75} (2007)  77--105

\bibitem{ICH}
Chajda, I., Haviar, M.:
\newblock {Induced Pseudo Orders}.
\newblock Acta Univ. Palack. Olomou \textbf{30}(1) (1991)  9--16

\bibitem{BU}
Burmeister, P.:
\newblock {A Model-Theoretic Oriented Approach to Partial Algebras}.
\newblock Akademie-Verlag (1986, 2002)

\bibitem{LJ}
Ljapin, E.S.:
\newblock {Partial Algebras and Their Applications}.
\newblock Academic, Kluwer (1996)

\bibitem{Sha56}
Moore, E.F., Shannon, C.E.:
\newblock {Reliable Circuits Using Less Reliable Relays-I, II}.
\newblock Bell Systems Technical Journal (1956)  191--208, 281--297

\bibitem{Sha48}
Shannon, C.E.:
\newblock {A Mathematical Theory of Communication}.
\newblock Bell Systems Technical Journal \textbf{27} (1948)  379--423, 623--656

\bibitem{LZ9}
Zadeh, L.A.:
\newblock {Fuzzy sets and information granularity}.
\newblock In Gupta, N.,  et~al., eds.: {Advances in Fuzzy Set Theory and
  Applications}.
\newblock North Holland, Amsterdam (1979)  3--18

\bibitem{Ya01}
Yao, Y.:
\newblock {Information granulation and rough set approximation}.
\newblock Int. J. of Intelligent Systems \textbf{16} (2001)  87--104

\bibitem{TYL}
Lin, T.Y.:
\newblock {Granular Computing -1: The Concept of Granulation and its Formal
  Model}.
\newblock Int. J. Granular Computing, Rough Sets and Int Systems \textbf{1}(1)
  (2009)  21--42

\bibitem{AM69}
Mani, A.:
\newblock {Meaning, Choice and Similarity Based Rough Set Theory}.
\newblock Internat. Conf. Logic and Appl., Jan'2009 Chennai;(Refereed),
  http://arxiv.org/abs/0905.1352 (2009)  1--12

\bibitem{SW3}
Wasilewski, P., Slezak, D.:
\newblock {Foundations of Rough Sets from Vagueness Perspective}.
\newblock In Hassanien, A.,  et~al., eds.: {Rough Computing: Theories,
  Technologies and Applications}. {Information Science Reference}.
\newblock IGI, Global (2008)  1--37

\bibitem{SW}
Slezak, D., Wasilewski, P.:
\newblock {Granular Sets - Foundations and Case Study of Tolerance Spaces}.
\newblock In An, A., Stefanowski, J., Ramanna, S., Butz, C.J., Pedrycz, W.,
  Wang, G., eds.: {RSFDGrC 2007, LNCS}. Volume 4482.
\newblock Springer (2007)  435--442

\bibitem{AM105}
Mani, A.:
\newblock {Algebraic Semantics of Similarity-Based Bitten Rough Set Theory}.
\newblock Fundamenta Informaticae \textbf{97}(1-2) (2009)  177--197

\bibitem{KCM}
Keet, C.M.:
\newblock {A Formal Theory of Granules - Phd Thesis}.
\newblock PhD thesis, Fac of Comp.Sci., Free University of Bozen (2008)

\bibitem{CC5}
Cattaneo, G., Ciucci, D.:
\newblock {Lattices with Interior and Closure Operators and Abstract
  Approximation Spaces}.
\newblock In Peters, J.F.,  et~al., eds.: {Transactions on Rough Sets X, LNCS
  5656}.
\newblock Springer (2009)  67--116

\bibitem{CD3}
Ciucci, D.:
\newblock {Approximation Algebra and Framework}.
\newblock Fundamenta Informaticae \textbf{94} (2009)  147--161

\bibitem{JJ}
Jarvinen, J.:
\newblock {Lattice Theory for Rough Sets}.
\newblock In Peters, J.F.,  et~al., eds.: {Transactions on Rough Sets VI}.
  Volume LNCS 4374.
\newblock Springer Verlag (2007)  400--498

\bibitem{DM94}
Makinson, D.
\newblock In: {General Patterns in Nonmonotonic Reasoning}. Volume~3. Oxford
  University Press (1994)  35--110

\bibitem{DM2003}
Makinson, D.:
\newblock {Bridges between Classical and Nonmonotonic Logic}.
\newblock Logic Journal of the IGPL \textbf{11} (2003)  69--96

\bibitem{ZPB}
Pawlak, Z.:
\newblock {Rough Sets: Theoretical Aspects of Reasoning About Data}.
\newblock Kluwer Academic Publishers, Dodrecht (1991)

\bibitem{ZP5}
Pawlak, Z.:
\newblock {Some Issues in Rough Sets}.
\newblock In Skowron, A., Peters, J.F., eds.: {Transactions on Rough Sets- I}.
  Volume 3100.
\newblock Springer Verlag (2004)  1--58

\bibitem{YY2008}
Yao, Y.:
\newblock {Probabilistic Rough Set Approximations}.
\newblock Internat. J. Appr. Reasoning \textbf{49} (2008)  255--271

\bibitem{AH2009}
Hajek, A.:
\newblock {Fifteen Arguments Against Hypothetical Frequentism}.
\newblock Erkenntnis 211 \textbf{70} (2009)  211--235

\bibitem{PPM2}
Pagliani, P., Chakraborty, M.:
\newblock {A Geometry of Approximation: Rough Set Theory: Logic, Algebra and
  Topology of Conceptual Patterns}.
\newblock Springer, Berlin (2008)

\bibitem{AM909}
Mani, A.:
\newblock {Towards Logics of Some Rough Perspectives of Knowledge}.
\newblock In Suraj, Z., Skowron, A., eds.: {Intelligent Systems Reference
  Library dedicated to the memory of Prof. Pawlak,}.
\newblock Springer Verlag (2011-12)  342--367

\bibitem{CHP}
Chakraborty, M.K., Samanta, P.:
\newblock {Consistency Degree Between Knowledges}.
\newblock In Kryszkiewicz, M.,  et~al., eds.: {RSEISP'2007}. Volume LNAI 4583.,
  Springer Verlag (2007)  133--141

\end{thebibliography}

\end{document}